\documentclass{article}
\usepackage{ijcai24}

\usepackage{times}
\usepackage{soul}
\usepackage{url}
\usepackage[hidelinks]{hyperref}
\usepackage[utf8]{inputenc}
\usepackage[small]{caption}
\usepackage{graphicx}
\usepackage{amsmath}
\usepackage{amsthm}
\usepackage{booktabs}
\usepackage{algorithm}
\usepackage[switch]{lineno}

\usepackage{amsmath,amsfonts,bm}

\def\eqref#1{equation~\ref{#1}}

\def\1{\bm{1}}

\def\rvu{{\mathbf{i}}}

\def\rvu{{\mathbf{u}}}
\def\rvv{{\mathbf{v}}}
\def\rvw{{\mathbf{w}}}
\def\rvx{{\mathbf{x}}}
\def\rvy{{\mathbf{y}}}
\def\rvz{{\mathbf{z}}}

\DeclareMathAlphabet{\mathsfit}{\encodingdefault}{\sfdefault}{m}{sl}
\SetMathAlphabet{\mathsfit}{bold}{\encodingdefault}{\sfdefault}{bx}{n}

\def\gA{{\mathcal{A}}}

\def\gF{{\mathcal{F}}}

\def\gH{{\mathcal{H}}}

\def\gM{{\mathcal{M}}}

\def\gV{{\mathcal{V}}}
\def\gW{{\mathcal{W}}}
\def\gX{{\mathcal{X}}}
\def\gY{{\mathcal{Y}}}
\def\gZ{{\mathcal{Z}}}

\def\sP{{\mathbb{P}}}

\newcommand{\E}{\mathbb{E}}

\newcommand{\R}{\mathbb{R}}

\DeclareMathOperator*{\argmax}{arg\,max}

\usepackage[noend]{algorithmic}

\usepackage{amssymb}

\newtheorem{lemma}{Lemma}
\newtheorem{definition}{Definition}

\newtheorem{assumption}{Assumption}
\newtheorem{remark}{Remark}

\usepackage{multirow}

\newtheorem{theorem}{Theorem}

\title{Towards Sharper Risk Bounds for Minimax Problems}

\author{
Bowei Zhu$^{1,2}$
\and
Shaojie Li$^{1,2}$\and
Yong Liu$^{1,2}$\footnote{Corresponding author.}
\affiliations
$^1$Gaoling School of Artificial Intelligence, Renmin University of China, Beijing, China \\
$^2$Beijing Key Laboratory of Big Data Management and Analysis Methods, Beijing, China 
\emails
\{bowei.zhu, lishaojie95, liuyonggsai\}@ruc.edu.cn
}

\begin{document}

\maketitle

\begin{abstract}
   Minimax problems have achieved success in machine learning such as adversarial training, robust optimization, reinforcement learning. For theoretical analysis, current optimal excess risk bounds, which are composed by generalization error and optimization error, present 1/n-rates in strongly-convex-strongly-concave (SC-SC) settings. Existing studies mainly focus on minimax problems with specific algorithms for optimization error, with only a few studies on generalization performance, which limit better excess risk bounds. In this paper, we study the generalization bounds measured by the gradients of primal functions using uniform localized convergence. We obtain a sharper high probability generalization error bound for nonconvex-strongly-concave (NC-SC) stochastic minimax problems. Furthermore, we provide dimension-independent results under Polyak-Lojasiewicz condition for the outer layer. Based on our generalization error bound, we analyze some popular algorithms such as empirical saddle point (ESP), gradient descent ascent (GDA) and stochastic gradient descent ascent (SGDA). We derive better excess primal risk bounds with further reasonable assumptions, which, to the best of our knowledge, are n times faster than exist results in minimax problems.
\end{abstract}

\section{Introduction}
Modern machine learning settings such as reinforcement learning \cite{du2017stochastic,dai2018sbeed}, adversarial learning \cite{goodfellow2016deep}, robust optimization \cite{chen2017robust,namkoong2017variance} often need to solve minimax problems, which divide the training process into two groups: one for minimization and one for maximization. To solve the problems, various efficient optimization \text{algorithms} such as gradient descent ascent (GDA), stochastic gradient descent ascent (SGDA) have been proposed and widely used in application.

In theoretical analysis, an essential issue is the excess risk, which compares the risk of certain parameters to the Bayes optimal parameters. The standard technique to bound excess risk is to divide it into generalization error and optimization error. Current optimal excess primal risk\footnote{Primal function is one of the common measures in minimax problems. Please refer to Section \ref{sec:preliminary} for details.} bounds are $O(1/n)$ in strongly-convex-strongly-concave (SC-SC) minimax problems, which are derived by \cite{li2021high}. In this paper, we derive $O(1/n^2)$ excess primal risk bounds with some reasonable assumptions, which is, to the best of our knowledge, the optimal results in minimax problems.

Since excess risk can be bounded by generalization error and optimization error, most of existing studies such as \cite{palaniappan2016stochastic,hsieh2019convergence,lin2020gradient,luo2020stochastic} were focused on iteration complexity for certain \text{algorithms}, which only considered the optimization error. In contrast, the generalization performance analysis is less considered, which is an important measure to foresee their prediction behavior after training and limits better excess risk bounds.

In this paper, our goal is to improve the generalization error bounds and further derives better excess risk bounds. We use local methods to consider variance information and obtain a tighter generalization error bound comparing with Rademacher complexity method~\cite{zhang2022uniform}. Note that we introduce a novel ``uniform localized convergence'' framework using generic chaining developed by \cite{xu2020towards} to minimax problems which is different from traditional local Rademacher complexity technique~\cite{bartlett2002localized}.

Our contributions are summarized as follows:
\begin{itemize}
    \item We introduce local uniform convergence using new generic chaining techniques. Comparing with traditional uniform convergence results in \cite{zhang2022uniform}, we derive sharper generalization bounds measured by the gradients of primal functions for NC-SC minimax problems. It provides problem independent results that can be used in various minimax \text{algorithms}.
    
    \item Under the Polyak-Lojasiewicz condition for the outer layer, we provide dimension-independent results and remove the dimension of parameters $d$ from our generalization bound when the sample size $n$ is large enough, which is, to our knowledge, the first result in minimax problems.
    
    \item We extend our main theorems into various \text{algorithms} such as ESP, GDA, SGDA. We gain $O(1/n^2)$ bounds with further assumptions that the optimal population risk is small. To our best knowledge, it is the first time to gain $O(1/n^2)$ for PL-SC minimax problems with expectation version and the first excess primal risk bounds for $O(1/n^2)$ with high probability for SC-SC settings.

\end{itemize}

This paper is organized as follows. In Section \ref{sec:relatework}, we review the related work. In Section \ref{sec:preliminary}, we introduce the notations and assumptions about the problems. Section \ref{sec:main-section} presents our improved generalization error bounds. Then we apply our main theorems into various \text{algorithms} and give sharper bounds for different settings in Section \ref{sec:application}. Section \ref{sec:conclusion} concludes our paper. All the proofs in our paper are given in Appendix.

\section{Related Work}
\label{sec:relatework}
\textbf{Minimax optimization.} Minimax optimization analysis has been widely studied in different settings. For example, one of the most popular SGDA \text{algorithm} and its variants have been analyzed in several recent works including \cite{palaniappan2016stochastic,hsieh2019convergence} for SC-SC cases, \cite{nedic2009subgradient,nemirovski2009robust} for convex-concave (C-C) cases, \cite{lin2020gradient,luo2020stochastic,yan2020optimal,rafique2022weakly} for NC-SC problems, \cite{thekumparampil2019efficient,yan2020optimal} for nonconvex-concave (NC-C) cases and \cite{loizou2020stochastic,liu2021first,yang2020global} for nonconvex-nonconcave (NC-NC) minimax optimization problems. All these works focus on the iteration complexity (or the gradient complexity) of the \text{algorithms}, which only proved the optimization error bounds for the sum of $T$ iteration's gradient of primal empirical function in expectation. Recently \cite{li2021high,lei2021stability} gave optimization bounds with high probability for Primal-Dual risk. We notice that the optimization error of the gradients of primal functions with high probability haven't been studied yet.

\textbf{Algorithmic stability.} Algorithmic stability is a classical approach in generalization analysis, which was presented by \cite{rogers1978finite}. It gave the generalization bound by analyzing the sensitivity of a particular learning \text{algorithm} when changing one data point in the dataset. Modern framework of stability analysis was established by \cite{bousquet2002stability}, where they presented an important concept called uniform stability. Since then, a lot of works based on uniform stability have emerged. On one hand, generalization bounds with algorithmic stability have been significantly improved by \cite{bousquet2020sharper,feldman2018generalization,feldman2019high,klochkov2021stability}. On the other hand, 
different algorithmic stability measures such as uniform argument stability \cite{liu2017algorithmic,bassily2020stability}, on average stability \cite{shalev2010learnability,kuzborskij2018data}, collective stability \cite{london2016stability} have been developed. For minimax problems, many useful stability measures have also been extended, for example, weak stability \cite{lei2021stability}, argument stability \cite{lei2021stability,li2021high}, and uniform stability \cite{lei2021stability,li2021high,zhang2021generalization,farnia2021train,ozdaglar2022good}. Most of them focused on the expectation generalization bounds and only \cite{lei2021stability,li2021high} established some high probability bounds. 

\textbf{Uniform convergence.} Uniform convergence is another popular approach in statistical learning theory to study generalization bounds \cite{fisher1922mathematical,vapnik1999overview,van2000asymptotic}. The main idea is to bound the generalization gap by its supremum over the whole (or a subset) of the hypothesis space via some space complexity measures, such as VC dimension, covering number and Rademacher complexity. For finite-dimensional problem, \cite{kleywegt2002sample} provided that the generalization error is $O(\sqrt{d/n})$ depended on the sample size $n$ and the dimension of parameters $d$ in high probability. For nonconvex settings, \cite{mei2018landscape} showed that the empirical of generalization error is $O(\sqrt{d/n})$. \cite{xu2020towards} developed a novel ``uniform localized convergence'' framework using generic chaining for the minimization problems and \cite{li2021improved} extended it to analyze stochastic algorithms. In minimax problems, \cite{zhang2022uniform} established the first uniform convergence and showed that the empirical generalization error of the gradients for primal functions is $O(\sqrt{d/n})$ under NC-SC settings.

\section{Preliminaries}
\label{sec:preliminary}
Let $\gX \in \R^d$ and $\gY \in \R^{d'}$ be two nonempty closed convex parameters spaces. Let $\sP$ be a probability measure defined on a sample space $\gZ$. We consider the following minimax optimization problem
\begin{align}
\label{eq:population-minimax}
    \min_{\rvx \in \gX} \max_{\rvy \in \gY} F(\rvx, \rvy) := \E_{\rvz \sim \sP} [f(\rvx, \rvy; \rvz)],
\end{align}
where $f : \gX \times \gY \times \gZ \rightarrow \R$ is continuously differentiable and Lipschitz smooth jointly in $\rvx$ and $\rvy$ for any $\rvz$. This above minimax objective called as the population minimax problem represents an expectation of a cost function $f(\rvx, \rvy; \rvz)$ for minimization variable $\rvx$, maximization variable $\rvy$ and data variable $\rvz$. In this paper, we focus on the NC-SC problem which means that $f$ is nonconvex in $\rvx$ and strongly concave in $\rvy$. Obviously, our goal is to gain the optimal solution $(\rvx^*,\rvy^*)$ to (\ref{eq:population-minimax}). Since the distribution $\sP$ is unavailable, we can only gain a dataset $S = \{ \rvz_1, \ldots ,\rvz_n \}$ drawn $n$ times independently from $\sP$. 
Therefore, we solve the following empirical minimax problem instead
\begin{align}
\label{eq:empirical-minimax}
    \min_{\rvx \in \gX} \max_{\rvy \in \gY} F_S(\rvx, \rvy) := \frac1n \sum_{i=1}^{n} f(\rvx, \rvy; \rvz_i).
\end{align}

Next we introduce one of the common measures in minimax problems called primal functions.
\begin{definition}[Primal (empirical/population) function]
     The primal empirical function and the primal population function are given by
    \begin{align*}
        \Phi_S(\rvx) := \max_{\rvy \in \gY} F_S(\rvx, \rvy) 
        \quad
        \text{and}
        \quad
        \Phi(\rvx) := \max_{\rvy \in \gY} F(\rvx, \rvy).
        \end{align*}
\end{definition}

Since $F_S$ and $F$ are nonconvex in $\rvx$, it is difficult to find the global optimal solution in general. In practice, we design an \text{algorithm} $\gA$ that finds an $\epsilon$-stationary point
\begin{align}
    \| \nabla \Phi(\gA_\rvx(S)) \| \leq \epsilon,
\end{align}
where $\gA_\rvx(S)$ is the $x$-component of the output using any \text{algorithm} $\gA(S) = (\gA_\rvx(S),\gA_\rvy(S))$ for solving (\ref{eq:empirical-minimax}). Then the optimization error for solve the population minimax problem (\ref{eq:population-minimax}) can be decomposed into two terms:

\begin{align*}
    \| \nabla \Phi(\gA_\rvx(S)) \| & \leq \| \nabla \Phi_S(\gA_\rvx(S)) \| \\
    & \quad + \| \nabla \Phi(\gA_\rvx(S)) - \nabla \Phi_S(\gA_\rvx(S)) \|,
\end{align*}
where the first term on the right-hand-side corresponds to the optimization error of solving the empirical minimax problem (\ref{eq:empirical-minimax}) and the second term corresponds to the generalization error of the gradients for primal functions. The above inequality satisfies from the triangle inequality. 

Let $\| \cdot \|$ be the Euclidean norm for simplicity and $B(\rvx_0,R):= \{\rvx \in \R^d: \| \rvx - \rvx_0 \| \leq R \}$ denotes a ball with center $\rvx_0 \in \R^d$ and radius $R$. For the closed convex set $\gX$, we assume that there is a radius $R_1$ such that $\gX \in B(\rvx^*, R_1)$. Let $\gA(S) := ( \gA_\rvx(S), \gA_\rvy(S) )$ denote the output of an \text{algorithm} $\gA$ for solving the empirical minimax problem (\ref{eq:empirical-minimax}) with dataset $S$ and $\nabla f = \left( \nabla_\rvx f, \nabla_\rvy f \right)$ denote the gradient of a function $f$.

\begin{definition}[Strongly convex function]
Let $\mu_\rvy > 0$. A differentiable function $g: \gW \rightarrow \R$ is called $\mu$-strongly-convex in $\rvw$ if the following inequality holds for every $\rvw_1$, $\rvw_2$:
\begin{align*}
    g(\rvw_1) - g(\rvw_2) \geq \langle \nabla g(\rvw_2), \rvw_1-\rvw_2 \rangle + \frac{\mu}{2} \| \rvw_1 - \rvw_2 \|^2,
\end{align*}
we say $g$ is $\mu$-strongly-concave if $-g$ is $\mu$-strongly-convex.
\end{definition}

\begin{definition}[Smooth function]
\label{def:smooth}
Let $\beta > 0$. A function $f: \gX \times \gY \times \gZ \rightarrow \R$ is $\beta$-smooth in $(\rvx, \rvy)$ if the function is continuous differentiable and for any $\rvx_1, \rvx_2 \in \gX$, $\rvy_1, \rvy_2 \in \gY$ and $\rvz \in \gZ$, $f(\rvx, \rvy; \rvz)$ satisfies
\begin{small}
\begin{align*}
    \begin{Vmatrix} 
    \begin{pmatrix}
    \nabla_\rvx f(\rvx_1, \rvy_1; \rvz) - \nabla_\rvx f(\rvx_2, \rvy_2; \rvz) \\
    \nabla_\rvy f(\rvx_1, \rvy_1; \rvz) - \nabla_\rvy f(\rvx_2, \rvy_2; \rvz)
    \end{pmatrix}
    \end{Vmatrix}
    \leq \beta
    \begin{Vmatrix}
    \begin{pmatrix}
    \rvx_1 - \rvx_2 \\
    \rvy_1 - \rvy_2
    \end{pmatrix}
    \end{Vmatrix}.
\end{align*}    
\end{small}
\end{definition}
\begin{assumption}[Nonconvex-strongly-concave minimax problem]
\label{assumption:NC-SC}
In order to obtain meaningful conclusions, we make the following assumptions:
\begin{itemize}
    \item Let $\mu_\rvy > 0$. The function $f(\rvx, \rvy; \rvz)$ is $\mu_\rvy$-strongly concave in $\rvy \in \gY$ for any $\rvx \in \gX$ and $\rvz \in \gZ$. 
    \item The function $f(\rvx, \rvy; \rvz)$ is $\beta$-smooth in $(\rvx, \rvy) \in \gX \times \gY$ for any $\rvz$.
    \item $\gX$ and $\gY$ are compact convex sets, which means that there exist constants $D_\gX, D_\gY > 0$ such that for any $\rvx \in \gX$, $\| \rvx \|^2 \leq D_\gX$ and for any $\rvy \in \gY$, $\| \rvy \|^2 \leq D_\gY$.
\end{itemize}
\end{assumption}
The first two assumptions in Assumption \ref{assumption:NC-SC} are standard in NC-SC minimax problems \cite{zhang2021generalization,farnia2021train,lei2021stability,li2021high} and the last one in Assumption \ref{assumption:NC-SC} is widely used in uniform convergence analysis \cite{kleywegt2002sample,davis2022graphical,zhang2022uniform}.

\begin{assumption}[Lipschitz continuity]
\label{assumption:f-lipschitz}
    Let $L > 0$, assume that for any $\rvx \in \gX$ and any $\rvy \in \gY$ respectively for any $\rvz$, the function $f(\rvx, \rvy; \rvz)$ satisfies
    \begin{align*}
        \| \nabla_\rvx f(\rvx, \rvy; \rvz) \| \leq L  
        \quad\text{and}\quad 
        \| \nabla_\rvy f(\rvx, \rvy; \rvz) \| \leq L.
    \end{align*}
\end{assumption}
Lipschitz assumption is also the standard assumption and widely used in literature such as \cite{zhang2021generalization,farnia2021train,lei2021stability,li2021high}. But we need to emphasize that our main Theorem \ref{theorem:main} and Theorem \ref{theorem:nablaf-without-d} do not require the Lipschitz assumption. Instead, we introduce a weaker assumption called Bernstein condition in minimax problems.
\begin{definition}[Bernstein condition]
\label{def:bernstein-condition}
Given a random variable $X$ with mean $\mu = \E [X]$ and variance $ \sigma^2 = \E [X^2] - \mu^2$, we say that Bernstein’s condition holds if there exists $B > 0$ such that for all $k \geq 2, k \in \mathbb{N}$,
\begin{align*}
    \left| \E \left[(X - \mu)^k\right] \right| \leq \frac{k!}{2}  \sigma^2 B^{k-2}.
\end{align*}
\end{definition}
\begin{remark} \rm{}
\label{remark:bernstein-0}
    Bernstein condition has been widely used to obtain tail bounds that may be tighter than the Hoeffding bounds. It is easy to verify that any bounded variable satisfies Bernstein condition. Next, we introduce a straightforward generalization of Bernstein condition to minimax problems. We formally state these extension in the following assumptions.
\end{remark}

\begin{assumption}
\label{assumption:minimax-bernstein-condition}
In minimax problems, the function $f(\rvx, \rvy;\rvz)$ satisfies Bernstein condition in $\rvx^*$ for $\rvy^*$: there exists $B_{\rvx^*} > 0$ such that for all $k \geq 2, k \in \mathbb{N}$,
\begin{align*}
    \E\left[\left\| \nabla_\rvx f (\rvx^*, \rvy^*; \rvz) \right\|^k\right] \leq \frac{k!}{2} \E\left[\| \nabla_\rvx f (\rvx^*, \rvy^*; \rvz) \|^2\right] B_{\rvx^*}^{k-2}.
\end{align*}    
And the function $f(\rvx, \rvy;\rvz)$ satisfies Bernstein condition in $\rvy^*$ for $\rvx^*$: there exists $B_{\rvy^*} > 0$ such that for all $k \geq 2, k \in \mathbb{N}$,
\begin{align*}
    \E\left[\left\| \nabla_\rvy f (\rvx^*, \rvy^*; \rvz) \right\|^k\right] \leq \frac{k!}{2} \E\left[\| \nabla_\rvy f (\rvx^*, \rvy^*; \rvz) \|^2\right] B_{\rvy^*}^{k-2},
\end{align*}
\end{assumption}
\begin{remark} \rm{}
\label{remark:bernstein}
    We can easily obtain that Assumption \ref{assumption:f-lipschitz} can derive Assumption \ref{assumption:minimax-bernstein-condition}. For example, if function $f$ is $L$-Lipschitz continuous, then $\|\nabla_\rvx f(\rvx, \rvy;\rvz) \| \leq L$. Thus for any $\rvx \in \gX, \rvy \in \gY$ and for all $k \geq 2, k \in \mathbb{N}$, we have $\E \left[\| \nabla_\rvy f(\rvx, \rvy;\rvz) \|^k \right] \leq \frac{k!}{2} \E \left[\| \nabla_\rvy f(\rvx, \rvy; \rvz) \|^2\right] L^{k-2}$, which means that the function $f$ satisfies Bernstein condition for any $\rvx, \rvy$. Similarly, $\E \left[\| \nabla_\rvx f(\rvx, \rvy;\rvz) \|^k \right] \leq \frac{k!}{2} \E \left[\| \nabla_\rvx f(\rvx, \rvy; \rvz) \|^2\right] L^{k-2}$ can be easily derived.
    Moreover, Bernstein condition is milder than the bounded assumption of random variables and is also satisfied by various unbounded variables. For example, a random variable is sub-exponetial if it satisfies Bernstein condition \cite{wainwright2019high}. Please refer to \cite{wainwright2019high} for more discussions. 
    Furthermore, Bernstein condition assumption is pretty mild since $B_{\rvy^*}$ and $B_{\rvx^*}$ only depends on gradients at $(\rvx^*, \rvy^*)$. 
\end{remark}

\section{Uniform Localized Convergence and Generalization Bounds}
\label{sec:main-section}
Uniform convergence of the gradients for primal functions measures the deviation between the gradients of the primal population function $\nabla \Phi(\rvx)$ and the gradients of the primal empirical function $\nabla \Phi_S(\rvx)$. In this section, we provide the sharper uniform convergence of the gradients for primal functions comparing with \cite{zhang2022uniform}.
\begin{theorem}
\label{theorem:main}
Under Assumption \ref{assumption:NC-SC} and \ref{assumption:minimax-bernstein-condition}, for any $\delta \in (0,1)$, with probability at least $1-\delta$, it holds for all $\rvx \in \gX$ that
\begin{small}
\begin{align*}
    &  \| \nabla \Phi(\rvx) - \nabla \Phi_S(\rvx) \| \\
    \leq & \frac{\beta}{\mu_\rvy} \left(
    \sqrt{\frac{2 \E \|\nabla_\rvy f(\rvx^*, \rvy^*; \rvz)\|^2 \log{\frac{8}{\delta}}}{n}} 
        + \frac{B_{\rvy^*} \log{\frac8\delta}}{n}
    \right) \\
    & + \sqrt{\frac{2 \E \|\nabla_\rvx f(\rvx^*, \rvy^*; \rvz)\|^2 \log{\frac{8}{\delta}}}{n}} 
    + \frac{B_{\rvx^*} \log{\frac8\delta}}{n} \\
    & + \frac{C\beta(\mu_\rvy+\beta)}{\mu_\rvy} \frac{(\mu_\rvy+\beta)}{\mu_\rvy}
    \max \left\{ \| \rvx - \rvx^* \|, \frac{1}{n} \right\} \times \\
    &   \left(
    \sqrt{\frac{d+\log{\frac{16\log_2(\sqrt{2}R_1 n+1)}{\delta}}}{n}} + 
    \frac{d+\log{\frac{16\log_2(\sqrt{2}R_1 n+1)}{\delta}}}{n}
    \right),
\end{align*}    
\end{small}
where $C$ is a absolute constant.
\end{theorem}
There is only one uniform convergence of gradients for primal functions in minimax problems given in \cite{zhang2022uniform}. Here is their main theorem in NC-SC settings.
\begin{theorem}[Theorem in \cite{zhang2022uniform}]
\label{theorem:zhang}
    Under Assumption \ref{assumption:NC-SC} and \ref{assumption:f-lipschitz}, we have
    \begin{align*}
        \E \left[ 
        \max_{\rvx \in \gX} \| \nabla \Phi(\rvx) - \nabla \Phi_S(\rvx) \|
        \right]
        = \tilde{O} \left(\frac{L(\mu_\rvy + \beta)}{\mu_\rvy}\sqrt{\frac{d}{n}}\right),
    \end{align*}
    where $\tilde{O}(\cdot)$ hides logarithmic factors.
\end{theorem}
\begin{remark} \rm{}
\label{remark:compare}
    We now compare our uniform convergence of gradient for primal functions with \cite{zhang2022uniform}. Firstly, our result is the only one with high-probability format. Besides, we successfully relax the assumptions. Theorem \ref{theorem:zhang} requires the Lipschitz continuity assumption, while our result only needs Bernstein condition assumption. Please refer to Remark \ref{remark:bernstein-0} Remark \ref{remark:bernstein} for the detailed comparison between these assumptions. Then, the factor in Theorem \ref{theorem:zhang} is $\frac{L(\mu_\rvy + \beta)}{\mu_\rvy}$, while our result in Theorem \ref{theorem:main} is $\frac{C\beta(\mu_\rvy+\beta)}{\mu_\rvy} \frac{(\mu_\rvy+\beta)}{\mu_\rvy} \max \left\{ \| \rvx - \rvx^* \|, \frac{1}{n} \right\}$, not involving the term $L$, which may be very large and even infinite without Lipschitz continuity assumption. Finally, while \cite{zhang2022uniform} studied the worst-case upper bounds on the parameters, results based on generic chaining yield upper bound related to the parameters. As shown Theorem \ref{theorem:main}, we have the term $\max\{ \|\rvx -\rvx^* \|,\frac1n \}$ before the term $O(\sqrt{d/n})$, indicating that our results improve as the calculated parameters of algorithms approach the optimal solution. For example, when $\| \rvx- \rvx^*\| = O\left(\frac{1}{\sqrt{n}}\right)$, our result is  $O(\sqrt{d}/n)$. In some optimal scenario, when $\| \rvx- \rvx^*\| \leq \frac1n$, we can attain the best results.
\end{remark}

\begin{remark} \rm{}
    In fact, \cite{zhang2022uniform} derived an expectation generalization error for primal functions in minimax problems using complexity. Naturally, we want to use local methods to introduce variance information and obtain a tighter upper bound.  A straightforward idea is that we can continue with the traditional localized approach and solve the problem with covering numbers \cite{bartlett2002localized}. However, these technologies require additional bounded assumptions (Assumption~\ref{assumption:f-lipschitz}), or need certain distributional assumptions for unbounded condition. For example, \cite{mei2018landscape} introduced the ``Hessian statistical noise'' assumption when using covering numbers. Fortunately, \cite{xu2020towards} developed a novel ``uniform localized convergence'' framework using generic chaining for the minimization problems and \cite{li2021improved} extended it to analyze stochastic algorithms.

    This novel framework can not only relax the bounded (or specific distribution) assumptions but also impose fewer restrictions on the surrogate function for the localized method, enabling us to design the measurement functional to achieve a sharper bound. Consequently, we introduce this remarkable framework into minimax problems. Our generalization bound in Theorem \ref{theorem:main} uses weaker assumptions comparing with \cite{zhang2022uniform} and is sharper in some conditions due to our utilization of variance information.

    Introducing this new framework into minimax problems is not straightforward. \cite{zhang2022uniform} indeed established a connection between inner and outer layers with the loss of primal functions, but we need do this with a new generic chaining approach. Furthermore, it is noteworthy that the optimal point $\rvy^*(\rvx) := \argmax_{\rvy \in \gY} F(\rvx, \rvy)$ for a given $\rvx$ differs from $\rvy_S^*(\rvx) := \argmax_{\rvy \in \gY} F_S(\rvx, \rvy)$, thus introducing an additional error term $\| \rvy^*(\rvx) - \rvy_S^*(\rvx) \|$. Compared to \cite{zhang2022uniform}, they only need to bound this term with $O(1/\sqrt{n})$. But we need to reach the upper bound to $O(1/n)$ under certain assumptions.
\end{remark}

Next, we provide a dimension-free uniform convergence of gradients for primal functions when the PL condition is satisfied. Firstly, we introduce the extension of the PL condition to minimax problems used in \cite{guo2020fast,yang2020global}.
\begin{assumption}[$\rvx$-side $\mu_\rvx$-Polyak-Lojasiewicz condition]
\label{assumption:PL-condition-x}
    For any $\rvy \in \gY$, the function $F(\rvx, \rvy)$ satisfies the $\rvx$-side $\mu_\rvx$-Polyak-Lojasiewicz (PL) condition with parameter $\mu_\rvx > 0$ on all $\rvx \in \gX$ if
    \begin{align*}
        F(\rvx, \rvy) - \inf_{\rvx'} F(\rvx', \rvy) \leq \frac{1}{2\mu_\rvx} \| \nabla_\rvx F(\rvx,\rvy) \|^2.
    \end{align*}
\end{assumption}
\begin{remark} \rm{}
    Numerous studies have been conducted on deep learning to provide evidence for the validity of the PL condition in risk minimization problems. This condition has been demonstrated to hold either globally or locally in certain networks with specific structural, activation, or loss function characteristics \cite{hardt2016identity,li2018learning,arora2018convergence,charles2018stability,du2018gradient,allen2019convergence}. For instance, \cite{du2018gradient} has exhibited that if a two-layer neural network possesses a sufficiently wide width, the PL condition is upheld within a ball centered at the initial solution, and the global optimum is situated within this same ball. Additionally, \cite{allen2019convergence} has further demonstrated that in overparameterized deep neural networks utilizing ReLU activation, the PL condition is applicable to a global optimum located in the vicinity of a random initial solution. 
\end{remark}

\begin{theorem}
\label{theorem:nablaf-without-d}
Under Assumption \ref{assumption:NC-SC} and \ref{assumption:minimax-bernstein-condition}, assume that the population risk $F(\rvx,\rvy)$ satisfies Assumption \ref{assumption:PL-condition-x} with parameter $\mu_\rvx$ and let $c = \max \{16C^2, 1\}$. For any $\delta \in (0,1)$ when 
$ n \geq \frac{c\beta^2(\mu_\rvy+\beta)^4(d+ \log{\frac{16\log_2{\sqrt{2}R_1 n + 1}}{\delta}})}{\mu_\rvy^4\mu_\rvx^2}$, 
with probability at least $1-\delta$, it holds for all $\rvx \in \gX$ that
\begin{small}
\begin{align*}
    &  \| \nabla \Phi(\rvx) - \nabla \Phi_S(\rvx) \| \\
     \leq & \| \nabla \Phi_S(\rvx) \| 
        + 2\sqrt{\frac{2 \E \left[ \|\nabla_\rvx f(\rvx^*, \rvy^*; \rvz)\|^2 \right] \log{\frac{8}{\delta}}}{n}} + \frac{2 B_{\rvx^*} \log{\frac8\delta}}{n} \\
         + & \frac{\mu_\rvx}{n}  
        + \frac{2\beta}{\mu_\rvy} \left(
    \sqrt{\frac{2 \E \left[ \|\nabla_\rvy f(\rvx^*, \rvy^*; \rvz)\|^2 \right]\log{\frac{8}{\delta}}}{n}} 
    + \frac{B_{\rvy^*} \log{\frac8\delta}}{n}
    \right).
\end{align*}    
\end{small}
\end{theorem}
\begin{remark} \rm{}
\label{remark:main-PL}
    The following inequality can be easily derived using triangle inequality and  Cauchy–Bunyakovsky–Schwarz inequality.
    \begin{equation}
    \label{eq:theorem-nablaf-withoud-d-new}
        \begin{aligned}
            & \Phi(\rvx) - \Phi(\rvx^*) \\
    \leq &  \frac{8 \| \nabla \Phi_S(\rvx) \|^2}{\mu_\rvx}   + \frac{16 \E \left[ \|\nabla_\rvx f(\rvx^*, \rvy^*; \rvz)\|^2 \right] \log{\frac{8}{\delta}}}{\mu_\rvx n} \\ 
    & \quad +
    \frac{16 \beta^2\E \left[ \|\nabla_\rvy f(\rvx^*, \rvy^*; \rvz)\|^2 \right] \log{\frac{8}{\delta}}}{\mu_\rvx \mu_\rvy^2 n} \\
    & \qquad  + \frac{ 2 \left(
        \frac{2\beta B_{\rvy^*}}{\mu_\rvy} \log{\frac{8}{\delta}} + 2B_{\rvx^*} \log{\frac{8}{\delta}} + \mu_\rvx
        \right)^2
        }{\mu_\rvx n^2}.
        \end{aligned}
    \end{equation}
We can easily derive (\ref{eq:theorem-nablaf-withoud-d-new}) from Theorem \ref{theorem:nablaf-without-d} to gain the excess primal risk bound, where $ \| \nabla \Phi_S(\rvx) \| $ is the empirical optimization error of the primal function. In Theorem \ref{theorem:nablaf-without-d} and (\ref{eq:theorem-nablaf-withoud-d-new}), $\| \nabla \Phi_S(\rvx) \|$ can be very tiny since most famous optimization \text{algorithms} such as GDA and SGDA, can optimize it small enough. The term $\E \left[ \|\nabla_\rvx f(\rvx^*, \rvy^*; \rvz)\|^2 \right]$ and $\E \left[ \|\nabla_\rvy f(\rvx^*, \rvy^*; \rvz)\|^2 \right]$ can be also tiny since they only depend on the the gradient of the optima $(\rvx^*,\rvy^*)$ w.r.t $\rvx$ and $\rvy$. We further analyze these two terms $\E \left[ \|\nabla_\rvx f(\rvx^*, \rvy^*; \rvz)\|^2 \right]$ and $\E \left[ \|\nabla_\rvy f(\rvx^*, \rvy^*; \rvz)\|^2 \right]$ using ``self-bounding'' property for smooth functions \cite{srebro2010optimistic} and considering specific \text{algorithms} in Section \ref{sec:application}, which can derive to $O(1/n^2)$ bounds. Thus, comparing with Theorem \ref{theorem:zhang} in \cite{zhang2022uniform}, this uniform localized convergence bound is clearly tighter when relaxing Lipschitz continuity (Assumption \ref{assumption:f-lipschitz}) and considering PL condition (Assumption \ref{assumption:PL-condition-x}). Additionally, uniform convergence often implies results with a square-root dependence on the dimension $d$ such as Theorem \ref{theorem:main} and \cite{zhang2022uniform}. Another distinctive improvement of Theorem \ref{theorem:nablaf-without-d} is that we remove the dimension $d$ when the population risk $F(\rvx,\rvy)$ satisfies the $\rvx$-side PL condition and the sample size $n$ is large enough. 
\end{remark}

\section{Application}
\label{sec:application}

\subsection{Empirical Saddle Point}
Empirical saddle point (ESP) \text{algorithm}, which is also known as sample average approximation (SAA)~\cite{zhang2021generalization} refers to (\ref{eq:empirical-minimax}). We denote $(\hat{\rvx}^*, \hat{\rvy}^*)$ as one of the ESP solution to (\ref{eq:empirical-minimax}). Then we can provide some important theorems in this subsection.

\begin{theorem}
\label{theorem:esp-nabla}
Suppose the empirical saddle point $(\hat{\rvx}^*,\hat{\rvy}^*)$ exists and Assumption \ref{assumption:NC-SC} and \ref{assumption:minimax-bernstein-condition} hold, for any $\delta \in (0,1)$, with probability at least $1-\delta$, we have 
\begin{align*}
    &  \| \nabla \Phi(\hat{\rvx}^*) \| 
    = O\left(
    \sqrt{
    \frac{d + \log{\frac{\log{n}}{\delta}}}{n}
    }
    \right)
\end{align*}
\end{theorem}
\begin{remark} \rm{}
    When Assumption \ref{assumption:NC-SC} and \ref{assumption:minimax-bernstein-condition} hold, Theorem \ref{theorem:esp-nabla} shows that the \emph{population} optimization error $\| \nabla \Phi(\hat{\rvx}^*) \|$ is  $O\left(\sqrt{\frac{d + \log{\frac{1}{\delta}}}{n}}\right)$ ($\log{n}$ is small and can be ignored typically). Note that this result doesn't require the Lipschitz continuity assumption (Assumption \ref{assumption:f-lipschitz}). Although it may be hard to find $(\hat{\rvx}^*, \hat{\rvy}^*)$ in NC-SC minimax problems, it is still meaningful when assuming the ESP $(\hat{\rvx}^*, \hat{\rvy}^*)$ has been found.
\end{remark}

\begin{theorem}
\label{theorem:esp-fast-rate}
Suppose Assumption \ref{assumption:NC-SC} and \ref{assumption:minimax-bernstein-condition} hold. Assume that the population risk $F(\rvx,\rvy)$ satisfies Assumption \ref{assumption:PL-condition-x} with parameter $\mu_\rvx$. For any $\delta \in (0,1)$, with probability at least $1-\delta$, when $ n \geq \frac{c\beta^2(\mu_\rvy+\beta)^4(d+ \log{\frac{16\log_2{\sqrt{2}R_1 n + 1}}{\delta}})}{\mu_\rvy^4\mu_\rvx^2}$, where $c$ is an absolute constant, we have 
\begin{small}
\begin{align*}
    \Phi(\hat{\rvx}^*) - \Phi(\rvx^*)  \leq &         
        \frac{12 \beta^2\E \|\nabla_\rvy f(\rvx^*, \rvy^*; \rvz)\|^2 \log{\frac{8}{\delta}}}{\mu_\rvx \mu_\rvy^2 n} \\
    & + \frac{12 \E \|\nabla_\rvx f(\rvx^*, \rvy^*; \rvz)\|^2  \log{\frac{8}{\delta}}}{\mu_\rvx n} \\ 
    & + \frac{ 3 \left(
        \frac{2\beta B_{\rvy^*}}{\mu_\rvy} \log{\frac{8}{\delta}} + 2B_{\rvx^*} \log{\frac{8}{\delta}} + \mu_\rvx
        \right)^2
        }{2 \mu_\rvx n^2}.
\end{align*}
\end{small}
Furthermore, if we assume the function $f(\rvx,\rvy;\rvz)$ is non-negative, we have
\begin{align*}
    \Phi(\hat{\rvx}^*) - \Phi(\rvx^*) = O\left(
        \frac{\Phi(\rvx^*) \log{\frac{1}{\delta}}}{n} + \frac{\log^2{\frac{1}{\delta}}}{n^2}
        \right).
\end{align*}
When $ \Phi(\rvx^*) = O\left(\frac{1}{n}\right)$, we have 
\begin{align*}
    \Phi(\hat{\rvx}^*) - \Phi(\rvx^*) = O \left(
    \frac{\log^2{\frac{1}{\delta}}}{n^2}
    \right).
\end{align*}
\end{theorem}
\begin{remark} \rm{}
    Theorem \ref{theorem:esp-fast-rate} shows that when the population minimax risk $F(\rvx,\rvy)$ satisfies $\rvx$-side PL condition, we can provide a sharper excess risk bound for primal function, which can be $O(1/n^2)$. Note that the optimal population primal function $\Phi(\rvx^*) = O\left( 1/n \right)$ is a very common assumption in many researches such as \cite{srebro2010optimistic,zhang2017empirical,liu2018fast,zhang2019stochastic,lei2020fine}, which is natural because $F(\rvx^*, \rvy^*)$ is the minimal population risk. Now we compare our results with recent related work \cite{li2021improved}, which studied the general machine learning settings for $f(\rvw)$ under PL condition. Their empirical risk minimizer (ERM) excess risk bound provided $O\left( 1/n^2 \right)$ order rates. We analyze the excess risk with primal functions in minimax problems and our result for ESP is $O\left(1/n^2\right)$, which is the same order as theirs.
\end{remark}

\subsection{Gradient Descent Ascent}
Gradient descent ascent (GDA) presented in \text{Algorithm} \ref{algo:gda} is one of the most popular \text{algorithms} and has been widely used in minimax problems. In this subsection, we provide the population optimization error bound and the excess risk bounds of primal functions with the two-timescale GDA \text{algorithm} which is harder to analyze compared to GDMax and multistep GDA \cite{lin2020gradient}.

\begin{figure}
\centering
\begin{minipage}{0.45\textwidth}
\begin{algorithm}[H]
\begin{algorithmic}[1]
  \STATE {\bfseries Input:} {$(\rvx_1, \rvy_1)=(\bm{0}, \bm{0})$, step sizes $\eta_\rvx > 0, \eta_\rvy > 0$ and dataset $S = \{\rvz_1, \dots ,\rvz_n\}$}
  
  \FOR{$t=1,\dots,T$}
  \STATE  update $\rvx_{t+1} = \rvx_{t} - \eta_\rvx \nabla_{\rvx} F_S(\rvx_t, \rvy_t) $ 
  \STATE  update $\rvy_{t+1} = \rvy_{t} + \eta_\rvy \nabla_{\rvy} F_S(\rvx_t, \rvy_t) $
  \ENDFOR
  \caption{Two-timescale GDA for minimax problem}
  \label{algo:gda}
\end{algorithmic}
\end{algorithm}
\end{minipage}
\hspace{.1cm}
\begin{minipage}{0.45\textwidth}
\begin{algorithm}[H]
  \begin{algorithmic}[1]
  \STATE {\bfseries Input:} {$(\rvx_1, \rvy_1) = (\bm{0}, \bm{0})$, step sizes $\{\eta_{\rvx_t}\}_t > 0$, $\{\eta_{\rvy_t}\}_t >0$ and dataset $S = \{\rvz_1, \dots ,\rvz_n\}$}
  
  \FOR{$t=1,\dots,T$}
    \STATE update $\rvx_{t+1} = \rvx_{t} - \eta_{\rvx_t} \nabla_{\rvx} f(\rvx_t, \rvy_t;\rvz_{i_t}) $ 
    \STATE update $\rvy_{t+1} = \rvy_{t} + \eta_{\rvy_t} \nabla_{\rvy} f(\rvx_t, \rvy_t;\rvz_{i_t}) $
    \ENDFOR
  \caption{Two-timescale SGDA for minimax problem}
  \label{algo:sgda}
\end{algorithmic}
\end{algorithm}
\end{minipage}
\end{figure}

\begin{theorem}
\label{theorem:gda-nabla}
Suppose Assumption \ref{assumption:NC-SC} and \ref{assumption:f-lipschitz} hold. Let $\{\rvx_t\}_t$ be the sequence produced by \text{Algorithm} \ref{algo:gda} with the step sizes chosen as $\eta_\rvx = \frac{1}{16(\frac{\beta}{\mu} + 1)^2\beta}$ and $\eta_\rvy= \frac{1}{\beta}$, for any $\delta \in (0,1)$, with probability at least $1-\delta$, we have 
\begin{align*}
    & \frac{1}{T}\sum_{t=1}^T \| \nabla \Phi(\rvx_t) \|^2 \\
    \leq & O\left( \frac{1}{T} \right) 
    + 
    O \left( 
    \frac{d+\log{\frac{16\log_2(\sqrt{2}R_1 n+1)}{\delta}}}{n} T
    \right).
\end{align*}    
Furthermore, when $T \asymp O\left(\sqrt{\frac{n}{d}}\right)$, we have
\begin{align*}
    \frac{1}{T}\sum_{t=1}^T \| \nabla \Phi(\rvx_t) \|^2 \leq O\left(
    \frac{d + \log{\frac{\log{n}}{\delta}} }{\sqrt{nd}}
    \right).
\end{align*}
\end{theorem}
\begin{remark} \rm{}
    Theorem \ref{theorem:gda-nabla} also gives the \emph{population} optimization error which reveals that we need to balance the empirical optimization error and the generalization error for GDA. According to the results, the iterative complexity of \text{Algorithm} \ref{algo:gda} should be chosen as $T \asymp O\left(\sqrt{\frac{n}{d}}\right)$, which achieves the optimal population optimization error of primal function.

    In comparison to Theorem \ref{theorem:esp-nabla}, Theorem \ref{theorem:gda-nabla} derives into population optimization error w.r.t GDA, which is much more difficult. To establish population optimization error, we need to bound the empirical optimization error, an area where no research has been conducted in NC-SC settings with high probability. One possible approach is to construct the martingale difference sequence of step $T$ for primal functions, yet this constitutes a separate topic warranting further exploration. Theorem \ref{theorem:gda-nabla} aims to directly apply Theorem \ref{theorem:main} to GDA. Comparing with Theorem \ref{theorem:nablaf-without-d}, Theorem \ref{theorem:gda-nabla} only necessitates smooth and Lipschitz conditions (Assumption \ref{assumption:NC-SC} and \ref{assumption:f-lipschitz}) and doesn't require PL conditions.
\end{remark}

Next, we provide the excess risk of primal functions $ \Phi(\bar{\rvx}_T) - \Phi(\rvx^*) $ for \text{Algorithm} \ref{algo:gda}, where $\bar{\rvx}_T = \frac{1}{T} \sum_{t=1}^T {\rvx_t} $. We need to know the empirical optimization error $\| \nabla \Phi_S(\bar{\rvx}_T) \|$ firstly. 

Unfortunately, although the generalization bounds we proved are in NC-SC settings, we require the SC-SC assumptions to derive the empirical optimization error bound of primal functions,  to gain the high probability bound. We relax this SC-SC assumption in Appendix \ref{sec:expectation-bounds} using existing optimization error bound with expectation format.

\begin{table*}[]
\centering
\renewcommand\arraystretch{1.2}
\begin{tabular}{|c|c|c|c|c|}

\hline
                 Reference & \text{Algorithm} & Assumption & Measure & Bounds \\ \hline
\multirow{3}{*}{\cite{lei2021stability}} & \multirow{2}{*}{SGDA} & C-SC, Lip, S & (E.) $\Phi(\bar{\rvx}_T) - \Phi(\rvx^*)$ & $O(1/\sqrt{n})$ \\ \cline{3-5} 
                     &                       & C-SC, Lip, S  &  (HP.) $\Phi(\bar{\rvx}_T) - \Phi(\rvx^*)$   & $O(\log{n}/\sqrt{n})$ \\ \cline{2-5}
                     &      AGDA             & PL-SC, Lip, S  &  (E.) $\Phi(\bar{\rvx}_T) - \Phi(\rvx^*)$ & $O\left(n^{-\frac{c\beta+1}{2c\beta+1}}\right)$ \\ \hline
\multirow{3}{*}{\cite{li2021high}} & ESP & SC-SC, Lip, S, LN & (HP.) $\Phi(\hat{\rvx}^*) - \Phi(\rvx^*)$  & $O(\log{n}/n)$ \\ \cline{2-5} 
                    & GDA & SC-SC, Lip, S, LN & (HP.) $\Phi(\bar{\rvx}_T) - \Phi(\rvx^*)$  & $O(\log{n}/n)$ \\ \cline{2-5} 
                    & SGDA & SC-SC, Lip, S, LN & (HP.) $\Phi(\bar{\rvx}_T) - \Phi(\rvx^*)$  & $O(\log{n}/n)$ \\ \hline
\multirow{6}{*}{This work}
                  & ESP & PL-SC, B, S, LN & (HP.) $\Phi(\hat{\rvx}^*) - \Phi(\rvx^*)$ &  $O(1/n^2)$ \\ \cline{2-5}
                  & \multirow{2}{*}{GDA} & PL-SC, B, S, LN & (E.) $\Phi(\bar{\rvx}_T) - \Phi(\rvx^*)$ &  $O(1/n^2)$ \\ \cline{3-5}
                  &  & SC-SC, B, S, LN & (HP.) $\Phi(\bar{\rvx}_T) - \Phi(\rvx^*)$ &  $O(1/n^2)$ \\ \cline{2-5}
                  & \multirow{2}{*}{SGDA} & PL-SC, B, S, LN & (E.) $\Phi(\bar{\rvx}_T) - \Phi(\rvx^*)$ &  $O(1/n^2)$ \\ \cline{3-5}
                  &  & SC-SC, Lip, S, LN & (HP.) $\Phi(\bar{\rvx}_T) - \Phi(\rvx^*)$ & $O(1/n^2)$ \\ \cline{2-5}
                  & AGDA & PL-SC, B, S, LN & (E.) $\Phi(\bar{\rvx}_T) - \Phi(\rvx^*)$ & $O(1/n^2)$ \\ \hline
\end{tabular}
\caption{Summary of the results. These bounds are established by choosing optimal iterate number $T$.}
\label{table:summary}
\end{table*}

\begin{definition}
    A function $g: \gX \times \gY \rightarrow \R$ is $\mu_\rvx$-strongly-convex-$\mu_\rvy$-strongly-concave if $g(\cdot, \rvy)$ is $\mu_\rvx$-strongly-convex for any $\rvy \in \gY$ and $g(\rvx, \cdot)$ is $\mu_\rvy$-strongly-concave for any $\rvx \in \gX$.
\end{definition}
\begin{assumption}[Strongly-convex-strongly-concave minimax problem]
\label{assumption:SC-SC}
       Assume Assumption \ref{assumption:NC-SC} holds and let $\mu_\rvx > 0, \mu_\rvy > 0$. The function $f(\rvx, \rvy; \rvz)$ is $\mu_\rvx$-strongly-convex-$\mu_\rvy$-strongly-concave in $\rvy \in \gY$ for any $\rvx \in \gX$ and $\rvz \in \gZ$.
\end{assumption}
\begin{remark} \rm{}
    Assumption \ref{assumption:SC-SC} is commonly used in SC-SC problems \cite{zhang2021generalization,li2021high}. We require this assumption to derive the empirical \textbf{optimization error} bound of primal functions. 
    The detailed proofs of the optimization error bound $\| \nabla \Phi_S(\bar{\rvx}_T) \|$ are given in Section \ref{sec:appendix-GDA} for GDA and in Section \ref{sec:appendix-SGDA} for SGDA.
\end{remark}

\begin{theorem}
\label{theorem:gda-fast-rate}
Suppose Assumption \ref{assumption:minimax-bernstein-condition} and \ref{assumption:SC-SC} hold. Let $\{\rvx_t\}_t$ be the sequence produced by \text{Algorithm} \ref{algo:gda} and $\bar{\rvx}_T = \frac{1}{T} \sum_{t=1}^T {\rvx_t}$ with the step sizes chosen as $\eta_\rvx = \frac{1}{16(\frac{\beta}{\mu} + 1)^2\beta}$ and $\eta_\rvy= \frac{1}{\beta}$, for any $\delta \in (0,1)$, with probability at least $1-\delta$, when $T \asymp n$ and $ n \geq \frac{c\beta^2(\mu_\rvy+\beta)^4(d+ \log{\frac{16\log_2{\sqrt{2}R_1 n + 1}}{\delta}})}{\mu_\rvy^4\mu_\rvx^2}$, where $c$ is an absolute constant, we have 
\begin{align*}
     \Phi(\bar{\rvx}_T) - \Phi(\rvx^*) 
        =  O \Bigg( \frac{\E \left[ \|\nabla_\rvx f(\rvx^*, \rvy^*; \rvz)\|^2 \right] \log{\frac{1}{\delta}}}{n}  \\
         + \frac{\E \left[ \|\nabla_\rvy f(\rvx^*, \rvy^*; \rvz)\|^2 \right] \log{\frac{1}{\delta}}}{n} + \frac{\log^2{\frac{1}{\delta}}}{n^2} \Bigg).
\end{align*}
Furthermore, Let $T \asymp n^2$. Assume the function $f(\rvx,\rvy;\rvz)$ is non-negative, we have  
\begin{align*}
    \Phi(\bar{\rvx}_T) - \Phi(\rvx^*) = O\left(
        \frac{\Phi(\rvx^*) \log{\frac{1}{\delta}}}{n} + \frac{\log^2{\frac{1}{\delta}}}{n^2}
        \right).
\end{align*}
When $ \Phi(\rvx^*) = O\left(\frac{1}{n}\right)$, we have 
\begin{align*}
    \Phi(\bar{\rvx}_T) - \Phi(\rvx^*) = O \left(
    \frac{\log^2{\frac{1}{\delta}}}{n^2}
    \right).
\end{align*}
\end{theorem}
\begin{remark} \rm{}
    Theorem \ref{theorem:gda-fast-rate} shows that the excess risk for primal functions can be bound by $O\left(1/n^2\right)$ comparing with the optimal result $O(1/n)$ given in \cite{li2021high} when $n$ is large enough. Note that we require the SC-SC assumption to derive the empirical optimization error. If we give this bound in expectation, we can relax the SC-SC assumption with $\rvx$-side PL-strongly-concave assumption instead.
\end{remark}

\subsection{Stochastic Gradient Descent Ascent}

We now analyze the excess risk bound of primal functions for stochastic gradient descent ascent (SGDA). The algorithmic scheme that we study is two-timescale SGDA ($\eta_\rvx \neq \eta_\rvy$) with variable step sizes, presented in \text{Algorithm} \ref{algo:sgda} which is more nature in real problems.

\begin{theorem}
\label{theorem:sgda-fast-rate} 
Suppose Assumption \ref{assumption:f-lipschitz} and \ref{assumption:SC-SC} hold, let $\{\rvx_t\}_t$ be the sequence produced by \text{Algorithm} \ref{algo:sgda} and $\bar{\rvx}_T = \frac{1}{T} \sum_{t=1}^T {\rvx_t}$ with the step sizes chosen as $\eta_{\rvx_t} = \frac{1}{\mu_\rvx(t+t_0)}$ and $\eta_{\rvy_t} = \frac{1}{\mu_\rvy(t+t_0)} $, for any $\delta \in (0,1)$, with probability at least $1-\delta$, when $T \asymp n^2$ and $ n \geq \frac{c\beta^2(\mu_\rvy+\beta)^4(d+ \log{\frac{16\log_2{\sqrt{2}R_1 n + 1}}{\delta}})}{\mu_\rvy^4\mu_\rvx^2}$, where $c$ is an absolute constant, we have 
\begin{align*}
    \Phi(\bar{\rvx}_T) - \Phi(\rvx^*) 
        = O \Bigg( \frac{\E \left[ \|\nabla_\rvx f(\rvx^*, \rvy^*; \rvz)\|^2 \right] \log{\frac{1}{\delta}}}{n}  \\
        + \frac{\E \left[ \|\nabla_\rvy f(\rvx^*, \rvy^*; \rvz)\|^2 \right] \log{\frac{1}{\delta}}}{n} + \frac{\log^2{\frac{1}{\delta}}}{n^2} \Bigg).
\end{align*}
Furthermore, Let $T \asymp n^4$. Assume the function $f(\rvx,\rvy;\rvz)$ is non-negative, we have  
\begin{align*}
    \Phi(\bar{\rvx}_T) - \Phi(\rvx^*) = O\left(
        \frac{\Phi(\rvx^*) \log{\frac{1}{\delta}}}{n} + \frac{\log^2{\frac{1}{\delta}}}{n^2}
        \right).
\end{align*}
When $ \Phi(\rvx^*) = O\left(\frac{1}{n}\right)$, we have 
\begin{align*}
    \Phi(\bar{\rvx}_T) - \Phi(\rvx^*) = O \left(
    \frac{\log^2{\frac{1}{\delta}}}{n^2}
    \right).
\end{align*}
\end{theorem}

\begin{remark} \rm{}
    Theorem \ref{theorem:sgda-fast-rate} reveals that under the SC-SC settings, the excess risk bound can be $O(1/n^2)$ comparing with the optimal result $O(1/n)$ given in \cite{li2021high}. Similarly, since the SC-SC assumption is required to derive the empirical optimization bound, we can relax the assumptions when we only need expectation bounds instead of high probability bounds.
\end{remark}

Table \ref{table:summary} gives the summary of existing results. AGDA is alternating gradient descent ascent \text{algorithm} proposed in \cite{yang2020global}. Lip means Lipschitz continuity. S means smoothness. B means Bernstein condition. LN means low noise condition. PL-SC means $\rvx$-side PL condition strongly concave settings. E. means expectation results. HP. means high probability results. Since most of existing works on optimization errors are the expectation format, and our high probability results of generalization error can be transformed into the expectation results. so we give the proofs of the expectation result to relax some assumptions such as SC-SC condition in Appendix \ref{sec:expectation-bounds}.

\section{Conclusion}
\label{sec:conclusion}
In this paper, we provide the improved generalization bounds for minimax problems with uniform localized convergence. We firstly provide a sharper bound measured by the gradients of primal functions with weaker assumptions in NC-SC settings. Then we provide dimension-independent results under PL-SC condition. Finally we extend our main theorems into various \text{algorithms} to reach the optimal excess primal risk bounds. Our excess primal risk bounds are $O(1/n^2)$ in SC-SC settings with high probability format and $O(1/n^2)$ in PL-SC settings with expectation version. We notice that most optimization works focused on the gradient complexity with expectation results. It would be interesting to give the optimization error of $\bar{\rvx}_T$ or even $\rvx_T$ with high probability under weaker conditions (for example in PL-SC and even NC-SC settings). Combining with our generalization work, we can get a tighter excess primal risk bound with weaker conditions.

\section*{Ethical Statement}

Here are no ethical issues.

\section*{Acknowledgments}
We thank the anonymous reviewers for their valuable and constructive suggestions and comments. This work is supported by the Beijing Natural Science Foundation (No.4222029); the National Natural Science Foundation of China (N0.62076234); the National Key Research and Development Project (No.2022YFB2703102); the ``Intelligent Social Governance Interdisciplinary Platform, Major Innovation \& Planning Interdisciplinary Platform for the ``Double-First Class'' Initiative, Renmin University of China''; the Beijing Outstanding Young Scientist Program (NO.BJJWZYJH012019100020098); the Public Computing Cloud, Renmin University of China; the Fundamental Research Funds for the Central Universities, and the Research Funds of Renmin University of China (NO.2021030199), the Huawei-Renmin University joint program on Information Retrieval: the Unicom Innovation Ecological Cooperation Plan; the CCF-Huawei Populus Grove Fund; and the Fundamental Research Funds for the Central Universities, and the Research Funds of Renmin University of China (NO.23XNH030).


\begin{small}
\bibliographystyle{named}
\bibliography{ijcai24}
\end{small}

\newpage
\onecolumn

\appendix

\section{Additional definitions and lemmata}

\begin{lemma}[Bernstein's inequality \cite{dirksen2015tail}]
\label{lemma:Bernstein-inequality}
    Let $X_1, \ldots, X_n$ be real-valued, independent, mean-zero random variables and suppose that for some constants $\sigma, B > 0$,
    \begin{align*}
        \frac1n \sum_{i=1}^{n} \E |X_i|^k \leq \frac{k!}{2} \sigma^2 B^{k-2}, \quad k = 2, 3, \ldots
    \end{align*}
    Then, $\forall \delta \in (0,1)$, with probability at least $1-\delta$
    \begin{align}
    \label{eq:Bernstein-inequality}
        \left| \frac{1}{n} \sum_{i=1}^n X_i \right| \leq \sqrt{\frac{2\sigma^2 \log{\frac{2}{\delta}}}{n}} + \frac{B \log{\frac{2}{\delta}}}{n}.
    \end{align}
\end{lemma}

\begin{lemma}[A variant of the ``uniform localized convergence'' argument \cite{xu2020towards}]
\label{lemma:a-variant-uniform-localized-convergence-argument}
Let $\sP$ be a probability measure defined on a sample space $\gZ$ and $\sP_n$ be the corresponding empirical probability measure. For a function class $\gH = \{ h_f: f \in \gF \}$ and functional $T: \gF \rightarrow [0, R]$, assume there is a function $\psi(r;\delta)$ (possibly depending on the samples), which is non-decreasing with respect to $r$ and satisfies that $\forall \delta \in (0,1)$, $\forall r \in [0, R]$, with probability at least $1-\delta$,
\begin{align*}
    \sup_{f \in \gF: T(f) \leq r}(\sP - \sP_n) h_f \leq \psi(r; \delta).
\end{align*}
Then, given any $\delta \in (0,1)$ and $r_0 \in (0, R]$, with probability at least $1-\delta$, for all $f \in \gF$,
\begin{align*}
    (\sP - \sP_n) h_f \leq \psi \left(2T(f) \vee r_0 ; \frac{\delta}{C_{r_0}} \right),
\end{align*}
where $C_{r_0} = 2\log_2{\frac{2R}{r_0}}$.
\end{lemma}

\begin{definition}[$\text{Orlicz}_{\alpha}$ norm \cite{dirksen2015tail}]
\label{def:orlicz-alpha-norm}
    For every $\alpha \in (0, +\infty)$ we define the $\text{Orlicz}_{\alpha}$ norm of $a$ random $u$:
    \begin{align*}
        \| u \|_{\text{Orlicz}_{\alpha}} = \inf \{ K>0: \E \exp( (|u|/K)^{\alpha}) \leq 2 \}.
    \end{align*}
    A random variable (or vector) $X \in \R^d$ is $K$-sub-Gaussian if $\forall \bm\lambda \in \R^d$, we have
    \begin{align*}
        \| \bm\lambda^\mathrm{T} X \|_{\text{Orlicz}_{2}} = K \| \bm\lambda \|_2.
    \end{align*}
    A random variable (or vector) $X \in \R^d$ is $K$-sub-exponential if $\forall \bm\lambda \in \R^d$, we have
    \begin{align*}
        \| \bm\lambda^\mathrm{T} X \|_{\text{Orlicz}_{1}} = K \| \bm\lambda \|_2.
    \end{align*}
\end{definition}

\begin{definition}[$\text{Orlicz}_{\alpha}$ processes \cite{dirksen2015tail}]
    Let $\{ X_f \}_{f \in \gF}$ be a sequence of random variables. $\{ X_f \}_{f \in \gF}$ is called an $\text{Orlicz}_{\alpha}$ process for a metric $\mathsf{metr}(\cdot, \cdot)$ on $\gF$ if
    \begin{align*}
        \| X_{f_1} - X_{f_2} \|_{\text{Orlicz}_{\alpha}} \leq \mathsf{metr}(f_1, f_2), \quad \forall f_1, f_2 \in \gF.
    \end{align*}
    Typically, the $\text{Orlicz}_{2}$ process is called ``process with sub-Gaussian increments'' and the $\text{Orlicz}_{1}$ process is called ``process with sub-exponential increments''.
\end{definition}

\begin{definition}[Mixed sub-Gaussian-sub-exponential increments \cite{dirksen2015tail}]
\label{def:mix-sub-gaussian-sub-exp}
    We say a process $(X_{\theta})_{\theta \in \Theta}$ has mixed sub-Gaussian-sub-exponential increments with respect to the pair $(\mathsf{metr}_1, \mathsf{metr}_2)$ if for all $\theta_1, \theta_2 \in \Theta$,
    \begin{align*}
        Prob( \| X_{\theta_1} - X_{\theta_2} \| \geq \sqrt{u} \cdot \mathsf{metr}_2(\theta_1, \theta_2) + u \cdot \mathsf{metr}_1(\theta_1, \theta_2)) \leq 2 e^{-u}, \forall u \geq 0,
    \end{align*}
    where ``$Prob$'' means probability.
\end{definition}

\begin{definition}[Talagrand's $\gamma_\alpha$-functional \cite{dirksen2015tail}]
    A sequence $F=(\gF_n)_{n \geq 0}$ of subsets of $\gF$ is called admissible if $|\gF_0| = 1$ and $|\gF_n| \leq 2^{2^n}$ for all $n \geq 1$. For any $0 < \alpha < \infty$, the $\gamma_\alpha$-functional of $(\gF, \mathsf{metr})$ is defined by 
    \begin{align*}
        \gamma_{\alpha} (F, d) = \inf_F \sup_{f \in \gF} \sum_{n=0}^{\infty} 2^{\frac n\alpha} \mathsf{metr}(f, \gF_n),
    \end{align*}
    where the infimum is taken over all admissible sequences and we write $\mathsf{metr}(f, \gF_n) = \inf_{s \in \gF_n} \mathsf{metr}(f, s)$.
\end{definition}

\begin{lemma}[Bernstein's inequality for sub-exponential random variables \cite{wainwright2019high}]
\label{lemma:bernstein-ineq-for-sub-exp}
    If $X_1, \ldots X_n$ are sub-exponential random variables, the Bernstein's inequality in Lemma \ref{lemma:Bernstein-inequality} holds with
    \begin{align*}
        \sigma^2 = \frac1n \sum_{i=1}^n \| X_i \|_{\text{Orlicz}_1}^2, \quad 
        B = \max_{1 \leq i \leq n} \| X_i \|_{\text{Orlicz}_1}.
    \end{align*}
\end{lemma}

\begin{lemma}[Vector Bernstein's inequality \cite{pinelis1994optimum,smale2007learning,xu2020towards}]
\label{lemma:vector-bernstein-ineq}
Let $\{ X_i \}_{i=1}^n$ be a sequence of i.i.d. random variables taking values in a real separable Hilbert space. Assume that $\E[X_i] = \mu$, $\E[\|X_i - \mu \|^2] = \sigma^2$, $\forall 1\leq i \leq n$, we say that vector Bernstein's condition with parameter B holds if for all $1 \leq i \leq n$,
\begin{align}
    \E[\|X_i - \mu \|^k] \leq \frac12 k! \sigma^2 B^{k-2}, 
    \quad \forall 2 \leq k \leq n.
\end{align}
If this condition holds, then for all $\delta \in (0,1)$, with probability at least $1-\delta$ we have
\begin{align}
    \left\| \frac1n \sum_{i=1}^{n} X_i - \mu \right\| \leq \sqrt{\frac{2\sigma^2 \log(\frac2\delta)}{n}} + \frac{B\log{\frac{2}{\delta}}}{n}. 
\end{align}
\end{lemma}

\begin{definition}[Covering number \cite{wainwright2019high}]
    Assume $(\gM, \mathsf{metr})$ is a metric space and $\gF \subseteq \gM$. For any $\epsilon > 0$, a set $\gF_c$ is called an $\epsilon$-cover of $\gF$ if for any $f \in \gF$ we have an element $g \in \gF_c$ such that $\mathsf{metr} (f, g) \leq \epsilon$. We denote $N(\gF, \mathsf{metr}, \epsilon)$ the covering number as the cardinality of the minimal $\epsilon$-cover of $\gF$:
    \begin{align*}
        N(\gF, \mathsf{metr}, \epsilon) = \min \{ 
        |\gF_c|:\gF_c \text{is an } \epsilon \text{-cover of } \gF
        \}.
    \end{align*}
\end{definition}

\begin{lemma}[Dudley's integral bound for $\gamma_\alpha$ functional \cite{talagrand1996majorizing}]
\label{lemma:dudley-integral-bound}
    There exist a constant $C_\alpha$ depending only on $\alpha$ such that
    \begin{align*}
        \gamma_\alpha (\gF, \mathsf{metr}) \leq C_{\alpha} \int_0^{+\infty} (\log{N(\gF, \mathsf{metr}, \epsilon)})^{\frac1\alpha} d\epsilon
    \end{align*}
\end{lemma}

\begin{lemma}[Generic chaining for a process with mixed tail increments in \cite{dirksen2015tail}]
\label{lemma:generic-chaining-for-mixed-tail-increments}
    If $(X_f)_{f\in\gF}$ has mixed sub-Gaussian-sub-exponential increments with respect to the pair $(\mathsf{metr}_1, \mathsf{metr}_2)$, there are absolute constants $c, C > 0$ such that for $\forall \delta \in (0,1)$, with probability at least $1-\delta$
    \begin{align*}
        \sup_{\theta \in \Theta} \| X_f - X_{f_0} \| \leq 
        C(\gamma_2 (\gF, \mathsf{metr}_2) + \gamma_1 (\gF, \mathsf{metr}_1)) + \\
        c \left( 
        \sqrt{\log{\frac{1}{\delta}}} + \sup_{f_1, f_2 \in \gF} [\mathsf{metr}_2 (f_1, f_2)] + \log{\frac{1}{\delta}} \sup_{f_1, f_2 \in \gF} [\mathsf{metr}_1 (f_1, f_2)]
        \right).
    \end{align*}
\end{lemma}

\begin{lemma}[``Self-bounding'' property for smooth function \cite{srebro2010optimistic}]
\label{lemma:smooth-self-bounding-property}
    For a $\beta$-smooth and non-negative function $f: \rvw \rightarrow \R$, for all $\rvw \in \gW$:
    \begin{align*}
        \| \nabla f(\rvw) \| \leq \sqrt{4\beta f(\rvw)}
    \end{align*}
\end{lemma}

\section{Some basic lemmata in minimax problems}

\begin{lemma}[Smoothness for primal function \cite{nouiehed2019solving}]
\label{lemma:primal-function-is-smooth}
Suppose Assumption \ref{assumption:NC-SC} holds, then the function $\Phi(\rvx)$ and $\Phi_S(\rvx)$ is $\left(\beta + \frac{\beta^2}{\mu_\rvy} \right)$-smooth.
\end{lemma}

\begin{lemma}[PL condition for primal function \cite{yang2020global}]
\label{lemma:primal-function-is-PL}
For NC-SC setting, suppose Assumption \ref{assumption:NC-SC} holds. Assume that the population risk $F(\rvx,\rvy)$ satisfies Assumption \ref{assumption:PL-condition-x} with parameter $\mu_\rvx$, then function $\Phi(\rvx)$ satisfies the PL condition with $\mu_\rvx$, which means that for all $\rvx \in \gX$
\begin{align*}
    \Phi(\rvx) - \Phi(\rvx^*) \leq \frac{1}{2\mu_\rvx} \| \nabla \Phi(\rvx) \|^2.
\end{align*}

For SC-SC setting, suppose Assumption \ref{assumption:SC-SC} holds, then function $\Phi(\rvx)$ and $ \Phi_S(\rvx)$ satisfy the PL condition with $\mu_\rvx$, which means that for all $\rvx \in \gX$
\begin{align*}
    \Phi(\rvx) - \Phi(\rvx^*) \leq \frac{1}{2\mu_\rvx} \| \nabla \Phi(\rvx) \|^2
    \quad
    \text{and}
    \quad
    \Phi_S(\rvx) - \Phi_S(\rvx^*) \leq \frac{1}{2\mu_\rvx} \| \nabla \Phi_S(\rvx) \|^2.
\end{align*}
\end{lemma}

\begin{definition}
\label{define:definition-y*-ys*}
    For a given $\rvx$, the empirical optimal point $\rvy^*(\rvx)$ and the population optimal point $\rvy^*_S(\rvx)$ are given as follows,
    \begin{align*}
        \rvy^*(\rvx) := \argmax_{\rvy \in \gY} F(\rvx, \rvy)
        \quad
        \text{and}
        \quad
        \rvy^*_S(\rvx) := \argmax_{\rvy \in \gY} F_S(\rvx, \rvy).
    \end{align*}
\end{definition}
\begin{remark}
    According to above definition, we can easily derive the following equations that $\rvy^* = \rvy^*(\rvx^*) = \argmax_{\rvy \in \gY} F(\rvx^*, \rvy)$ if $(\rvx^*,\rvy^*)$ is the solution to (\ref{eq:population-minimax}) and $\hat{\rvy}^* = \rvy^*_S(\hat{\rvx}^*) = \argmax_{\rvy \in \gY} F_S(\hat{\rvx}^*, \rvy)$ if $(\hat{\rvx}^*,\hat{\rvy}^*)$ is the solution to (\ref{eq:empirical-minimax}).
\end{remark}

\begin{lemma}[Concentration of $\rvy^*$ \cite{zhang2022uniform}]
\label{lemma:concentration-of-y}
For $\rvy^*(\rvx)$ and $\rvy^*_S(\rvx)$ defined in Definition~\ref{define:definition-y*-ys*}, with Assumption \ref{assumption:NC-SC}, we have $\forall \rvx \in \gX$,
\begin{align*}
    \left\| \rvy^*(\rvx) - \rvy_S^*(\rvx)\right\|
    \leq \frac{1}{\mu_\rvy} \left\| \nabla_\rvy F (\rvx, \rvy^*(\rvx)) - \nabla_\rvy F_S (\rvx, \rvy^*(\rvx)) \right\|.
\end{align*}
\end{lemma}

\begin{lemma}[\cite{zhang2021generalization}]
\label{lemma:y2x}
    For $\rvy^*$ and $\rvy^*_S$ defined in Definition~\ref{define:definition-y*-ys*}, with Assumption \ref{assumption:NC-SC}, then for  $\forall \rvx_1, \rvx_2 \in \gX$,
\begin{align*}
    \left\| \rvy^*(\rvx_1) - \rvy^*(\rvx_2)\right\|
    \leq \frac{\beta}{\mu_\rvy} \left\|\rvx_1 - \rvx_2 \right\|.
\end{align*}
\end{lemma}

\begin{lemma}
\label{lemma:sub-exp}
Suppose a function $f: \gX \times \gY \times \gZ \rightarrow \R$ is $\beta$-smooth in $(\rvx, \rvy)$ and the function $f$ is $\mu_\rvy$-strongly concave in $\rvy \in \gY$ for any $\rvx \in \gX$ and $\rvz \in \gZ$. Then we have $ \frac{\rvu^\mathrm{T}(\nabla_\rvx f(\rvx_1, \rvy^*(\rvx_1); \rvz) - \nabla_\rvx f(\rvx_2, \rvy^*(\rvx_2); \rvz))}{\| \rvx_1 - \rvx_2 \|} $ and $ \frac{\rvu^\mathrm{T}(\nabla_\rvy f(\rvx_1, \rvy^*(\rvx_1); \rvz) - \nabla_\rvy f(\rvx_2, \rvy^*(\rvx_2); \rvz))}{\| \rvx_1 - \rvx_2 \|} $ are $\frac{\beta(\mu_\rvy+\beta)}{\mu_\rvy\ln{2}}$-sub-exponential random vectors. That is for any unit vector $\rvu \in B(0,1)$ and $\rvx_1, \rvx_2 \in \gX$,
\begin{align*}
    \E \left\{
    \exp \left(
    \frac{  \left| \rvu^\mathrm{T} (\nabla_\rvx f(\rvx_1, \rvy^*(\rvx_1);\rvz) - \nabla_\rvx f(\rvx_2, \rvy^*(\rvx_2);\rvz)) \right|}{\frac{\beta(\mu_\rvy+\beta)}{\mu_\rvy\ln{2}} \| \rvx_1 - \rvx_2 \|}
    \right)
    \right\} \leq 2.
\end{align*}
For any unit vector $\rvu \in B(0,1)$ and $\rvx_1, \rvx_2 \in \gX$,
\begin{align*}
    \E \left\{
    \exp \left(
    \frac{  \left| \rvu^\mathrm{T} (\nabla_\rvy f(\rvx_1, \rvy^*(\rvx_1);\rvz) - \nabla_\rvy f(\rvx_2, \rvy^*(\rvx_2);\rvz)) \right|}{\frac{\beta(\mu_\rvy+\beta)}{\mu_\rvy\ln{2}} \| \rvx_1 - \rvx_2 \|}
    \right)
    \right\} \leq 2.
\end{align*}
\end{lemma}

\begin{proof}[Proof of Lemma \ref{lemma:sub-exp}]
    According to Definition \ref{def:smooth}, for any sample $\rvz \in \gZ$ and $\rvx_1, \rvx_2 \in \gX$, we have
    \begin{align*}
        & \| \nabla_\rvx f(\rvx_1, \rvy^*(\rvx_1);\rvz) - \nabla_\rvx f(\rvx_2, \rvy^*(\rvx_2);\rvz) \| \\
        \leq & \beta \| \rvx_1 - \rvx_2 \| + \beta \| \rvy^*(\rvx_1) - \rvy^*(\rvx_2) \| \\
        \leq & \beta \| \rvx_1 - \rvx_2 \| + \frac{\beta^2}{\mu_\rvy} \| \rvx_1 - \rvx_2 \| \\ 
        = & \frac{\beta(\mu_\rvy+\beta)}{\mu_\rvy} \| \rvx_1 - \rvx_2 \|,
    \end{align*}
where the first inequality uses the smoothness and the second inequality applies Lemma \ref{lemma:y2x}.

    Then, for any unit vector $\rvu \in B(0,1)$, we have
    \begin{align*}
        & \left| \rvu^\mathrm{T} (\nabla_\rvx f(\rvx_1, \rvy^*(\rvx_1);\rvz) - \nabla_\rvx f(\rvx_2, \rvy^*(\rvx_2);\rvz)) \right|  \\
        \leq & \| \rvu^\mathrm{T} \| \| \nabla_\rvx f(\rvx_1, \rvy^*(\rvx_1);\rvz) - \nabla_\rvx f(\rvx_2, \rvy^*(\rvx_2);\rvz) \| \\
        \leq & \frac{\beta(\mu_\rvy+\beta)}{\mu_\rvy} \| \rvx_1 - \rvx_2 \|,
    \end{align*}
    which implies
    \begin{align*}
        \frac{\left| \rvu^\mathrm{T} (\nabla_\rvx f(\rvx_1, \rvy^*(\rvx_1);\rvz) - \nabla_\rvx f(\rvx_2, \rvy^*(\rvx_2);\rvz)) \right|}{\frac{\beta(\mu_\rvy+\beta)}{\mu_\rvy} \| \rvx_1 - \rvx_2 \|} \leq 1.
    \end{align*}
    Then we get
    \begin{align*}
        \E \left\{
        \exp \left(
        \frac{  \left| \rvu^\mathrm{T} (\nabla_\rvx f(\rvx_1, \rvy^*(\rvx_1);\rvz) - \nabla_\rvx f(\rvx_2, \rvy^*(\rvx_2);\rvz)) \right|}{\frac{\beta(\mu_\rvy+\beta)}{\mu_\rvy\ln{2}} \| \rvx_1 - \rvx_2 \|}
        \right)
        \right\} \leq 2.
    \end{align*}

Similarly, according to Definition \ref{def:smooth}, for any sample $\rvz \in \gZ$ and $\rvx_1, \rvx_2 \in \gX$, we have
    \begin{align*}
        & \| \nabla_\rvy f(\rvx_1, \rvy^*(\rvx_1);\rvz) - \nabla_\rvy f(\rvx_2, \rvy^*(\rvx_2);\rvz) \| 
        \leq  \beta \| \rvx_1 - \rvx_2 \| + \beta \| \rvy^*(\rvx_1) - \rvy^*(\rvx_2) \|.
    \end{align*}
We can easily derive the following result
\begin{align*}
        \E \left\{
        \exp \left(
        \frac{  \left| \rvu^\mathrm{T} (\nabla_\rvy f(\rvx_1, \rvy^*(\rvx_1);\rvz) - \nabla_\rvy f(\rvx_2, \rvy^*(\rvx_2);\rvz)) \right|}{\frac{\beta(\mu_\rvy+\beta)}{\mu_\rvy\ln{2}} \| \rvx_1 - \rvx_2 \|}
        \right)
        \right\} \leq 2.
    \end{align*}
The proof is complete.
\end{proof}

\begin{lemma}
\label{lemma:nablafx-nablafx*}
    Suppose Assumption \ref{assumption:NC-SC} holds, we have the following inequality that for all $\rvx \in \gX$ and for any $\delta \in (0,1)$, with probability at lest $1-\delta$,
    \begin{align*}
        & \| (\nabla_\rvx F(\rvx, \rvy^*(\rvx)) - \nabla_\rvx F_S(\rvx, \rvy^*(\rvx)))
        - (\nabla_\rvx F(\rvx^*, \rvy^*) - \nabla_\rvx F_S(\rvx^*, \rvy^*))
        \| \\ 
        \leq & 
        \frac{C\beta(\mu_\rvy+\beta)}{\mu_\rvy} 
        \max \left\{ \| \rvx - \rvx^* \|, \frac{1}{n} \right\} \times
        \left(
        \sqrt{\frac{d+\log{\frac{4\log_2(\sqrt{2}R_1 n+1)}{\delta}}}{n}} + 
        \frac{d+\log{\frac{4\log_2(\sqrt{2}R_1 n+1)}{\delta}}}{n}
        \right),
    \end{align*}
    where $C$ is an absolute constant.
\end{lemma}

\begin{proof}[Proof of Lemma \ref{lemma:nablafx-nablafx*}]
    We define $\gV = \{ \rvv \in \R^d: \| \rvv \| \leq \max \{ R_1, \frac1n \} \}$. For all $(\rvx, \rvv) \in \gX \times \gV$, let $g_{(\rvx, \rvv)} (\rvz) = \left( \nabla_\rvx f(\rvx, \rvy^*(\rvx);\rvz) - \nabla_\rvx f(\rvx^*, \rvy^*; \rvz) \right)^\mathrm{T}\rvv$. Then for any $(\rvx_1, \rvv_1)$ and $(\rvx_2, \rvv_2) \in \gX \times \gV$, we define the following norm on the product space $\gX \times \gV$
    \begin{align*}
        \left\| (\rvx_1, \rvv_1) - (\rvx_2, \rvv_2) \right\|_{\gX \times \gV} 
        = \left( 
        \| \rvx_1 -\rvx_2 \|^2 + \| \rvv_1 - \rvv_2 \|^2
        \right)^{\frac12}.
    \end{align*}
    Then we define a ball on the product space $\gX \times \gV$ that $B(\sqrt{r}) = \{ (\rvx, \rvv) \in \gX \times \gV : \| \rvx - \rvx^* \|^2 + \| \rvv \|^2  \leq r \}$. Given any $(\rvx_1, \rvv_1)$ and $(\rvx_2, \rvv_2) \in B(\sqrt{r})$, we have the following decomposition
    \begin{align*}
        & g_{(\rvx_1, \rvv_1)} (\rvz) - g_{(\rvx_2, \rvv_2)} (\rvz) \\
        =& \left( \nabla_\rvx f(\rvx_1, \rvy^*(\rvx_1);\rvz) - \nabla_\rvx f(\rvx^*, \rvy^*; \rvz) \right)^\mathrm{T}\rvv_1  \\
        &\space - \left( \nabla_\rvx f(\rvx_2, \rvy^*(\rvx_2);\rvz) - \nabla_\rvx f(\rvx^*, \rvy^*; \rvz) \right)^\mathrm{T}\rvv_2 \\
        =& \left( \nabla_\rvx f(\rvx_1, \rvy^*(\rvx_1);\rvz) - \nabla_\rvx f(\rvx^*, \rvy^*; \rvz) \right)^\mathrm{T}(\rvv_1 -\ \rvv_2)  \\
        & \space + \left( \nabla_\rvx f(\rvx_1, \rvy^*(\rvx_1);\rvz) - \nabla_\rvx f(\rvx^*, \rvy^*; \rvz) \right)^\mathrm{T} \rvv_2 \\
        &\space - \left( \nabla_\rvx f(\rvx_2, \rvy^*(\rvx_2);\rvz) - \nabla_\rvx f(\rvx^*, \rvy^*; \rvz) \right)^\mathrm{T}\rvv_2 \\
        =& \left( \nabla_\rvx f(\rvx_1, \rvy^*(\rvx_1);\rvz) - \nabla_\rvx f(\rvx^*, \rvy^*; \rvz) \right)^\mathrm{T}(\rvv_1 -\ \rvv_2)  \\
        & \space + \left( \nabla_\rvx f(\rvx_1, \rvy^*(\rvx_1);\rvz) - \nabla_\rvx f(\rvx_2, \rvy^*(\rvx_2); \rvz) \right)^\mathrm{T} \rvv_2.
    \end{align*}
    Since $(\rvx_1, \rvv_1)$ and $(\rvx_2, \rvv_2) \in B(\sqrt{r})$, we have
    \begin{align}
    \label{eq:proof-1-1}
        \| \rvx_1 - \rvx^* \| \| \rvv_1 -\rvv_2 \| \leq \sqrt{r} \| \rvv_1 - \rvv_2 \| \leq \sqrt{r} \| (\rvx_1, \rvv_1) - (\rvx_2, \rvv_2) \|_{\gX \times \gV}.
    \end{align}
    Next, from Assumption \ref{assumption:NC-SC} and Lemma \ref{lemma:sub-exp}, we know that $ \frac{\nabla_\rvx f(\rvx_1, \rvy^*(\rvx_1); \rvz) - \nabla_\rvx f(\rvx_2, \rvy^*(\rvx_2); \rvz)}{\| \rvx_1 - \rvx_2 \|} $ is a $\frac{\beta(\mu_\rvy+\beta)}{\mu_\rvy\ln{2}}$-sub-exponential random vector for all $\rvx_1, \rvx_2 \in \gX$, which means that
    \begin{align}
    \label{eq:proof-1-2}
        \E \left\{
        \exp \left(
        \frac{(\nabla_\rvx f(\rvx_1, \rvy^*(\rvx_1); \rvz) - \nabla_\rvx f(\rvx^*, \rvy^*(\rvx^*); \rvz))^\mathrm{T}(\rvv_1 -\rvv_2)}{ \frac{\beta(\mu_\rvy+\beta)}{\mu_\rvy\ln{2}} \| \rvx_1 - \rvx^* \| \| \rvv_1 - \rvv_2 \| }
        \right)
        \right\} \leq 2.
    \end{align}
    We combine (\ref{eq:proof-1-1}) and (\ref{eq:proof-1-2}), according to Definition \ref{def:orlicz-alpha-norm}, we can easily deduce that $\left( \nabla_\rvx f(\rvx_1, \rvy^*(\rvx_1);\rvz) - \nabla_\rvx f(\rvx^*, \rvy^*; \rvz) \right)^\mathrm{T}(\rvv_1 -\ \rvv_2)$ is $\frac{\beta\sqrt{r}(\mu_\rvy+\beta)}{\mu_\rvy\ln{2}}\| (\rvx_1, \rvv_1) - (\rvx_2, \rvv_2) \|_{\gX \times \gV}$-sub-exponential. Similarly, we can derive that
    \begin{align*}
        \| \rvx_1 - \rvx_2 \| \| \rvv_2 \| \leq \sqrt{r} \| \rvx_1 - \rvx_2 \| \leq \sqrt{r} \| (\rvx_1, \rvv_1) - (\rvx_2, \rvv_2) \|_{\gX \times \gV}.
    \end{align*}
    Also, there holds that
    \begin{align*}
        \E \left\{
        \exp \left(
        \frac{(\nabla_\rvx f(\rvx_1, \rvy^*(\rvx_1); \rvz) - \nabla_\rvx f(\rvx_2, \rvy^*(\rvx_2); \rvz))^\mathrm{T}(\rvv_2)}{ \frac{\beta(\mu_\rvy+\beta)}{\mu_\rvy\ln{2}}  \| \rvx_1 - \rvx_2 \| \|\rvv_2\|}
        \right)
        \right\} \leq 2.
    \end{align*}
    Thus, we have that $ (\nabla_\rvx f(\rvx_1, \rvy^*(\rvx_1); \rvz) - \nabla_\rvx f(\rvx_2, \rvy^*(\rvx_2); \rvz))^\mathrm{T}(\rvv_2)$ is also $\frac{\beta \sqrt{r} (\mu_\rvy+\beta)}{\mu_\rvy\ln{2}} \| (\rvx_1, \rvv_1) - (\rvx_2, \rvv_2) \|_{\gX \times \gV}$-sub-exponential. 
    Now for any $(\rvx_1, \rvv_1)$ and $(\rvx_2, \rvv_2) \in B(\sqrt{r})$, we know
    \begin{equation}
    \label{eq:g-is-sub-exp}
        \begin{aligned}
            & \E \left\{
            \exp \left(
            \frac{g_{(\rvx_1,\rvv_1)}(\rvz) - g_{(\rvx_2,\rvv_2)}(\rvz)}{ \frac{2 \beta \sqrt{r} (\mu_\rvy+\beta)}{\mu_\rvy\ln{2}} \| (\rvx_1, \rvv_1) - (\rvx_2, \rvv_2) \|_{\gX \times \gV}}
            \right)
            \right\} \\
            \leq & \E \left\{
            \frac{1}{2}    \exp \left(
            \frac{(\nabla_\rvx f(\rvx_1, \rvy^*(\rvx_1); \rvz) - \nabla_\rvx f(\rvx^*, \rvy^*; \rvz))^\mathrm{T}(\rvv_1 - \rvv_2)}{ \frac{\beta \sqrt{r} (\mu_\rvy+\beta)}{\mu_\rvy\ln{2}} \| (\rvx_1, \rvv_1) - (\rvx_2, \rvv_2) \|_{\gX \times \gV}}
            \right)
            \right\} \\
            + & \E \left\{
            \frac{1}{2}    \exp \left(
            \frac{(\nabla_\rvx f(\rvx_1, \rvy^*(\rvx_1); \rvz) - \nabla_\rvx f(\rvx_2, \rvy^*(\rvx_2); \rvz))^\mathrm{T}(\rvv_2)}{ \frac{\beta \sqrt{r} (\mu_\rvy+\beta)}{\mu_\rvy\ln{2}} \| (\rvx_1, \rvv_1) - (\rvx_2, \rvv_2) \|_{\gX \times \gV}}
            \right)
            \right\}
            \leq 2,
        \end{aligned}
    \end{equation}
where the first inequality follows from Jensen's inequality. (\ref{eq:g-is-sub-exp}) implies that $g_{(\rvx_1,\rvv_1)}(\rvz) - g_{(\rvx_2,\rvv_2)}(\rvz)$ is a $\frac{2 \beta \sqrt{r} (\mu_\rvy+\beta)}{\mu_\rvy\ln{2}} \| (\rvx_1, \rvv_1) - (\rvx_2, \rvv_2) \|_{\gX \times \gV}$-sub-exponential random variable, for which we have
\begin{equation}
    \begin{aligned}
        \| g_{(\rvx_1,\rvv_1)}(\rvz) - g_{(\rvx_2,\rvv_2)}(\rvz) \|_{\text{Orlicz}_{1}}
        \leq \frac{2 \beta \sqrt{r} (\mu_\rvy+\beta)}{\mu_\rvy\ln{2}} \| (\rvx_1, \rvv_1) - (\rvx_2, \rvv_2) \|_{\gX \times \gV}.
    \end{aligned}
\end{equation}

From Bernstein inequality for sub-exponential variables (Lemma \ref{lemma:bernstein-ineq-for-sub-exp}), for any fixed $u \leq 0$ and $(\rvx_1. \rvv_1), (\rvx_2,\rvv_2) \in \gX \times \gV$, we have
\begin{equation}
    \begin{aligned}
    \label{eq:mix-sub-gaussian-sub-exp}
    Pr\left(
    |(\mathbb{P} - \mathbb{P}_n)[ g_{(\rvx_1,\rvv_1)}(\rvz) - g_{(\rvx_2,\rvv_2)}(\rvz) ] |
    \geq \frac{2 \beta \sqrt{r} (\mu_\rvy+\beta)}{\mu_\rvy\ln{2}} \| (\rvx_1, \rvv_1) - (\rvx_2, \rvv_2) \|_{\gX \times \gV} \sqrt{\frac{2u}{n}}  \right.\\
    \left. + \frac{2 u \beta \sqrt{r} (\mu_\rvy+\beta)}{n\mu_\rvy\ln{2}} \| (\rvx_1, \rvv_1) - (\rvx_2, \rvv_2) \|_{\gX \times \gV}
    \right) 
    \leq 2e^{-u}.        
    \end{aligned}
\end{equation}

According to Definition \ref{def:mix-sub-gaussian-sub-exp}, we know that $(\mathbb{P} - \mathbb{P}_n) g_{(\rvx,\rvv)}(\rvz)$ is a mixed sub-Gaussian-sub-exponential increments w.r.t. the metrics 
$(\frac{2 \beta \sqrt{r} (\mu_\rvy+\beta)}{n\mu_\rvy\ln{2}} \| \cdot \|_{\gX \times \gV},
\frac{2 \beta \sqrt{2r} (\mu_\rvy+\beta)}{\sqrt{n}\mu_\rvy\ln{2}} \| \cdot \|_{\gX \times \gV})$
from (\ref{eq:mix-sub-gaussian-sub-exp}). 

Then from the generic chaining for a process with mixed tail increments in Lemma \ref{lemma:generic-chaining-for-mixed-tail-increments}, we have the following inequality that for all $\delta \in (0,1)$, with probability at least $1-\delta$
\begin{align*}
    & \sup_{\|\rvx_1 - \rvx^* \|^2 + \| \rvv \|^2 \leq r}
    | (\mathbb{P} - \mathbb{P}_n) g_{(\rvx,\rvv)}(\rvz) | \\
    \leq & C \Bigg(
    \gamma_2 
    \left(
        B(\sqrt{r}), \frac{2 \beta \sqrt{2r} (\mu_\rvy+\beta)}{\sqrt{n}\mu_\rvy\ln{2}} \| \cdot \|_{\gX \times \gV}
    \right)+
    \gamma_1 
    \left(
        B(\sqrt{r}), \frac{2 \beta \sqrt{r} (\mu_\rvy+\beta)}{n\mu_\rvy\ln{2}} \| \cdot \|_{\gX \times \gV}
    \right) \\
     &
    + \frac{r\beta(\mu_\rvy+\beta)}{\mu_\rvy \ln{2}} \sqrt{\frac{\log{\frac1\delta}}{n}}
    + \frac{r\beta(\mu_\rvy+\beta)}{\mu_\rvy \ln{2}} \frac{\log{\frac1\delta}}{n}
    \Bigg).
\end{align*}

Next we use the Dudley's integral (Lemma \ref{lemma:dudley-integral-bound}) to bound the $\gamma_1$ and $\gamma_2$ functional. So there exists an absolute constant $C$ such that for any $\delta \in (0,1)$, with probability at least $1-\delta$
\begin{align}
\label{eq:sup-before-a-varient-uniform-localized-convergenc-argument}
    \sup_{\|\rvx_1 - \rvx^* \|^2 + \| \rvv \|^2 \leq r}
    | (\mathbb{P} - \mathbb{P}_n) g_{(\rvx,\rvv)}(\rvz) |
    \leq \frac{rC\beta(\mu_\rvy+\beta)}{\mu_\rvy \ln{2}}
    \left(
    \sqrt{\frac{d+\log{\frac{1}{\delta}}}{n}} + 
    \frac{d+\log{\frac{1}{\delta}}}{n}
    \right).
\end{align}

Then, the next step is to apply Lemma \ref{lemma:a-variant-uniform-localized-convergence-argument} to (\ref{eq:sup-before-a-varient-uniform-localized-convergenc-argument}). We denote $T(f) = \|\rvx - \rvx^* \|^2 + \| \rvv \|^2 $,
$\psi(r;\delta) = \frac{rC\beta(\mu_\rvy+\beta)}{\mu_\rvy \ln{2}} \left(
    \sqrt{\frac{d+\log{\frac{1}{\delta}}}{n}} + 
    \frac{d+\log{\frac{1}{\delta}}}{n}
    \right)$. 
Since $\|\rvx - \rvx^* \|^2 + \| \rvv \|^2 \leq R_1^2 + R_1^2 + \frac{1}{n^2}$, we set $R^2 = 2R_1^2 + \frac{1}{n^2}$ and $r_0 = \frac{2}{n^2}$. According to Lemma \ref{lemma:a-variant-uniform-localized-convergence-argument}, we have the following inequality that for any $\delta \in (0,1)$, with probability at least $1-\delta$, for all $\rvx \in \gX$ and $\rvv \in \gV$,
\begin{equation}
    \begin{aligned}
\label{eq:g-before-the-final-result}
        & (\mathbb{P} - \mathbb{P}_n) g_{(\rvx,\rvv)}(\rvz) \\
        = & (\mathbb{P} - \mathbb{P}_n) [(\nabla_\rvx f(\rvx, \rvy^*(\rvx);\rvz) - \nabla_\rvx f(\rvx^*, \rvy^*;\rvz))^\mathrm{T} \rvv ]\\
        \leq & \psi\left(
        \max \left\{ \| \rvx - \rvx^* \|^2 + \| \rvv \|^2, \frac{2}{n^2} \right\};
        \frac{\delta}{2\log_2(Rn^2)}
        \right) \\
        \leq & 
        \frac{C\beta(\mu_\rvy+\beta)}{\mu_\rvy \ln{2}} 
        \max \left\{ \| \rvx - \rvx^* \|^2 + \| \rvv \|^2, \frac{2}{n^2} \right\} \times
        \left(
        \sqrt{\frac{d+\log{\frac{2\log_2(Rn^2)}{\delta}}}{n}} + 
        \frac{d+\log{\frac{2\log_2(Rn^2)}{\delta}}}{n}
        \right).
    \end{aligned}
\end{equation}

Finally, we choose $\rvv = \max \left\{ \|\rvx - \rvx^*\|, \frac{1}{n} \right\}  \frac{(\mathbb{P} - \mathbb{P}_n) (\nabla_\rvx f(\rvx, \rvy^*(\rvx);\rvz) - \nabla_\rvx f(\rvx^*, \rvy^*;\rvz))}{\| (\mathbb{P} - \mathbb{P}_n) (\nabla_\rvx f(\rvx, \rvy^*(\rvx);\rvz) - \nabla_\rvx f(\rvx^*, \rvy^*;\rvz)) \|} $.
It is clear that $\|\rvv\| = \max \{ \|\rvx-\rvx^*\| , \frac1n \} \leq \max \{R_1, \frac1n\}$, which belongs to the space $\gV$. Plugging this $\rvv$ into 
(\ref{eq:g-before-the-final-result}), we have the following inequality that for any $\delta \in (0,1)$, with probability at least $1-\delta$, for all $\rvx \in \gX$,
\begin{equation}
    \begin{aligned}
    & \|(\mathbb{P} - \mathbb{P}_n) (\nabla_\rvx f(\rvx, \rvy^*(\rvx);\rvz) - \nabla_\rvx f(\rvx^*, \rvy^*;\rvz))\| \\
    \leq & 
    \frac{C\beta(\mu_\rvy+\beta)}{\mu_\rvy \ln{2}} 
    \max \left\{ \| \rvx - \rvx^* \|, \frac{1}{n} \right\} \times
    \left(
    \sqrt{\frac{d+\log{\frac{2\log_2(Rn^2)}{\delta}}}{n}} + 
    \frac{d+\log{\frac{2\log_2(Rn^2)}{\delta}}}{n}
    \right) \\
    = &
    \frac{C\beta(\mu_\rvy+\beta)}{\mu_\rvy \ln{2}} 
    \max \left\{ \| \rvx - \rvx^* \|, \frac{1}{n} \right\} \times
    \left(
    \sqrt{\frac{d+\log{\frac{4\log_2(\sqrt{2}R_1 n+1)}{\delta}}}{n}} + 
    \frac{d+\log{\frac{4\log_2(\sqrt{2}R_1 n+1)}{\delta}}}{n}
    \right).
    \end{aligned}
\end{equation}

Finally we have
\begin{align*}
    & \| (\nabla_\rvx F(\rvx, \rvy^*(\rvx)) - \nabla_\rvx F_S(\rvx, \rvy^*(\rvx)))
    - (\nabla_\rvx F(\rvx^*, \rvy^*) - \nabla_\rvx F_S(\rvx^*, \rvy^*))
    \| \\ 
    = & 
    \|(\mathbb{P} - \mathbb{P}_n) (\nabla_\rvx f(\rvx, \rvy^*(\rvx);\rvz) - \nabla_\rvx f(\rvx^*, \rvy^*(\rvx^*);\rvz))\| \\
    \leq & 
    \frac{C\beta(\mu_\rvy+\beta)}{\mu_\rvy} 
    \max \left\{ \| \rvx - \rvx^* \|, \frac{1}{n} \right\} \times
    \left(
    \sqrt{\frac{d+\log{\frac{4\log_2(\sqrt{2}R_1 n+1)}{\delta}}}{n}} + 
    \frac{d+\log{\frac{4\log_2(\sqrt{2}R_1 n+1)}{\delta}}}{n}
    \right),
\end{align*}
where $C$ is an absolute constant.

The proof is complete.
\end{proof}

\begin{lemma}
\label{lemma:nablafx-without-nablafx*}
    Suppose Assumption \ref{assumption:NC-SC} holds, we have the following inequality that for all $\rvx \in \gX$ and for any $\delta \in (0,1)$, with probability at lest $1-\delta$,
        \begin{align*}
        & \| \nabla_\rvx F(\rvx, \rvy^*(\rvx)) - \nabla_\rvx F_S(\rvx, \rvy^*(\rvx)) \| 
        \leq  \sqrt{\frac{2 \E \|\nabla_\rvx f(\rvx^*, \rvy^*; \rvz)\|^2 \log{\frac{4}{\delta}}}{n}} 
        + \frac{B_{\rvx^*} \log{\frac4\delta}}{n} \\
        & + \space
        \frac{C\beta(\mu_\rvy+\beta)}{\mu_\rvy} 
        \max \left\{ \| \rvx - \rvx^* \|, \frac{1}{n} \right\} \times
        \left(
        \sqrt{\frac{d+\log{\frac{8\log_2(\sqrt{2}R_1 n+1)}{\delta}}}{n}} + 
        \frac{d+\log{\frac{8\log_2(\sqrt{2}R_1 n+1)}{\delta}}}{n}
        \right).
    \end{align*}
\end{lemma}

\begin{proof}[Proof of Lemma \ref{lemma:nablafx-without-nablafx*}]

According to Lemma \ref{lemma:nablafx-nablafx*}, we have the following inequality that for any $\delta \in (0,1)$, with probability at least $1-\frac{\delta}{2}$
\begin{equation}
    \begin{aligned}
\label{eq:lemma-nablafx-nablafx*-1-1}
    & \| \nabla_\rvx F(\rvx, \rvy^*(\rvx)) - \nabla_\rvx F_S(\rvx, \rvy^*(\rvx)) \| 
    \leq   \|  \nabla_\rvx F(\rvx^*, \rvy^*) - \nabla_\rvx F_S(\rvx^*, \rvy^*)\|  \\
    & + \space
    \frac{C\beta(\mu_\rvy+\beta)}{\mu_\rvy} 
    \max \left\{ \| \rvx - \rvx^* \|, \frac{1}{n} \right\} \times
    \left(
    \sqrt{\frac{d+\log{\frac{8\log_2(\sqrt{2}R_1 n+1)}{\delta}}}{n}} + 
    \frac{d+\log{\frac{8\log_2(\sqrt{2}R_1 n+1)}{\delta}}}{n}
    \right),
    \end{aligned}
\end{equation}
where this inequality applies norm triangle inequality. Then we need to bound $ \|  \nabla_\rvx F(\rvx^*, \rvy^*) - \nabla_\rvx F_S(\rvx^*, \rvy^*)\|$.

According to Lemma \ref{lemma:vector-bernstein-ineq}, with probability at least $1-\frac{\delta}{2}$
\begin{align}
\label{eq:lemma-nablafx-nablafx*-1-2}
    \| \nabla_\rvx F(\rvx^*, \rvy^*) - \nabla_\rvx F_S(\rvx^*, \rvy^*)\| 
    \leq \sqrt{\frac{2 \E [ \|\nabla_\rvx f(\rvx^*, \rvy^*; \rvz)\|^2 ] \log{\frac{4}{\delta}}}{n}} 
    + \frac{B_{\rvx^*} \log{\frac4\delta}}{n}.
\end{align}

Combining (\ref{eq:lemma-nablafx-nablafx*-1-1}) and (\ref{eq:lemma-nablafx-nablafx*-1-2}), for any $\delta \in (0,1)$, with probability at least $1-\delta$, we have
    \begin{align*}
        & \| \nabla_\rvx F(\rvx, \rvy^*(\rvx)) - \nabla_\rvx F_S(\rvx, \rvy^*(\rvx)) \| 
        \leq  \sqrt{\frac{2 \E [ \|\nabla_\rvx f(\rvx^*, \rvy^*; \rvz)\|^2 ] \log{\frac{4}{\delta}}}{n}} 
        + \frac{B_{\rvx^*} \log{\frac4\delta}}{n} \\
        & + \space
        \frac{C\beta(\mu_\rvy+\beta)}{\mu_\rvy} 
        \max \left\{ \| \rvx - \rvx^* \|, \frac{1}{n} \right\} \times
        \left(
        \sqrt{\frac{d+\log{\frac{8\log_2(\sqrt{2}R_1 n+1)}{\delta}}}{n}} + 
        \frac{d+\log{\frac{8\log_2(\sqrt{2}R_1 n+1)}{\delta}}}{n}
        \right).
    \end{align*}
The proof is complete.
\end{proof}

Next, we introduce Lemma \ref{lemma:nablafy-nablafy*} and \ref{lemma:nablafy-without-nablafy*}. These proofs are similar with Lemma \ref{lemma:nablafx-nablafx*} and \ref{lemma:nablafx-without-nablafx*}.

\begin{lemma}
\label{lemma:nablafy-nablafy*}
    Suppose Assumption \ref{assumption:NC-SC} holds, we have the following inequality that for all $\rvx \in \gX$ and for any $\delta \in (0,1)$, with probability at lest $1-\delta$,
    \begin{align*}
        & \| (\nabla_\rvy F(\rvx, \rvy^*(\rvx)) - \nabla_\rvy F_S(\rvx, \rvy^*(\rvx)))
        - (\nabla_\rvy F(\rvx^*, \rvy^*) - \nabla_\rvy F_S(\rvx^*, \rvy^*))
        \| \\ 
        \leq & 
        \frac{C\beta(\mu_\rvy+\beta)}{\mu_\rvy} 
        \max \left\{ \| \rvx - \rvx^* \|, \frac{1}{n} \right\} \times
        \left(
        \sqrt{\frac{d+\log{\frac{4\log_2(\sqrt{2}R_1 n+1)}{\delta}}}{n}} + 
        \frac{d+\log{\frac{4\log_2(\sqrt{2}R_1 n+1)}{\delta}}}{n}
        \right),
    \end{align*}
    where $C$ is an absolute constant.
\end{lemma}

\begin{proof}[Proof of Lemma \ref{lemma:nablafy-nablafy*}]
    We define $\gV = \{ \rvv \in \R^d: \| \rvv \| \leq \max \{ R_1, \frac1n \} \}$. For all $(\rvx, \rvv) \in \gX \times \gV$, let $g_{(\rvx, \rvv)} (\rvz) = \left( \nabla_\rvy f(\rvx, \rvy^*(\rvx);\rvz) - \nabla_\rvy f(\rvx^*, \rvy^*; \rvz) \right)^\mathrm{T}\rvv$. Then for any $(\rvx_1, \rvv_1)$ and $(\rvx_2, \rvv_2) \in \gX \times \gV$, we define the following norm on the product space $\gX \times \gV$
    \begin{align*}
        \left\| (\rvx_1, \rvv_1) - (\rvx_2, \rvv_2) \right\|_{\gX \times \gV} 
        = \left( 
        \| \rvx_1 -\rvx_2 \|^2 + \| \rvv_1 - \rvv_2 \|^2
        \right)^{\frac12}.
    \end{align*}
    Then we define a ball on the product space $\gX \times \gV$ that $B(\sqrt{r}) = \{ (\rvx, \rvv) \in \gX \times \gV : \| \rvx - \rvx^* \|^2 + \| \rvv \|^2  \leq r \}$. Given any $(\rvx_1, \rvv_1)$ and $(\rvx_2, \rvv_2) \in B(\sqrt{r})$, we have the following decomposition
    \begin{align*}
        & g_{(\rvx_1, \rvv_1)} (\rvz) - g_{(\rvx_2, \rvv_2)} (\rvz) \\
        =& \left( \nabla_\rvy f(\rvx_1, \rvy^*(\rvx_1);\rvz) - \nabla_\rvy f(\rvx^*, \rvy^*; \rvz) \right)^\mathrm{T}\rvv_1  \\
        &\space - \left( \nabla_\rvy f(\rvx_2, \rvy^*(\rvx_2);\rvz) - \nabla_\rvy f(\rvx^*, \rvy^*; \rvz) \right)^\mathrm{T}\rvv_2 \\
        =& \left( \nabla_\rvy f(\rvx_1, \rvy^*(\rvx_1);\rvz) - \nabla_\rvy f(\rvx^*, \rvy^*; \rvz) \right)^\mathrm{T}(\rvv_1 -\ \rvv_2)  \\
        & \space + \left( \nabla_\rvy f(\rvx_1, \rvy^*(\rvx_1);\rvz) - \nabla_\rvy f(\rvx^*, \rvy^*; \rvz) \right)^\mathrm{T} \rvv_2 \\
        &\space - \left( \nabla_\rvy f(\rvx_2, \rvy^*(\rvx_2);\rvz) - \nabla_\rvy f(\rvx^*, \rvy^*; \rvz) \right)^\mathrm{T}\rvv_2 \\
        =& \left( \nabla_\rvy f(\rvx_1, \rvy^*(\rvx_1);\rvz) - \nabla_\rvy f(\rvx^*, \rvy^*; \rvz) \right)^\mathrm{T}(\rvv_1 -\ \rvv_2)  \\
        & \space + \left( \nabla_\rvy f(\rvx_1, \rvy^*(\rvx_1);\rvz) - \nabla_\rvy f(\rvx_2, \rvy^*(\rvx_2); \rvz) \right)^\mathrm{T} \rvv_2.
    \end{align*}
    Since $(\rvx_1, \rvv_1)$ and $(\rvx_2, \rvv_2) \in B(\sqrt{r})$, we have
    \begin{align}
    \label{eq:lemma:nablafy-nablafy*-1-1}
        \| \rvx_1 - \rvx^* \| \| \rvv_1 -\rvv_2 \| \leq \sqrt{r} \| \rvv_1 - \rvv_2 \| \leq \sqrt{r} \| (\rvx_1, \rvv_1) - (\rvx_2, \rvv_2) \|_{\gX \times \gV}.
    \end{align}
    Next, from Assumption \ref{assumption:NC-SC} and Lemma \ref{lemma:sub-exp}, we know that $ \frac{\nabla_\rvy f(\rvx_1, \rvy^*(\rvx_1); \rvz) - \nabla_\rvy f(\rvx_2, \rvy^*(\rvx_2); \rvz)}{\| \rvx_1 - \rvx_2 \|} $ is a $\frac{\beta(\mu_\rvy+\beta)}{\mu_\rvy\ln{2}}$-sub-exponential random vector for all $\rvx_1, \rvx_2 \in \gX$, which means that
    \begin{align}
    \label{eq:lemma:nablafy-nablafy*-1-2}
        \E \left\{
        \exp \left(
        \frac{(\nabla_\rvy f(\rvx_1, \rvy^*(\rvx_1); \rvz) - \nabla_\rvy f(\rvx^*, \rvy^*(\rvx^*); \rvz))^\mathrm{T}(\rvv_1 -\rvv_2)}{ \frac{\beta(\mu_\rvy+\beta)}{\mu_\rvy\ln{2}} \| \rvx_1 - \rvx^* \| \| \rvv_1 - \rvv_2 \| }
        \right)
        \right\} \leq 2.
    \end{align}
    Combining (\ref{eq:lemma:nablafy-nablafy*-1-1}) and (\ref{eq:lemma:nablafy-nablafy*-1-2}), according to Definition \ref{def:orlicz-alpha-norm}, we can easily deduce that $\left( \nabla_\rvy f(\rvx_1, \rvy^*(\rvx_1);\rvz) - \nabla_\rvy f(\rvx^*, \rvy^*; \rvz) \right)^\mathrm{T}(\rvv_1 -\ \rvv_2)$ is $\frac{\beta\sqrt{r}(\mu_\rvy+\beta)}{\mu_\rvy\ln{2}}\| (\rvx_1, \rvv_1) - (\rvx_2, \rvv_2) \|_{\gX \times \gV}$-sub-exponential. Similarly, we can derive that
    \begin{align*}
        \| \rvx_1 - \rvx_2 \| \| \rvv_2 \| \leq \sqrt{r} \| \rvx_1 - \rvx_2 \| \leq \sqrt{r} \| (\rvx_1, \rvv_1) - (\rvx_2, \rvv_2) \|_{\gX \times \gV}.
    \end{align*}
    Also, there holds that
    \begin{align*}
        \E \left\{
        \exp \left(
        \frac{(\nabla_\rvy f(\rvx_1, \rvy^*(\rvx_1); \rvz) - \nabla_\rvy f(\rvx_2, \rvy^*(\rvx_2); \rvz))^\mathrm{T}(\rvv_2)}{ \frac{\beta(\mu_\rvy+\beta)}{\mu_\rvy\ln{2}}  \| \rvx_1 - \rvx_2 \| \|\rvv_2\|}
        \right)
        \right\} \leq 2.
    \end{align*}
    Thus, we have that $ (\nabla_\rvy f(\rvx_1, \rvy^*(\rvx_1); \rvz) - \nabla_\rvy f(\rvx_2, \rvy^*(\rvx_2); \rvz))^\mathrm{T}(\rvv_2)$ is also $\frac{\beta \sqrt{r} (\mu_\rvy+\beta)}{\mu_\rvy\ln{2}} \| (\rvx_1, \rvv_1) - (\rvx_2, \rvv_2) \|_{\gX \times \gV}$-sub-exponential. 
    Now for any $(\rvx_1, \rvv_1)$ and $(\rvx_2, \rvv_2) \in B(\sqrt{r})$, we know
    \begin{equation}
    \label{eq:lemma:nablafy-nablafy*-g-is-sub-exp}
        \begin{aligned}
            & \E \left\{
            \exp \left(
            \frac{g_{(\rvx_1,\rvv_1)}(\rvz) - g_{(\rvx_2,\rvv_2)}(\rvz)}{ \frac{2 \beta \sqrt{r} (\mu_\rvy+\beta)}{\mu_\rvy\ln{2}} \| (\rvx_1, \rvv_1) - (\rvx_2, \rvv_2) \|_{\gX \times \gV}}
            \right)
            \right\} \\
            \leq & \E \left\{
            \frac{1}{2}    \exp \left(
            \frac{(\nabla_\rvy f(\rvx_1, \rvy^*(\rvx_1); \rvz) - \nabla_\rvy f(\rvx^*, \rvy^*; \rvz))^\mathrm{T}(\rvv_1 - \rvv_2)}{ \frac{\beta \sqrt{r} (\mu_\rvy+\beta)}{\mu_\rvy\ln{2}} \| (\rvx_1, \rvv_1) - (\rvx_2, \rvv_2) \|_{\gX \times \gV}}
            \right)
            \right\} \\
            + & \E \left\{
            \frac{1}{2}    \exp \left(
            \frac{(\nabla_\rvy f(\rvx_1, \rvy^*(\rvx_1); \rvz) - \nabla_\rvy f(\rvx_2, \rvy^*(\rvx_2); \rvz))^\mathrm{T}(\rvv_2)}{ \frac{\beta \sqrt{r} (\mu_\rvy+\beta)}{\mu_\rvy\ln{2}} \| (\rvx_1, \rvv_1) - (\rvx_2, \rvv_2) \|_{\gX \times \gV}}
            \right)
            \right\}
            \leq 2,
        \end{aligned}
    \end{equation}
where the first inequality follows from Jensen's inequality. (\ref{eq:lemma:nablafy-nablafy*-g-is-sub-exp}) implies that $g_{(\rvx_1,\rvv_1)}(\rvz) - g_{(\rvx_2,\rvv_2)}(\rvz)$ is a $\frac{2 \beta \sqrt{r} (\mu_\rvy+\beta)}{\mu_\rvy\ln{2}} \| (\rvx_1, \rvv_1) - (\rvx_2, \rvv_2) \|_{\gX \times \gV}$-sub-exponential random variable, for which we have
\begin{equation}
    \begin{aligned}
        \| g_{(\rvx_1,\rvv_1)}(\rvz) - g_{(\rvx_2,\rvv_2)}(\rvz) \|_{\text{Orlicz}_{1}}
        \leq \frac{2 \beta \sqrt{r} (\mu_\rvy+\beta)}{\mu_\rvy\ln{2}} \| (\rvx_1, \rvv_1) - (\rvx_2, \rvv_2) \|_{\gX \times \gV}.
    \end{aligned}
\end{equation}

From Bernstein inequality for sub-exponential variables (Lemma \ref{lemma:bernstein-ineq-for-sub-exp}), for any fixed $u \leq 0$ and $(\rvx_1. \rvv_1), (\rvx_2,\rvv_2) \in \gX \times \gV$, we have
\begin{equation}
    \begin{aligned}
    \label{eq:lemma:nablafy-nablafy*-mix-sub-gaussian-sub-exp}
    Pr\left(
    |(\mathbb{P} - \mathbb{P}_n)[ g_{(\rvx_1,\rvv_1)}(\rvz) - g_{(\rvx_2,\rvv_2)}(\rvz) ] |
    \geq \frac{2 \beta \sqrt{r} (\mu_\rvy+\beta)}{\mu_\rvy\ln{2}} \| (\rvx_1, \rvv_1) - (\rvx_2, \rvv_2) \|_{\gX \times \gV} \sqrt{\frac{2u}{n}}  \right.\\
    \left. + \frac{2 u \beta \sqrt{r} (\mu_\rvy+\beta)}{n\mu_\rvy\ln{2}} \| (\rvx_1, \rvv_1) - (\rvx_2, \rvv_2) \|_{\gX \times \gV}
    \right) 
    \leq 2e^{-u}.        
    \end{aligned}
\end{equation}

According to Definition \ref{def:mix-sub-gaussian-sub-exp}, we know that $(\mathbb{P} - \mathbb{P}_n) g_{(\rvx,\rvv)}(\rvz)$ is a mixed sub-Gaussian-sub-exponential increments w.r.t. the metrics 
$(\frac{2 \beta \sqrt{r} (\mu_\rvy+\beta)}{n\mu_\rvy\ln{2}} \| \cdot \|_{\gX \times \gV},
\frac{2 \beta \sqrt{2r} (\mu_\rvy+\beta)}{\sqrt{n}\mu_\rvy\ln{2}} \| \cdot \|_{\gX \times \gV})$
from (\ref{eq:lemma:nablafy-nablafy*-mix-sub-gaussian-sub-exp}). 

Then from the generic chaining for a process with mixed tail increments in Lemma \ref{lemma:generic-chaining-for-mixed-tail-increments}, we have the following inequality that for all $\delta \in (0,1)$, with probability at least $1-\delta$
\begin{align*}
    & \sup_{\|\rvx_1 - \rvx^* \|^2 + \| \rvv \|^2 \leq r}
    | (\mathbb{P} - \mathbb{P}_n) g_{(\rvx,\rvv)}(\rvz) | \\
    \leq & C \Bigg(
    \gamma_2 
    \left(
        B(\sqrt{r}), \frac{2 \beta \sqrt{2r} (\mu_\rvy+\beta)}{\sqrt{n}\mu_\rvy\ln{2}} \| \cdot \|_{\gX \times \gV}
    \right)+
    \gamma_1 
    \left(
        B(\sqrt{r}), \frac{2 \beta \sqrt{r} (\mu_\rvy+\beta)}{n\mu_\rvy\ln{2}} \| \cdot \|_{\gX \times \gV}
    \right) \\
     &
    + \frac{r\beta(\mu_\rvy+\beta)}{\mu_\rvy \ln{2}} \sqrt{\frac{\log{\frac1\delta}}{n}}
    + \frac{r\beta(\mu_\rvy+\beta)}{\mu_\rvy \ln{2}} \frac{\log{\frac1\delta}}{n}
    \Bigg).
\end{align*}

Next we use the Dudley's integral (Lemma \ref{lemma:dudley-integral-bound}) to bound the $\gamma_1$ and $\gamma_2$ functional. So there exists an absolute constant $C$ such that for any $\delta \in (0,1)$, with probability at least $1-\delta$
\begin{align}
\label{eq:lemma:nablafy-nablafy*-sup-before-a-varient-uniform-localized-convergenc-argument}
    \sup_{\|\rvx_1 - \rvx^* \|^2 + \| \rvv \|^2 \leq r}
    | (\mathbb{P} - \mathbb{P}_n) g_{(\rvx,\rvv)}(\rvz) |
    \leq \frac{rC\beta(\mu_\rvy+\beta)}{\mu_\rvy \ln{2}}
    \left(
    \sqrt{\frac{d+\log{\frac{1}{\delta}}}{n}} + 
    \frac{d+\log{\frac{1}{\delta}}}{n}
    \right).
\end{align}

Then, the next step is to apply Lemma \ref{lemma:a-variant-uniform-localized-convergence-argument} to (\ref{eq:lemma:nablafy-nablafy*-sup-before-a-varient-uniform-localized-convergenc-argument}). We denote $T(f) = \|\rvx - \rvx^* \|^2 + \| \rvv \|^2 $,
$\psi(r;\delta) = \frac{rC\beta(\mu_\rvy+\beta)}{\mu_\rvy \ln{2}} \left(
    \sqrt{\frac{d+\log{\frac{1}{\delta}}}{n}} + 
    \frac{d+\log{\frac{1}{\delta}}}{n}
    \right)$. 
Since $\|\rvx - \rvx^* \|^2 + \| \rvv \|^2 \leq R_1^2 + R_1^2 + \frac{1}{n^2}$, we set $R^2 = 2R_1^2 + \frac{1}{n^2}$ and $r_0 = \frac{2}{n^2}$. According to Lemma \ref{lemma:a-variant-uniform-localized-convergence-argument}, we have the following inequality that for any $\delta \in (0,1)$, with probability at least $1-\delta$, for all $\rvx \in \gX$ and $\rvv \in \gV$,
\begin{equation}
    \begin{aligned}
\label{eq:lemma:nablafy-nablafy*-g-before-the-final-result}
        & (\mathbb{P} - \mathbb{P}_n) g_{(\rvx,\rvv)}(\rvz) \\
        = & (\mathbb{P} - \mathbb{P}_n) [(\nabla_\rvy f(\rvx, \rvy^*(\rvx);\rvz) - \nabla_\rvy f(\rvx^*, \rvy^*;\rvz))^\mathrm{T} \rvv ]\\
        \leq & \psi\left(
        \max \left\{ \| \rvx - \rvx^* \|^2 + \| \rvv \|^2, \frac{2}{n^2} \right\};
        \frac{\delta}{2\log_2(Rn^2)}
        \right) \\
        \leq & 
        \frac{C\beta(\mu_\rvy+\beta)}{\mu_\rvy \ln{2}} 
        \max \left\{ \| \rvx - \rvx^* \|^2 + \| \rvv \|^2, \frac{2}{n^2} \right\} \times
        \left(
        \sqrt{\frac{d+\log{\frac{2\log_2(Rn^2)}{\delta}}}{n}} + 
        \frac{d+\log{\frac{2\log_2(Rn^2)}{\delta}}}{n}
        \right).
    \end{aligned}
\end{equation}

Finally, we choose $\rvv = \max \left\{ \|\rvx - \rvx^*\|, \frac{1}{n} \right\}  \frac{(\mathbb{P} - \mathbb{P}_n) (\nabla_\rvy f(\rvx, \rvy^*(\rvx);\rvz) - \nabla_\rvy f(\rvx^*, \rvy^*;\rvz))}{\| (\mathbb{P} - \mathbb{P}_n) (\nabla_\rvy f(\rvx, \rvy^*(\rvx);\rvz) - \nabla_\rvy f(\rvx^*, \rvy^*;\rvz)) \|} $.
It is clear that $\|\rvv\| = \max \{ \|\rvx-\rvx^*\| , \frac1n \} \leq \max \{R_1, \frac1n\}$, which belongs to the space $\gV$. Plugging this $\rvv$ into 
\begin{equation}
    \begin{aligned}
    & \|(\mathbb{P} - \mathbb{P}_n) (\nabla_\rvy f(\rvx, \rvy^*(\rvx);\rvz) - \nabla_\rvy f(\rvx^*, \rvy^*;\rvz))\| \\
    \leq & 
    \frac{C\beta(\mu_\rvy+\beta)}{\mu_\rvy \ln{2}} 
    \max \left\{ \| \rvx - \rvx^* \|, \frac{1}{n} \right\} \times
    \left(
    \sqrt{\frac{d+\log{\frac{2\log_2(Rn^2)}{\delta}}}{n}} + 
    \frac{d+\log{\frac{2\log_2(Rn^2)}{\delta}}}{n}
    \right) \\
    = &
    \frac{C\beta(\mu_\rvy+\beta)}{\mu_\rvy \ln{2}} 
    \max \left\{ \| \rvx - \rvx^* \|, \frac{1}{n} \right\} \times
    \left(
    \sqrt{\frac{d+\log{\frac{4\log_2(\sqrt{2}R_1 n+1)}{\delta}}}{n}} + 
    \frac{d+\log{\frac{4\log_2(\sqrt{2}R_1 n+1)}{\delta}}}{n}
    \right).
    \end{aligned}
\end{equation}

Finally we have
\begin{align*}
    & \| (\nabla_\rvy F(\rvx, \rvy^*(\rvx)) - \nabla_\rvy F_S(\rvx, \rvy^*(\rvx)))
    - (\nabla_\rvy F(\rvx^*, \rvy^*) - \nabla_\rvy F_S(\rvx^*, \rvy^*))
    \| \\ 
    = & 
    \|(\mathbb{P} - \mathbb{P}_n) (\nabla_\rvy f(\rvx, \rvy^*(\rvx);\rvz) - \nabla_\rvy f(\rvx^*, \rvy^*(\rvx^*);\rvz))\| \\
    \leq & 
    \frac{C\beta(\mu_\rvy+\beta)}{\mu_\rvy} 
    \max \left\{ \| \rvx - \rvx^* \|, \frac{1}{n} \right\} \times
    \left(
    \sqrt{\frac{d+\log{\frac{4\log_2(\sqrt{2}R_1 n+1)}{\delta}}}{n}} + 
    \frac{d+\log{\frac{4\log_2(\sqrt{2}R_1 n+1)}{\delta}}}{n}
    \right),
\end{align*}
where $C$ is an absolute constant.

The proof is complete.
\end{proof}

\begin{lemma}
\label{lemma:nablafy-without-nablafy*}
    Suppose Assumption \ref{assumption:NC-SC} holds, we have the following inequality that for all $\rvx \in \gX$ and for any $\delta \in (0,1)$, with probability at lest $1-\delta$,
        \begin{align*}
        & \| \nabla_\rvy F(\rvx, \rvy^*(\rvx)) - \nabla_\rvy F_S(\rvx, \rvy^*(\rvx)) \| 
        \leq  \sqrt{\frac{2 \E \|\nabla_\rvy f(\rvx^*, \rvy^*; \rvz)\|^2 \log{\frac{4}{\delta}}}{n}} 
        + \frac{B_{\rvy^*} \log{\frac4\delta}}{n} \\
        & + \space
        \frac{C\beta(\mu_\rvy+\beta)}{\mu_\rvy} 
        \max \left\{ \| \rvx - \rvx^* \|, \frac{1}{n} \right\} \times
        \left(
        \sqrt{\frac{d+\log{\frac{8\log_2(\sqrt{2}R_1 n+1)}{\delta}}}{n}} + 
        \frac{d+\log{\frac{8\log_2(\sqrt{2}R_1 n+1)}{\delta}}}{n}
        \right).
    \end{align*}
\end{lemma}

\begin{proof}[Proof of Lemma \ref{lemma:nablafy-without-nablafy*}]

According to Lemma \ref{lemma:nablafy-nablafy*}, we have the following inequality that for any $\delta \in (0,1)$, with probability at least $1-\frac{\delta}{2}$
\begin{equation}
    \begin{aligned}
\label{eq:lemma-nablafy-nablafy*-1-1}
    & \| \nabla_\rvy F(\rvx, \rvy^*(\rvx)) - \nabla_\rvy F_S(\rvx, \rvy^*(\rvx)) \| 
    \leq   \|  \nabla_\rvy F(\rvx^*, \rvy^*) - \nabla_\rvy F_S(\rvx^*, \rvy^*)\|  \\
    & + \space
    \frac{C\beta(\mu_\rvy+\beta)}{\mu_\rvy} 
    \max \left\{ \| \rvx - \rvx^* \|, \frac{1}{n} \right\} \times
    \left(
    \sqrt{\frac{d+\log{\frac{8\log_2(\sqrt{2}R_1 n+1)}{\delta}}}{n}} + 
    \frac{d+\log{\frac{8\log_2(\sqrt{2}R_1 n+1)}{\delta}}}{n}
    \right),
    \end{aligned}
\end{equation}
where this inequality applies norm triangle inequality. Then we need to bound $ \|  \nabla_\rvy F(\rvx^*, \rvy^*) - \nabla_\rvy F_S(\rvx^*, \rvy^*)\|$.

According to Lemma \ref{lemma:vector-bernstein-ineq}, with probability at least $1-\frac{\delta}{2}$
\begin{align}
\label{eq:lemma-nablafy-nablafy*-1-2}
    \| \nabla_\rvy F(\rvx^*, \rvy^*) - \nabla_\rvy F_S(\rvx^*, \rvy^*)\| 
    \leq \sqrt{\frac{2 \E [ \|\nabla_\rvy f(\rvx^*, \rvy^*; \rvz)\|^2 ] \log{\frac{4}{\delta}}}{n}} 
    + \frac{B_{\rvy^*} \log{\frac4\delta}}{n}.
\end{align}

Combining (\ref{eq:lemma-nablafy-nablafy*-1-1}) and (\ref{eq:lemma-nablafy-nablafy*-1-2}), for any $\delta \in (0,1)$, with probability at least $1-\delta$, we have
    \begin{align*}
        & \| \nabla_\rvy F(\rvx, \rvy^*(\rvx)) - \nabla_\rvy F_S(\rvx, \rvy^*(\rvx)) \| 
        \leq  \sqrt{\frac{2 \E [ \|\nabla_\rvy f(\rvx^*, \rvy^*; \rvz)\|^2 ] \log{\frac{4}{\delta}}}{n}} 
        + \frac{B_{\rvy^*} \log{\frac4\delta}}{n} \\
        & + \space
        \frac{C\beta(\mu_\rvy+\beta)}{\mu_\rvy} 
        \max \left\{ \| \rvx - \rvx^* \|, \frac{1}{n} \right\} \times
        \left(
        \sqrt{\frac{d+\log{\frac{8\log_2(\sqrt{2}R_1 n+1)}{\delta}}}{n}} + 
        \frac{d+\log{\frac{8\log_2(\sqrt{2}R_1 n+1)}{\delta}}}{n}
        \right).
    \end{align*}
The proof is complete.
\end{proof}

\section{Proofs in Section \ref{sec:main-section}}

\begin{proof}[Proof of Theorem \ref{theorem:main}]
    
Firstly, for all $\rvx \in \gX$, we divide $\| \nabla \Phi(\rvx) - \nabla \Phi_S(\rvx) \|$ into two terms
\begin{equation}
\begin{aligned}
\label{eq:theorem-main-1-1}
    &  \| \nabla \Phi(\rvx) - \nabla \Phi_S(\rvx) \| \\
    = & 
    \left\| 
    \E_\rvz \nabla_\rvx f(\rvx, \rvy^*(\rvx); \rvz) - \frac1n \sum_{i=1}^n \nabla_\rvx f(\rvx, \rvy_S^*(\rvx); \rvz_i) 
    \right\| \\
    = & 
    \left\| 
    \E_\rvz \nabla_\rvx f(\rvx, \rvy^*(\rvx); \rvz) - \frac1n \sum_{i=1}^n \nabla_\rvx f(\rvx, \rvy^*(\rvx); \rvz_i)  \right.\\
    & \quad \quad \left. + \frac1n \sum_{i=1}^n \nabla_\rvx f(\rvx, \rvy^*(\rvx); \rvz_i)  - \frac1n \sum_{i=1}^n \nabla_\rvx f(\rvx, \rvy_S^*(\rvx); \rvz_i) 
    \right\| \\
    \leq & 
    \left\| 
    \E_\rvz \nabla_\rvx f(\rvx, \rvy^*(\rvx); \rvz) - \frac1n \sum_{i=1}^n \nabla_\rvx f(\rvx, \rvy^*(\rvx); \rvz_i)  \right\| \\
    & \quad \quad + \left\| \frac1n \sum_{i=1}^n \nabla_\rvx f(\rvx, \rvy^*(\rvx); \rvz_i)  - \frac1n \sum_{i=1}^n \nabla_\rvx f(\rvx, \rvy_S^*(\rvx); \rvz_i) 
    \right\| \\
    \leq & 
    \left\| 
    \E_\rvz \nabla_\rvx f(\rvx, \rvy^*(\rvx); \rvz) - \frac1n \sum_{i=1}^n \nabla_\rvx f(\rvx, \rvy^*(\rvx); \rvz_i)  \right\| + \beta \left\| \rvy^*(\rvx) - \rvy_S^*(\rvx)
    \right\| \\
    = &
    \| \nabla_\rvx F(\rvx, \rvy^*(\rvx)) - \nabla_\rvx F_S(\rvx, \rvy^*(\rvx)) \|  + \beta \left\| \rvy^*(\rvx) - \rvy_S^*(\rvx) \right\|,
\end{aligned}
\end{equation}
where the first inequality satisfies from the triangle inequality, the second inequality holds by the smoothness of $f$. Next we need to upper bound these two terms respectively.

Firstly we need to upper bound $\| \nabla_\rvx F(\rvx, \rvy^*(\rvx)) - \nabla_\rvx F_S(\rvx, \rvy^*(\rvx)) \|$.
According to Lemma \ref{lemma:nablafx-without-nablafx*}, for all $\rvx \in \gX$ and for any $\delta \in (0,1)$, with probability at least $1-\frac{\delta}{2}$, we have
\begin{equation}
    \begin{aligned}
    \label{eq:theorem-main-1-2}
        & \| \nabla_\rvx F(\rvx, \rvy^*(\rvx)) - \nabla_\rvx F_S(\rvx, \rvy^*(\rvx)) \| 
        \leq  \sqrt{\frac{2 \E [ \|\nabla_\rvx f(\rvx^*, \rvy^*; \rvz)\|^2 ] \log{\frac{8}{\delta}}}{n}} 
        + \frac{B_{\rvx^*} \log{\frac8\delta}}{n} \\
        & + \space
        \frac{C\beta(\mu_\rvy+\beta)}{\mu_\rvy} 
        \max \left\{ \| \rvx - \rvx^* \|, \frac{1}{n} \right\} \times
        \left(
        \sqrt{\frac{d+\log{\frac{16\log_2(\sqrt{2}R_1 n+1)}{\delta}}}{n}} + 
        \frac{d+\log{\frac{16\log_2(\sqrt{2}R_1 n+1)}{\delta}}}{n}
        \right).
    \end{aligned}
\end{equation}

Next we will bound $\left\| \rvy^*(\rvx) - \rvy_S^*(\rvx)\right\|$. According to Lemma \ref{lemma:concentration-of-y} and Lemma \ref{lemma:nablafy-without-nablafy*}, with probability at least $1-\frac{\delta}{2}$ we have the following inequalities
\begin{equation}
\begin{aligned}
\label{eq:theorem-main-1-3}
    & \left\| \rvy^*(\rvx) - \rvy_S^*(\rvx)\right\|\\
    \leq & \frac{1}{\mu_\rvy} \left(
     \sqrt{\frac{2 \E \|\nabla_\rvy f(\rvx^*, \rvy^*; \rvz)\|^2 \log{\frac{8}{\delta}}}{n}} 
        + \frac{B_{\rvy^*} \log{\frac8\delta}}{n} \right. \\
        & + \space \left.
        \frac{C\beta(\mu_\rvy+\beta)}{\mu_\rvy} 
        \max \left\{ \| \rvx - \rvx^* \|, \frac{1}{n} \right\} \times
        \left(
        \sqrt{\frac{d+\log{\frac{16\log_2(\sqrt{2}R_1 n+1)}{\delta}}}{n}} + 
        \frac{d+\log{\frac{16\log_2(\sqrt{2}R_1 n+1)}{\delta}}}{n}
        \right)
    \right).
\end{aligned}
\end{equation}

Finally, we plug (\ref{eq:theorem-main-1-2}) and (\ref{eq:theorem-main-1-3}) into (\ref{eq:theorem-main-1-1}), we obtain the result that for any $\delta \in (0,1)$, with probability at least $1-\delta$,
\begin{align*}
    &  \| \nabla \Phi(\rvx) - \nabla \Phi_S(\rvx) \| 
    \leq \frac{\beta}{\mu_\rvy} \left(
    \sqrt{\frac{2 \E \|\nabla_\rvy f(\rvx^*, \rvy^*; \rvz)\|^2 \log{\frac{8}{\delta}}}{n}} 
        + \frac{B_{\rvy^*} \log{\frac8\delta}}{n}
    \right) \\
    & + \sqrt{\frac{2 \E [ \|\nabla_\rvx f(\rvx^*, \rvy^*; \rvz)\|^2 ] \log{\frac{8}{\delta}}}{n}} 
    + \frac{B_{\rvx^*} \log{\frac8\delta}}{n}
    + \frac{C\beta(\mu_\rvy+\beta)}{\mu_\rvy} \frac{(\mu_\rvy+\beta)}{\mu_\rvy}
    \max \left\{ \| \rvx - \rvx^* \|, \frac{1}{n} \right\} \times \\
    & \left(
    \sqrt{\frac{d+\log{\frac{16\log_2(\sqrt{2}R_1 n+1)}{\delta}}}{n}} + 
    \frac{d+\log{\frac{16\log_2(\sqrt{2}R_1 n+1)}{\delta}}}{n}
    \right).
\end{align*}
The proof is complete.
\end{proof}

\begin{proof}[Proof of Theorem \ref{theorem:nablaf-without-d}]

According to Theorem \ref{theorem:main}, for any $\delta \in (0,1)$, with probability at least $1-\delta$,
\begin{equation}
    \begin{aligned}
    \label{lemma:x-x*-under-condition-1-1}
    &  \| \nabla \Phi(\rvx) - \nabla \Phi_S(\rvx) \| 
    \leq \frac{\beta}{\mu_\rvy} \left(
    \sqrt{\frac{2 \E \|\nabla_\rvy f(\rvx^*, \rvy^*; \rvz)\|^2 \log{\frac{8}{\delta}}}{n}} 
        + \frac{B_{\rvy^*} \log{\frac8\delta}}{n}
    \right) \\
    & + \sqrt{\frac{2 \E [ \|\nabla_\rvx f(\rvx^*, \rvy^*; \rvz)\|^2 ] \log{\frac{8}{\delta}}}{n}} 
    + \frac{B_{\rvx^*} \log{\frac8\delta}}{n}
    + \frac{C\beta(\mu_\rvy+\beta)}{\mu_\rvy} \frac{(\mu_\rvy+\beta)}{\mu_\rvy}
    \max \left\{ \| \rvx - \rvx^* \|, \frac{1}{n} \right\} \times \\
    & \left(
    \sqrt{\frac{d+\log{\frac{16\log_2(\sqrt{2}R_1 n+1)}{\delta}}}{n}} + 
    \frac{d+\log{\frac{16\log_2(\sqrt{2}R_1 n+1)}{\delta}}}{n}
    \right).        
    \end{aligned}
\end{equation}

Since the population risk $F(\rvx,\rvy)$ satisfies Assumption \ref{assumption:PL-condition-x} with parameter $\mu_\rvx$, according to Lemma \ref{lemma:primal-function-is-PL}, $\Phi(\rvx)$ satisfies the PL condition with $\mu_\rvx$, there holds the error bound property (refer to Theorem 2 in \cite{karimi2016linear})
\begin{align}
\label{lemma:x-x*-under-condition-1-2}
    \mu_\rvx \| \rvx - \rvx^* \| \leq \| \nabla \Phi(\rvx) \|
\end{align}

Thus, combing (\ref{lemma:x-x*-under-condition-1-1}) and (\ref{lemma:x-x*-under-condition-1-2}) we have
\begin{equation}
    \begin{aligned}
\label{eq:theorem:nablaf-without-d-1-1}
        & \mu_\rvx \| \rvx - \rvx^* \|  
        \leq \| \nabla \Phi(\rvx) \| \\
        \leq &  
        \frac{C\beta(\mu_\rvy+\beta)}{\mu_\rvy}   \frac{(\mu_\rvy+\beta)}{\mu_\rvy}
        \max \left\{ \| \rvx - \rvx^* \|, \frac{1}{n} \right\} \times
        \left(
        \sqrt{\frac{d+\log{\frac{16\log_2(\sqrt{2}R_1 n+1)}{\delta}}}{n}} + 
        \frac{d+\log{\frac{16\log_2(\sqrt{2}R_1 n+1)}{\delta}}}{n}
        \right) \\
        & +
        \| \nabla \Phi_S(\rvx) \|
        + \sqrt{\frac{2 \E [ \|\nabla_\rvx f(\rvx^*, \rvy^*; \rvz)\|^2 ] \log{\frac{8}{\delta}}}{n}} 
        + \frac{B_{\rvx^*} \log{\frac8\delta}}{n} \\
        & \quad + \frac{\beta}{\mu_\rvy} \left(
    \sqrt{\frac{2 \E \|\nabla_\rvy f(\rvx^*, \rvy^*; \rvz)\|^2 \log{\frac{8}{\delta}}}{n}} 
        + \frac{B_{\rvy^*} \log{\frac8\delta}}{n}
    \right).
    \end{aligned}
\end{equation}

On the other hand, according to Lemma \ref{lemma:primal-function-is-smooth} with Assumption \ref{assumption:NC-SC}, the function $\Phi(\rvx)$ is $\beta + \frac{\beta^2}{\mu_\rvy}$-smooth in $\rvx \in \gX$. According to \cite{nesterov2003introductory}, $\Phi(\rvx)$ holds the following property
\begin{align*}
    \frac{1}{2(\beta + \frac{\beta^2}{\mu_\rvy})} \| \nabla  \Phi(\rvx) \|^2 \leq \Phi(\rvx) - \Phi(\rvx^*).
\end{align*}

We know that $\Phi(\rvx)$ satisfies the PL condition with $\mu_\rvx$. Thus, we have
\begin{align*}
    \frac{1}{2(\beta + \frac{\beta^2}{\mu_\rvy})}\| \nabla  \Phi(\rvx) \| \leq \Phi(\rvx) - \Phi(\rvx^*) \leq  \frac{1}{2\mu_\rvx} \| \nabla \Phi(\rvx) \|^2,
\end{align*}

which means that $\frac{\mu_\rvx\mu_\rvy}{\beta(\mu_\rvy+\beta)} \leq 1$. Then let $c = \max \{16C^2, 1\}$, when 
\begin{align*}
    n \geq \frac{c\beta^2(\mu_\rvy+\beta)^4(d+ \log{\frac{16\log_2{\sqrt{2}R_1 n + 1}}{\delta}})}{\mu_\rvy^4\mu_\rvx^2},
\end{align*}
we have 
\begin{align}
\label{eq:theorem:nablaf-without-d-1-2}
 \frac{C\beta(\mu_\rvy+\beta)}{\mu_\rvy} 
        \left(
        \sqrt{\frac{d+\log{\frac{16\log_2(\sqrt{2}R_1 n+1)}{\delta}}}{n}} + 
        \frac{d+\log{\frac{16\log_2(\sqrt{2}R_1 n+1)}{\delta}}}{n}
        \right)
        \leq \frac{\mu_\rvx}{2},   
\end{align}
with the fact that $\frac{\mu_\rvx\mu_\rvy}{\beta(\mu_\rvy+\beta)} \leq 1$.

Next, plugging (\ref{eq:theorem:nablaf-without-d-1-2}) into (\ref{eq:theorem:nablaf-without-d-1-1}), we obtain that
\begin{equation}
\begin{aligned}
\label{eq:lemma:nablafx-without-nablafx*-without-d-1-1}
    \| \rvx -\rvx^* \| 
    & \leq \frac{2}{\mu_\rvx}
    \left( \frac{\beta}{\mu_\rvy} \left(
    \sqrt{\frac{2 \E [ \|\nabla_\rvy f(\rvx^*, \rvy^*; \rvz)\|^2 ] \log{\frac{8}{\delta}}}{n}} + \frac{B_{\rvy^*} \log{\frac8\delta}}{n}  \right) \right.  \\
      & \quad \left.  + \sqrt{\frac{2 \E [ \|\nabla_\rvx f(\rvx^*, \rvy^*; \rvz)\|^2 ] \log{\frac{8}{\delta}}}{n}} 
        + \frac{B_{\rvx^*} \log{\frac8\delta}}{n} + \frac{\mu_\rvx}{2n}
        + 
    \| \nabla \Phi_S(\rvx)) \| 
    \right).
\end{aligned}    
\end{equation}

Combing (\ref{eq:lemma:nablafx-without-nablafx*-without-d-1-1}), (\ref{eq:theorem:nablaf-without-d-1-2}) and (\ref{lemma:x-x*-under-condition-1-1}), we have that for all $\rvx \in \gX$, let $c = \max \{ 16C^2, 1 \}$, when 
$  n \geq \frac{c\beta^2(\mu_\rvy+\beta)^4(d+ \log{\frac{16\log_2{\sqrt{2}R_1 n + 1}}{\delta}})}{\mu_\rvy^4\mu_\rvx^2}$, 
with probability at least $1-\delta$
\begin{align*}
    &  \| \nabla \Phi(\rvx) - \nabla \Phi_S(\rvx) \| 
    \leq \| \nabla \Phi_S(\rvx) \|
        + 2\sqrt{\frac{2 \E [ \|\nabla_\rvx f(\rvx^*, \rvy^*; \rvz)\|^2 ] \log{\frac{8}{\delta}}}{n}} \\
        & \quad + \frac{2 B_{\rvx^*} \log{\frac8\delta}}{n}
        + \frac{\mu_\rvx}{n}  
        + \frac{2\beta}{\mu_\rvy} \left(
    \sqrt{\frac{2 \E [ \|\nabla_\rvy f(\rvx^*, \rvy^*; \rvz)\|^2 ] \log{\frac{8}{\delta}}}{n}} 
    + \frac{B_{\rvy^*} \log{\frac8\delta}}{n}
    \right).
\end{align*}
The proof is complete.
\end{proof}

\begin{proof}[Proof of Remark \ref{remark:main-PL}]
    Here we briefly prove the results given in Remark \ref{remark:main-PL}. Since $\Phi(\rvx)$ satisfies the PL condition with $\mu_\rvx$,
     we have 
     we have 
\begin{align}
\label{eq:remark-main-PL-1-1}
    \Phi(\rvx) - \Phi(\rvx^*) \leq \frac{\| \nabla \Phi(\rvx) \|^2}{2\mu_\rvx}.
\end{align}
Therefore, we need to bound $\| \nabla \Phi(\rvx) \|^2$. According to Theorem \ref{theorem:nablaf-without-d}, for any $\delta \in (0,1)$, when $ n \geq \frac{c\beta^2(\mu_\rvy+\beta)^4(d+ \log{\frac{16\log_2{\sqrt{2}R_1 n + 1}}{\delta}})}{\mu_\rvy^4\mu_\rvx^2}$, with probability at least $1-\delta$,
\begin{equation}
\begin{aligned}
\label{eq:remark-main-PL-1-2}
    &  \| \nabla \Phi(\rvx) \| \leq 2 \| \nabla \Phi_S(\rvx) \| 
        + 2\sqrt{\frac{2 \E [ \|\nabla_\rvx f(\rvx^*, \rvy^*; \rvz)\|^2 ] \log{\frac{8}{\delta}}}{n}} \\
        & \quad + \frac{2 B_{\rvx^*} \log{\frac8\delta}}{n}
        + \frac{\mu_\rvx}{n}  
        + \frac{2\beta}{\mu_\rvy} \left(
    \sqrt{\frac{2 \E [ \|\nabla_\rvy f(\rvx^*, \rvy^*; \rvz)\|^2 ] \log{\frac{8}{\delta}}}{n}} 
    + \frac{B_{\rvy^*} \log{\frac8\delta}}{n}
    \right).
\end{aligned}    
\end{equation}

Then, substituting (\ref{eq:remark-main-PL-1-2}) into (\ref{eq:remark-main-PL-1-1}), we have
\begin{align*}
     \Phi(\rvx) - \Phi(\rvx^*) 
            \leq &
            \frac{\| \nabla \Phi(\rvx) \|^2}{2\mu_\rvx} \\
            \leq & \frac{1}{2\mu_\rvx} \Bigg\{ 
            2 \| \nabla \Phi_S(\rvx) \| +
            \frac{2\beta}{\mu_\rvy} \left(
    \sqrt{\frac{2 \E [ \|\nabla_\rvy f(\rvx^*, \rvy^*; \rvz)\|^2 ] \log{\frac{8}{\delta}}}{n}} 
    + \frac{B_{\rvy^*} \log{\frac8\delta}}{n}
    \right) \\
        & \quad + 2\sqrt{\frac{2 \E [ \|\nabla_\rvx f(\rvx^*, \rvy^*; \rvz)\|^2 ] \log{\frac{8}{\delta}}}{n}} 
        + \frac{2 B_{\rvx^*} \log{\frac{8}{\delta}}}{n}
        + \frac{\mu_\rvx}{n}
        \Bigg\}^2 \\
       \leq  &         
        \frac{8 \| \nabla \Phi_S(\rvx) \|^2}{\mu_\rvx} +
    \frac{16 \beta^2\E [ \|\nabla_\rvy f(\rvx^*, \rvy^*; \rvz)\|^2 ] \log{\frac{8}{\delta}}}{\mu_\rvx \mu_\rvy^2 n} \\
        & \quad + \frac{16 \E [ \|\nabla_\rvx f(\rvx^*, \rvy^*; \rvz)\|^2 ] \log{\frac{8}{\delta}}}{\mu_\rvx n}  
        + \frac{ 2 \left(
        \frac{2\beta B_{\rvy^*}}{\mu_\rvy} \log{\frac{8}{\delta}} + 2B_{\rvx^*} \log{\frac{8}{\delta}} + \mu_\rvx
        \right)^2
        }{\mu_\rvx n^2},
\end{align*}
where the second inequality holds with Cauchy–Bunyakovsky–Schwarz inequality.

The proof is complete.
\end{proof}

\section{Application}

\subsection{Empirical saddle point}

\begin{proof}[Proof of Theorem \ref{theorem:esp-nabla}]
   Plugging $\hat{\rvx}^*$ into Theorem \ref{theorem:main}, for any $\delta \in (0,1)$, with probability at least $1-\delta$, we have
   \begin{align*}
       &  \| \nabla \Phi(\hat{\rvx}^*) \| - \| \nabla \Phi_S(\hat{\rvx}^*) \| 
    \leq \frac{\beta}{\mu_\rvy} \left(
    \sqrt{\frac{2 \E [ \|\nabla_\rvy f(\rvx^*, \rvy^*; \rvz)\|^2 ] \log{\frac{8}{\delta}}}{n}} 
    + \frac{B_{\rvy^*} \log{\frac8\delta}}{n}
    \right) \\
    & + \sqrt{\frac{2 \E [ \|\nabla_\rvx f(\rvx^*, \rvy^*; \rvz)\|^2 ] \log{\frac{8}{\delta}}}{n}} 
    + \frac{B_{\rvx^*} \log{\frac8\delta}}{n}
    + \frac{C\beta(\mu_\rvy+\beta)}{\mu_\rvy} \frac{(\mu_\rvy+\beta)}{\mu_\rvy}
    \max \left\{ \| \rvx - \rvx^* \|, \frac{1}{n} \right\}  \\
    &  \times \left(
    \sqrt{\frac{d+\log{\frac{16\log_2(\sqrt{2}R_1 n+1)}{\delta}}}{n}} + 
    \frac{d+\log{\frac{16\log_2(\sqrt{2}R_1 n+1)}{\delta}}}{n}
    \right).
   \end{align*}
    Since $\hat{\rvx}^*$ is the solution of (\ref{eq:empirical-minimax}), there holds that $\| \nabla \Phi_S(\hat{\rvx}^*) \| = 0$. Thus, we can derive that 
    \begin{align*}
          \| \nabla \Phi(\hat{\rvx}^*) \| 
    & \leq \frac{\beta}{\mu_\rvy} \left(
    \sqrt{\frac{2 \E [ \|\nabla_\rvy f(\rvx^*, \rvy^*; \rvz)\|^2 ] \log{\frac{8}{\delta}}}{n}} 
    + \frac{B_{\rvy^*} \log{\frac8\delta}}{n}
    \right) \\
    & \quad + \sqrt{\frac{2 \E [ \|\nabla_\rvx f(\rvx^*, \rvy^*; \rvz)\|^2 ] \log{\frac{8}{\delta}}}{n}} 
    + \frac{B_{\rvx^*} \log{\frac8\delta}}{n}
    + \frac{C\beta(\mu_\rvy+\beta)}{\mu_\rvy} \frac{(\mu_\rvy+\beta)}{\mu_\rvy}
    \left( R_1 + \frac{1}{n} \right)  \\
    & \quad \quad \times \left(
    \sqrt{\frac{d+\log{\frac{16\log_2(\sqrt{2}R_1 n+1)}{\delta}}}{n}} + 
    \frac{d+\log{\frac{16\log_2(\sqrt{2}R_1 n+1)}{\delta}}}{n}
    \right) \\
    & = O\left(
    \sqrt{
    \frac{d + \log{\frac{\log{n}}{\delta}}}{n}
    }
    \right).
    \end{align*}
The proof is complete.
\end{proof}

\begin{proof}[Proof of Theorem \ref{theorem:esp-fast-rate}]
According to Lemma \ref{lemma:primal-function-is-PL}, $\Phi(\rvx)$ satisfies the PL condition with $\mu_\rvx$, we have 
\begin{align}
\label{eq:esp-1-1}
    \Phi(\hat{\rvx}^*) - \Phi(\rvx^*) \leq \frac{\| \nabla \Phi(\hat{\rvx}^*) \|^2}{2\mu_\rvx}.
\end{align}
Therefore, we need to bound $\| \nabla \Phi(\hat{\rvx}^*) \|^2$. Plugging $\hat{\rvx}^*$ into Theorem \ref{theorem:nablaf-without-d}, for any $\delta \in (0,1)$, when $ n \geq \frac{c\beta^2(\mu_\rvy+\beta)^4(d+ \log{\frac{16\log_2{\sqrt{2}R_1 n + 1}}{\delta}})}{\mu_\rvy^4\mu_\rvx^2}$, with probability at least $1-\delta$
\begin{equation}
\begin{aligned}
\label{eq:esp-1-2}
    &  \| \nabla \Phi(\hat{\rvx}^*) \| \leq 2 \| \nabla \Phi_S(\hat{\rvx}^*) \| 
        + 2\sqrt{\frac{2 \E [ \|\nabla_\rvx f(\rvx^*, \rvy^*; \rvz)\|^2 ] \log{\frac{8}{\delta}}}{n}} \\
        & \quad + \frac{2 B_{\rvx^*} \log{\frac8\delta}}{n}
        + \frac{\mu_\rvx}{n}  
        + \frac{2\beta}{\mu_\rvy} \left(
    \sqrt{\frac{2 \E [ \|\nabla_\rvy f(\rvx^*, \rvy^*; \rvz)\|^2 ] \log{\frac{8}{\delta}}}{n}} 
    + \frac{B_{\rvy^*} \log{\frac8\delta}}{n}
    \right).
\end{aligned}    
\end{equation}

Since $\nabla \Phi_S(\hat{\rvx}^*)  =  \nabla_\rvx F_S(\hat{\rvx}^*, \hat{\rvy}^*) = \bm{0}$, we have $\|\nabla \Phi_S(\hat{\rvx}^*)\| =  \|\nabla_\rvx F_S(\hat{\rvx}^*, \hat{\rvy}^*)\| = 0$. 
By plugging (\ref{eq:esp-1-2}) into (\ref{eq:esp-1-1}), we have
\begin{align*}
      \Phi(\hat{\rvx}^*) - \Phi(\rvx^*) 
            \leq &
            \frac{\| \nabla \Phi(\hat{\rvx}^*) \|^2}{2\mu_\rvx} \\
            \leq & \frac{1}{2\mu_\rvx} \Bigg\{
            \frac{2\beta}{\mu_\rvy} \left(
    \sqrt{\frac{2 \E [ \|\nabla_\rvy f(\rvx^*, \rvy^*; \rvz)\|^2 ] \log{\frac{8}{\delta}}}{n}} 
    + \frac{B_{\rvy^*} \log{\frac8\delta}}{n}
    \right) \\
        & \quad + 2\sqrt{\frac{2 \E [ \|\nabla_\rvx f(\rvx^*, \rvy^*; \rvz)\|^2 ] \log{\frac{8}{\delta}}}{n}} 
        + \frac{2 B_{\rvx^*} \log{\frac{8}{\delta}}}{n}
        + \frac{\mu_\rvx}{n}
        \Bigg\}^2 \\
        \leq  & 
        \frac{12 \beta^2\E [ \|\nabla_\rvy f(\rvx^*, \rvy^*; \rvz)\|^2 ] \log{\frac{8}{\delta}}}{\mu_\rvx \mu_\rvy^2 n}
        + \frac{12 \E [ \|\nabla_\rvx f(\rvx^*, \rvy^*; \rvz)\|^2 ] \log{\frac{8}{\delta}}}{\mu_\rvx n} \\ 
        & \quad + \frac{ 3 \left(
        \frac{2\beta B_{\rvy^*}}{\mu_\rvy} \log{\frac{8}{\delta}} + 2B_{\rvx^*} \log{\frac{8}{\delta}} + \mu_\rvx
        \right)^2
        }{2 \mu_\rvx n^2},
\end{align*}
where the second inequality holds with Cauchy–Bunyakovsky–Schwarz inequality.

Next, if we further assume $f(\rvx,\rvy;\rvz)$ is non-negative, 
according to Lemma \ref{lemma:smooth-self-bounding-property}, we have $\E \|\nabla_\rvx f(\rvx^*, \rvy^*; \rvz)\|^2 \leq 4\beta \E [f(\rvx^*, \rvy^*; \rvz)] $ and $ \E \|\nabla_\rvy f(\rvx^*, \rvy^*; \rvz)\|^2 \leq 4\beta \E [f(\rvx^*, \rvy^*; \rvz)]$. Then we have 
\begin{align*}
    \Phi(\hat{\rvx}^*) - \Phi(\rvx^*) 
         & \leq 
        \frac{48 \beta^3 \E [f(\rvx^*, \rvy^*; \rvz)] \log{\frac{8}{\delta}}}{\mu_\rvx \mu_\rvy^2 n}
        + \frac{48 \beta \E [f(\rvx^*, \rvy^*; \rvz)]  \log{\frac{8}{\delta}}}{\mu_\rvx n} \\ 
        & \quad + \frac{ 3 \left(
        \frac{2\beta B_{\rvy^*}}{\mu_\rvy} \log{\frac{8}{\delta}} + 2B_{\rvx^*} \log{\frac{8}{\delta}} + \mu_\rvx
        \right)^2
        }{2 \mu_\rvx n^2} \\
        & = 
        \frac{48 \beta (\beta^2 + \mu_\rvy^2) \Phi(\rvx^*) \log{\frac{8}{\delta}}}{\mu_\rvx \mu_\rvy^2 n}
        + \frac{ 3 \left(
        \frac{2\beta B_{\rvy^*}}{\mu_\rvy} \log{\frac{4}{\delta}} + 2B_{\rvx^*} \log{\frac{8}{\delta}} + \mu_\rvx
        \right)^2
        }{2 \mu_\rvx n^2} \\
        & = O\left(
        \frac{\Phi(\rvx^*) \log{\frac{1}{\delta}}}{n} + \frac{\log^2{\frac{1}{\delta}}}{n^2}
        \right).
\end{align*}
The proof is complete.
\end{proof}

\subsection{Gradient descent ascent}
\label{sec:appendix-GDA}

Firstly, we introduce the optimization error bound for GDA given in \cite{lin2020gradient}.

\begin{lemma}[Optimization error bound for NC-SC minimax problems \cite{lin2020gradient}]
\label{lemma:GDA-NC-SC-optimization-error-bound}
    under Assumption \ref{assumption:NC-SC}, and letting the step sizes be chosen as $\eta_\rvx = \frac{1}{16(\frac{\beta}{\mu} + 1)^2\beta}$ and $\eta_\rvy= \frac{1}{\beta}$, then the optimization error bound of \text{Algorithm} \ref{algo:gda} can be bounded by
    \begin{align*}
        \frac{1}{T}
        \sum_{t=1}^T \| \nabla \Phi_S(\rvx_t) \|^2         
        \leq
        \frac{128\beta^3 \Delta_\Phi}{\mu_\rvy^2 T} 
        + 
        \frac{5\beta^3 D_\gY}{\mu_\rvy T},
    \end{align*}
    where $\Delta_\Phi = \Phi_S(\rvx_0) - \min_\rvx \Phi_S(\rvx)$.
\end{lemma}

\begin{proof}[Proof of Theorem \ref{theorem:gda-nabla}]
    Firstly, we have
    \begin{equation}
        \begin{aligned}
            \frac{1}{T}\sum_{t=1}^T \| \nabla \Phi(\rvx_t) \|^2 
            \leq & \frac{2}{T}\sum_{t=1}^T \| \nabla \Phi(\rvx_t) -\nabla_\rvx \Phi_S(\rvx_t) \|^2 + \frac{2}{T}\sum_{t=1}^T \| \nabla_\rvx \Phi_S(\rvx_t) \|^2 \\
            \leq & \frac{2}{T}\sum_{t=1}^T \max_{t=1,\cdots,T} \| \nabla \Phi(\rvx_t) -\nabla_\rvx \Phi_S(\rvx_t) \|^2 + \frac{2}{T}\sum_{t=1}^T \| \nabla_\rvx \Phi_S(\rvx_t) \|^2 \\
            \leq & \frac{2}{T}\sum_{t=1}^T \max_{t=1,\cdots,T} \| \nabla \Phi(\rvx_t) -\nabla_\rvx \Phi_S(\rvx_t) \|^2 
            + \frac{256\beta^3 \Delta_\Phi}{\mu_\rvy^2 T} + \frac{10\beta^3 D_\gY}{\mu_\rvy T},
        \end{aligned}
    \end{equation}
    Thus, with probability at least $1-\delta$, we have
    \begin{equation}
        \begin{aligned}
        \label{eq:theorem:gda-nabla-1-1}
            \frac{1}{T}\sum_{t=1}^T \| \nabla \Phi(\rvx_t) \|^2 
            \leq & O\left(\frac{1}{T}\right) + O \Bigg( \max_{t=1,\cdots,T} \Bigg[ 
            \frac{C\beta(\mu_\rvy+\beta)}{\mu_\rvy} \frac{(\mu_\rvy+\beta)}{\mu_\rvy}
            \max \left\{ \| \rvx_t - \rvx^* \|, \frac{1}{n} \right\}  \\
            & \times \Bigg(
            \sqrt{\frac{d+\log{\frac{16\log_2(\sqrt{2}R_1 n+1)}{\delta}}}{n}} + 
            \frac{d+\log{\frac{16\log_2(\sqrt{2}R_1 n+1)}{\delta}}}{n}
            \Bigg)
            \Bigg]^2
            \Bigg).
        \end{aligned}
    \end{equation}
where the inequality holds according to Theorem \ref{theorem:main}.

Next, we need to bound $\| \rvx_t - \rvx^* \|$. Since we assume that $\rvx_1 = 0$, and $\rvx_{t+1} = \rvx_t - \eta_\rvx \nabla_\rvx F_S(\rvx_t, \rvy_t)$, we have $\rvx_{t+1} = -  \eta_\rvx \sum_{k=1}^t \nabla_\rvx F_S(\rvx_t, \rvy_t)$. then we have

\begin{equation}
\begin{aligned}
\label{eq:theorem:gda-nabla-1-2}
    & \left\| \rvx_{t+1} - \rvx^* \right\| \leq \| \rvx_{t+1} \| + \| \rvx^* \| \\
    \leq & \left\| \eta_\rvx \sum_{k=1}^t \nabla_\rvx F_S(\rvx_t, \rvy_t) \right\| + \| \rvx^* \| 
    = O \left( \eta_\rvx L t \right).
\end{aligned}    
\end{equation}

Then plugging (\ref{eq:theorem:gda-nabla-1-2}) into (\ref{eq:theorem:gda-nabla-1-1}), with probability at least $1-\delta$
\begin{align*}
    \frac{1}{T}\sum_{t=1}^T \| \nabla \Phi(\rvx_t) \|^2 
    \leq & O\left( \frac{1}{T} \right) 
    + 
    O \left( 
    \frac{d+\log{\frac{16\log_2(\sqrt{2}R_1 n+1)}{\delta}}}{n} T
    \right).
\end{align*}
Let $T \asymp O\left(\sqrt{\frac{n}{d}}\right)$, we can derive that 
\begin{align*}
    \frac{1}{T}\sum_{t=1}^T \| \nabla \Phi(\rvx_t) \|^2 \leq O\left(
    \frac{d^{\frac{1}{2}} + d^{-\frac{1}{2}} \log{\frac{\log{n}}{\delta}} }{n^{\frac{1}{2}}}
    \right).
\end{align*}
The proof is complete.
\end{proof}

\begin{proof}[Proof of Theorem \ref{theorem:gda-fast-rate}]
    According to Lemma \ref{lemma:primal-function-is-PL}, $\Phi(\rvx)$ satisfies the PL assumption with parameter $\mu_\rvx$, we have
    \begin{align}
    \label{eq:theorem:gda-fast-rate-1-1}
        \Phi(\rvx) - \Phi(\rvx^*) \leq \frac{\| \nabla \Phi(\rvx) \|^2}{2\mu_\rvx}.
    \end{align}
    To bound $\Phi(\bar{\rvx}_T) - \Phi(\rvx^*)$, we firstly need to bound the term $ \| \nabla \Phi(\bar{\rvx}_T) \|^2 $. There holds that
    \begin{align}
    \label{eq:theorem:gda-fast-rate-1-2}
        \| \nabla \Phi(\bar{\rvx}_T) \|^2 \leq 2 \| \nabla \Phi(\bar{\rvx}_T) - 
        \nabla \Phi_S(\bar{\rvx}_T) \|^2 
        + 2 \| \nabla \Phi_S(\bar{\rvx}_T) \|^2.
    \end{align}
    From Theorem \ref{theorem:nablaf-without-d}, under Assumption \ref{assumption:minimax-bernstein-condition} and \ref{assumption:SC-SC}, plugging $\bar{\rvx}_T$ into Theorem \ref{theorem:nablaf-without-d}, for any $\delta \in (0,1)$, when 
    $ n \geq \frac{c\beta^2(\mu_\rvy+\beta)^4(d+ \log{\frac{16\log_2{\sqrt{2}R_1 n + 1}}{\delta}})}{\mu_\rvy^4\mu_\rvx^2}$, 
    with probability at least $1-\delta$

    \begin{equation}
        \begin{aligned}
    \label{eq:theorem:gda-fast-rate-1-3}
        &  \| \nabla \Phi(\bar{\rvx}_T) - \nabla \Phi_S(\bar{\rvx}_T) \| 
        \leq \| \nabla \Phi_S(\bar{\rvx}_T) \|
            + 2\sqrt{\frac{2 \E [ \|\nabla_\rvx f(\rvx^*, \rvy^*; \rvz)\|^2 ] \log{\frac{8}{\delta}}}{n}} \\
            & \quad + \frac{2 B_{\rvx^*} \log{\frac{8}{\delta}}}{n}
            + \frac{\mu_\rvx}{n}  
            + \frac{2\beta}{\mu_\rvy} \left(
        \sqrt{\frac{2 \E [ \|\nabla_\rvy f(\rvx^*, \rvy^*; \rvz)\|^2 ] \log{\frac{8}{\delta}}}{n}} 
        + \frac{B_{\rvy^*} \log{\frac8\delta}}{n}
        \right).
    \end{aligned}
    \end{equation}
    
    Next, we need to bound the optimization error $\| \nabla \Phi_S(\bar{\rvx}_T) \|$. According to Lemma \ref{lemma:primal-function-is-PL}, $\Phi_S(\rvx)$ satisfies the PL assumption with parameter $\mu_\rvx$, then using Lemma \ref{lemma:GDA-NC-SC-optimization-error-bound} we have
    \begin{align*}
        \frac{1}{T}\sum_{t=1}^T \Phi_S(\rvx_t) - \Phi_S(\rvx^*) 
        \leq \frac{1}{2\mu_\rvx T} \sum_{t=1}^T \| \nabla \Phi_S(\rvx_t) \|^2 \leq 
        \frac{64\beta^3 \Delta_\Phi}{\mu_\rvx \mu_\rvy^2 T } 
        + 
        \frac{5\beta^3 D_\gY}{2 \mu_\rvx \mu_\rvy T }.
    \end{align*}
    From the convexity of $F_S(\cdot, \rvy)$, we get
    \begin{align*}
        \Phi_S(\bar{\rvx}_T) - \Phi_S(\rvx^*) \leq \frac{1}{T}\sum_{t=1}^T \Phi_S(\rvx_t) - \Phi_S(\rvx^*) 
        \leq 
        \frac{64\beta^3 \Delta_\Phi}{\mu_\rvx \mu_\rvy^2 T } 
        + 
        \frac{5\beta^3 D_\gY}{2 \mu_\rvx \mu_\rvy T }.
    \end{align*}
    According to \cite{nesterov2003introductory} and Lemma \ref{lemma:primal-function-is-smooth}, there holds the following property for $\beta + \frac{\beta^2}{\mu_\rvy}$ function $\Phi_S(\rvx)$, we have
    \begin{align}
    \label{eq:theorem:gda-fast-rate-1-4}
        \frac{1}{2\left(\beta + \frac{\beta^2}{\mu_\rvy}\right)} \| \nabla \Phi_S(\bar{\rvx}_T)\|^2 
        \leq
        \Phi_S(\bar{\rvx}_T) - \Phi_S(\rvx^*) 
        \leq
        \frac{64\beta^3 \Delta_\Phi}{\mu_\rvx \mu_\rvy^2 T} 
        + 
        \frac{5\beta^3 D_\gY}{2 \mu_\rvx \mu_\rvy T}.
    \end{align}
    Plugging (\ref{eq:theorem:gda-fast-rate-1-4}) into (\ref{eq:theorem:gda-fast-rate-1-3}), according to  Cauchy–Bunyakovsky–Schwarz inequality, we can derive that
    \begin{equation}
    \begin{aligned}
    \label{eq:theorem:gda-fast-rate-1-5}
        & \| \nabla \Phi(\bar{\rvx}_T) - \nabla \Phi_S(\bar{\rvx}_T) \|^2 
        \leq  O\left( \frac{1}{T} \right) 
        + \frac{32 \beta^2\E [ \|\nabla_\rvy f(\rvx^*, \rvy^*; \rvz)\|^2 ] \log{\frac{8}{\delta}}}{\mu_\rvy^2 n} \\
        & \quad + \frac{32 \E [ \|\nabla_\rvx f(\rvx^*, \rvy^*; \rvz)\|^2 ] \log{\frac{8}{\delta}}}{n}
        + \frac{ 4 \left(
        \frac{2\beta B_{\rvy^*}}{\mu_\rvy} \log{\frac{8}{\delta}} + 2B_{\rvx^*} \log{\frac{8}{\delta}} + \mu_\rvx
        \right)^2
        }{n^2}.
    \end{aligned}
    \end{equation}
    Then substituting (\ref{eq:theorem:gda-fast-rate-1-5}), (\ref{eq:theorem:gda-fast-rate-1-4}) into (\ref{eq:theorem:gda-fast-rate-1-2}), we derive that
    \begin{equation}
        \begin{aligned}
        \label{eq:theorem:gda-fast-rate-1-6}
         & \| \nabla \Phi(\bar{\rvx}_T) \|^2 
        \leq  O\left( \frac{1}{T} \right) 
        + \frac{64 \beta^2\E [ \|\nabla_\rvy f(\rvx^*, \rvy^*; \rvz)\|^2 ] \log{\frac{8}{\delta}}}{\mu_\rvy^2 n} \\
        & \quad + \frac{64 \E [ \|\nabla_\rvx f(\rvx^*, \rvy^*; \rvz)\|^2 ] \log{\frac{8}{\delta}}}{n}
        + \frac{ 8 \left(
        \frac{2\beta B_{\rvy^*}}{\mu_\rvy} \log{\frac{8}{\delta}} + 2B_{\rvx^*} \log{\frac{8}{\delta}} + \mu_\rvx
        \right)^2
        }{n^2}.
        \end{aligned}
    \end{equation}

    Finally, we plug (\ref{eq:theorem:gda-fast-rate-1-6}) into (\ref{eq:theorem:gda-fast-rate-1-1}) and choose $ T \asymp O\left( n \right) $, with probability at least $1-\delta$
    \begin{align*}
        \Phi(\bar{\rvx}_T) - \Phi(\rvx^*) 
        = O \Bigg( \frac{\E [ \|\nabla_\rvx f(\rvx^*, \rvy^*; \rvz)\|^2 ] \log{\frac{1}{\delta}}}{n} 
        + \frac{\E [ \|\nabla_\rvy f(\rvx^*, \rvy^*; \rvz)\|^2 ] \log{\frac{1}{\delta}}}{n} + \frac{\log^2{\frac{1}{\delta}}}{n^2} \Bigg).
    \end{align*}

Next, if we further assume $f(\rvx,\rvy;\rvz)$ is non-negative, According to Lemma \ref{lemma:smooth-self-bounding-property}, we have $\E  \|\nabla_\rvx f(\rvx^*, \rvy^*; \rvz)\|^2  \leq 4\beta \E [f(\rvx^*, \rvy^*; \rvz)] $ and $ \E \|\nabla_\rvy f(\bar{\rvx}_T, \rvy^*(\bar{\rvx}_T); \rvz)\|^2  \leq 4\beta \E [f(\rvx^*, \rvy^*; \rvz)]$. Plugging (\ref{eq:theorem:gda-fast-rate-1-6}) into (\ref{eq:theorem:gda-fast-rate-1-1}), we have 
\begin{align*}
        & \Phi(\bar{\rvx}_T) - \Phi(\rvx^*) \\
          \leq & \frac{\| \nabla \Phi(\bar{\rvx}_T) \|^2}{2\mu_\rvx} \\
          \leq & O\left( \frac{1}{T} \right) + 
        \frac{128 \beta^3 \E [f(\rvx^*, \rvy^*; \rvz)] \log{\frac{8}{\delta}}}{\mu_\rvx \mu_\rvy^2 n}
        + \frac{128 \beta \E [f(\rvx^*, \rvy^*; \rvz)]  \log{\frac{8}{\delta}}}{\mu_\rvx n} \\ 
        & \quad + \frac{ 4 \left(
        \frac{2\beta B_{\rvy^*}}{\mu_\rvy} \log{\frac{8}{\delta}} + 2B_{\rvx^*} \log{\frac{8}{\delta}} + \mu_\rvx
        \right)^2
        }{\mu_\rvx n^2} \\
         = & O\left( \frac{1}{T} \right) + 
        \frac{128 \beta (\beta^2 + \mu_\rvy^2) \Phi(\rvx^*) \log{\frac{8}{\delta}}}{\mu_\rvx \mu_\rvy^2 n}
        + \frac{4 \left(
        \frac{2\beta B_{\rvy^*}}{\mu_\rvy} \log{\frac{8}{\delta}} + 2B_{\rvx^*} \log{\frac{8}{\delta}} + \mu_\rvx
        \right)^2
        }{\mu_\rvx n^2},
\end{align*}

When $T \asymp n^2$ we have 
\begin{align*}
    \Phi(\bar{\rvx}_T) - \Phi(\rvx^*) = O \left(
    \frac{\Phi(\rvx^*) \log{\frac{1}{\delta}} }{n}
    + \frac{\log^2{\frac{1}{\delta}}}{n^2}
    \right).
\end{align*}
The proof is complete.
\end{proof}

\subsection{Stochastic gradient descent ascent}
\label{sec:appendix-SGDA}

In this subsection, we present empirical optimization error bounds of primal functions for SGDA, which are motivated by \cite{lei2021stability,li2021high}. The proofs are standard in the literature \cite{nedic2009subgradient,nemirovski2009robust} and we give the optimization error bounds with high probability. Firstly, we introduce two concentration inequalities for martingales.

\begin{lemma}[\cite{boucheron2013concentration}]
\label{lemma:Azuma-Hoeffding-inequality}
    Let $z_1, \ldots, z_n$ be a sequence of random variables such that $z_k$ may depend the previous variables $z_1, \ldots, z_{k-1}$ for all $k=1,\ldots,n$. Consider a sequence of functionals $\xi_k(z_1, \ldots, z_k), k = 1, \ldots, n$. Assume $| \xi_k - \E_{z_k} [\xi_k]| \leq b_k$ for each $k$. Let $\delta \in (0,1)$. With probability at least $1-\delta$
    \begin{align*}
        \sum_{k=1}^n \xi_k - \sum_{k=1}^n \E_{z_k} [\xi_k] 
        \leq
        \left(
        2\sum_{k=1}^n b_k^2 \log{\frac{1}{\delta}}
        \right)^{\frac{1}{2}}.
    \end{align*}
\end{lemma}

\begin{lemma}[\cite{tarres2014online}]
\label{lemma:Bernstein-type-martingale-inequality}
    Let $\{ \xi_k \}_{k \in \mathbb{N}}$ be a martingale difference sequence in $\R^d$. Suppose that almost surely $\| \xi_k \| \leq D$ and $\sum_{k=1}^t \E [\|\xi_k\|^2 | \xi_1,\ldots,\xi_{k-1}] \leq \sigma_t^2 $. Then, for any $\delta \in (0,1)$, the following inequality holds with probability at least $1-\delta$
    \begin{align*}
        \max_{1 \leq j \leq t} \left\| \sum_{k=1}^j \xi_k \right\|
        \leq 2\left( \frac{D}{3} + \sigma_t \right) \log{\frac{2}{\delta}}.
    \end{align*}
\end{lemma}

The following lemma shows the optimization error bounds of primal function for SGDA.

\begin{lemma}
\label{lemma:SGDA-SC-SC-optimization-error-bound}
    Suppose Assumption \ref{assumption:SC-SC} and \ref{assumption:f-lipschitz} hold and let the stepsizes be chosen as $\eta_{\rvx_t} = \frac{1}{\mu_\rvx(t+t_0)}$ and $\eta_{\rvy_t} = \frac{1}{\mu_\rvy(t+t_0)} $, for any $\delta \in (0,1)$, with probability at least $1-\delta$, then the optimization error of \text{Algorithm} \ref{algo:sgda} can be bounded by
    \begin{align*}
        &\Phi_S(\bar{\rvx}_T) - \Phi_S(\hat{\rvx}^*) \leq
        \frac{t_0(\mu_\rvx D_\rvx + \mu_\rvy D_\rvy)}{2T} + \frac{L^2 \log(eT)}{2T} \left( \frac{1}{\mu_\rvx} + \frac{1}{\mu_\rvy}\right) \\
        & \quad +  \frac{2(\sqrt{D_\gX} + \sqrt{D_\gY})}{T} \left( \frac{2L}{3} + 2L\sqrt{T} \right) \log{\frac{6}{\delta}}
        + \frac{2L (\sqrt{D_\gX} + \sqrt{D_\gY}) \left(2T\log{\frac{6}{\delta}}\right)^{\frac{1}{2}}}{T}.
    \end{align*}
\end{lemma}

\begin{proof}[Proof of Lemma \ref{lemma:SGDA-SC-SC-optimization-error-bound}]
    This proof mainly follows from \cite{lei2021stability,li2021high}. Firstly, we have
    \begin{align*}
        \| \rvx_{t+1} - \rvx \|^2 
        & = \| \rvx_t - \eta_{\rvx_t} \nabla_\rvx f(\rvx_t, \rvy_t; \rvz_{i_t}) - \rvx \|^2 \\
        & = \| \rvx_t - \rvx \|^2 + \eta_{\rvx_t}^2 \|\nabla_\rvx f(\rvx_t, \rvy_t; \rvz_{i_t})\|^2 + 2\eta_{\rvx_t} \langle \rvx - \rvx_t, \nabla_\rvx f(\rvx_t, \rvy_t; \rvz_{i_t}) \rangle \\
        & \leq  \| \rvx_t - \rvx \|^2 + \eta_{\rvx_t}^2 L^2 + 2\eta_{\rvx_t} \langle \rvx - \rvx_t, \nabla_\rvx f(\rvx_t, \rvy_t; \rvz_{i_t}) - \nabla_\rvx F_S (\rvx_t, \rvy_t) \rangle \\
        & \quad + 2\eta_{\rvx_t} \langle \rvx - \rvx_t, \nabla_\rvx F_S(\rvx_t, \rvy_t) \rangle,
    \end{align*}
    where the first inequality holds because of Assumption \ref{assumption:f-lipschitz}. According to the strong convexity of $F_S(\cdot, \rvy_t)$, we have
    \begin{align*}
        2\eta_{\rvx_t} (F_S(\rvx_t, \rvy_t) - F_S(\rvx, \rvy_t)) 
        \leq (1-\eta_{\rvx_t} \mu_\rvx)\| \rvx_t - \rvx \|^2 - \| \rvx_{t+1} - \rvx \|^2 + \eta_{\rvx_t}^2 L^2 \\
        + 2\eta_{\rvx_t} \langle \rvx -\rvx_t, \nabla_\rvx f(\rvx_t, \rvy_t; \rvz_{i_t}) - \nabla_\rvx F_S(\rvx_t, \rvy_t)\rangle.
    \end{align*}

    Let $\eta_{\rvx_t} = \frac{1}{\mu_\rvx(t+t_0)}$, we further get
    \begin{align*}
        \frac{2}{\mu_\rvx (t +t_0)} (F_S(\rvx_t, \rvy_t) - F_S(\rvx, \rvy_t)) \leq
        \left( 1-\frac{1}{t+t_0} \right) \|\rvx_t -\rvx\|^2 - \|\rvx_{t+1}-\rvx\|^2 \\
        + \left( \frac{L}{\mu_\rvx(t+t_0)} \right)^2  + \frac{2}{\mu_\rvx(t+t_0)} \langle \rvx -\rvx_t, \nabla_\rvx f(\rvx_t, \rvy_t; \rvz_{i_t}) - \nabla_\rvx F_S(\rvx_t, \rvy_t)\rangle.
    \end{align*}
    Multiplying both sides by $t+t_0$, we have
    \begin{align*}
        \frac{2}{\mu_\rvx} (F_S(\rvx_t, \rvy_t) - F_S(\rvx, \rvy_t)) \leq
        \left( t + t_0 -1 \right) \|\rvx_t -\rvx\|^2 - (t+t_0) \|\rvx_{t+1}-\rvx\|^2 \\
        + \frac{L^2}{\mu_\rvx^2(t+t_0)} + \frac{2}{\mu_\rvx} \langle \rvx -\rvx_t, \nabla_\rvx f(\rvx_t, \rvy_t; \rvz_{i_t}) - \nabla_\rvx F_S(\rvx_t, \rvy_t)\rangle.
    \end{align*}
    Since $\rvx_1 = 0$ and $\sum_{t=1}^T t^{-1} \leq \log{(eT)}$, by taking a summation of the above inequality from $t=1$ to $T$, we have 
    \begin{align*}
        & \sum_{t=1}^T (F_S(\rvx_t, \rvy_t) - F_S(\rvx, \rvy_t)) \leq
        \frac{\mu_\rvx}{2}t_0 D_\gX
        + \frac{L^2 \log{(eT)}}{2\mu_\rvx}  \\
        & + \sum_{t=1}^T \langle \rvx, \nabla_\rvx f(\rvx_t, \rvy_t; \rvz_{i_t}) - \nabla_\rvx F_S(\rvx_t, \rvy_t)\rangle 
        + \sum_{t=1}^T \langle \rvx_t, \nabla_\rvx F_S(\rvx_t, \rvy_t) - \nabla_\rvx f(\rvx_t, \rvy_t; \rvz_{i_t})\rangle.
    \end{align*}
    Since $\rvy_S^*(\rvx) = \argmax_{y\in \gY} F_S(\rvx,\rvy)$, for any $\rvx \in \gX$, we obtain that $F_S(\rvx, \rvy_t) \leq F_S(\rvx, \rvy_S^*(\rvx))$. Then we have
    \begin{align*}
        & \sum_{t=1}^T (F_S(\rvx_t, \rvy_t) - F_S(\rvx, \rvy_S^*(\rvx))) \leq
        \frac{\mu_\rvx}{2}t_0 D_\gX
        + \frac{L^2 \log{(eT)}}{2\mu_\rvx}  \\
        & + \sum_{t=1}^T \langle \rvx, \nabla_\rvx f(\rvx_t, \rvy_t; \rvz_{i_t}) - \nabla_\rvx F_S(\rvx_t, \rvy_t)\rangle 
        + \sum_{t=1}^T \langle \rvx_t, \nabla_\rvx F_S(\rvx_t, \rvy_t) - \nabla_\rvx f(\rvx_t, \rvy_t; \rvz_{i_t})\rangle.
    \end{align*}
    Since this inequality holds for any $\rvx$, we get
    \begin{align*}
        & \sum_{t=1}^T (F_S(\rvx_t, \rvy_t) - \inf_{\rvx \in \gX} F_S(\rvx, \rvy_S^*(\rvx))) \leq
        \frac{\mu_\rvx}{2}t_0 D_\gX
        + \frac{L^2 \log{(eT)}}{2\mu_\rvx}  \\
        & + \sum_{t=1}^T \sup_{x\in\gX} \langle \rvx, \nabla_\rvx f(\rvx_t, \rvy_t; \rvz_{i_t}) - \nabla_\rvx F_S(\rvx_t, \rvy_t)\rangle 
        + \sum_{t=1}^T \langle \rvx_t, \nabla_\rvx F_S(\rvx_t, \rvy_t) - \nabla_\rvx f(\rvx_t, \rvy_t; \rvz_{i_t})\rangle,
    \end{align*}
    which implies that
    \begin{align*}
        & \sum_{t=1}^T (F_S(\rvx_t, \rvy_t) - \Phi_S(\hat{\rvx}^*)) \leq
        \frac{\mu_\rvx}{2}t_0 D_\gX
        + \frac{L^2 \log{(eT)}}{2\mu_\rvx}  \\
        & + \sum_{t=1}^T \sup_{x\in\gX} \langle \rvx, \nabla_\rvx f(\rvx_t, \rvy_t; \rvz_{i_t}) - \nabla_\rvx F_S(\rvx_t, \rvy_t)\rangle 
        + \sum_{t=1}^T \langle \rvx_t, \nabla_\rvx F_S(\rvx_t, \rvy_t) - \nabla_\rvx f(\rvx_t, \rvy_t; \rvz_{i_t})\rangle.
    \end{align*}
    By Schwarz's inequality, we have
    \begin{align*}
        & \sum_{t=1}^T (F_S(\rvx_t, \rvy_t) - \Phi_S(\hat{\rvx}^*)) \leq
        \frac{\mu_\rvx}{2}t_0 D_\gX
        + \frac{L^2 \log{(eT)}}{2\mu_\rvx}  \\
        & +  \sqrt{D_\gX} \left\| \sum_{t=1}^T \left( \nabla_\rvx f(\rvx_t, \rvy_t; \rvz_{i_t}) - \nabla_\rvx F_S(\rvx_t, \rvy_t) \right) \right\| 
        + \sum_{t=1}^T \langle \rvx_t, \nabla_\rvx F_S(\rvx_t, \rvy_t) - \nabla_\rvx f(\rvx_t, \rvy_t; \rvz_{i_t})\rangle.
    \end{align*}
    Denote that $\xi_t = \langle \rvx_t, \nabla_\rvx F_S(\rvx_t, \rvy_t) - \nabla_\rvx f(\rvx_t, \rvx_t;\rvz_{i_t})\rangle$. Since $\E_{i_t} [\langle \rvx_t, \nabla_\rvx F_S(\rvx_t, \rvy_t) - \nabla_\rvx f(\rvx_t, \rvy_t;\rvz_{i_t})
    \rangle] = 0$, so $\{ \xi_t | t = 1,\ldots,T \}$ is a martingale difference sequence. By Schwarz's inequality and Assumption \ref{assumption:f-lipschitz}, we know that $|\langle \rvx_t, \nabla_\rvx F_S(\rvx_t, \rvy_t) - \nabla_\rvx f(\rvx_t,\rvy_t;\rvz_{i_t}) \rangle| \leq 2L \sqrt{D_\gX}$. Then according to Lemma \ref{lemma:Azuma-Hoeffding-inequality}, we have the following inequality with probability at least $1-\frac{\delta}{6}$
    \begin{align*}
        \sum_{t=1}^T \langle \rvx_t, \nabla_\rvx F_S(\rvx_t, \rvy_t) - \nabla_\rvx f(\rvx_t, \rvy_t; \rvz_{i_t})\rangle
        \leq 2L \sqrt{D_\gX} \left(2T\log{\frac{6}{\delta}}\right)^{\frac{1}{2}}.
    \end{align*}
    Define $\xi'_t = \nabla_\rvx f(\rvx_t, \rvy_t; \rvz_{i_t}) - \nabla_\rvx F_S(\rvx_t, \rvy_t)$. Then we get $\|\xi'_t\| \leq 2L$ and
    \begin{align*}
        \sum_{t=1}^T \E [\| \xi'_t \|^2|\xi'_1,\ldots, \xi'_{t-1}] \leq 4TL^2.
    \end{align*}
    Applying Lemma \ref{lemma:Bernstein-type-martingale-inequality} to the martingale difference sequence $\{ \xi'_t \}$, we have the following inequality with probability at least $1-\frac{\delta}{3}$
    \begin{align*}
        \left\| \sum_{t=1}^T \xi'_t \right\| 
        \leq 2 \left( \frac{2L}{3} + 2L\sqrt{T} \right) \log{\frac{6}{\delta}}.
    \end{align*}
    This implies that with probability at least $1-\frac{\delta}{3}$
    \begin{align*}
        \left\| \sum_{t=1}^T \left( \nabla_\rvx f(\rvx_t, \rvy_t; \rvz_{i_t}) - \nabla_\rvx F_S(\rvx_t, \rvy_t) \right) \right\| 
        \leq 2 \left( \frac{2L}{3} + 2L\sqrt{T} \right) \log{\frac{6}{\delta}}.
    \end{align*}
    Combined with the above results, we finally have the following inequality with probability at least $1-\frac{\delta}{2}$
    \begin{equation}
    \begin{aligned}
    \label{eq:sgda-optimization-error-1-1}
         & \frac{1}{T} \sum_{t=1}^T (F_S(\rvx_t, \rvy_t) - \Phi_S(\hat{\rvx}^*)) \leq
        \frac{\mu_\rvx t_0 D_\gX }{2T}
        + \frac{L^2 \log{(eT)}}{2\mu_\rvx T}  \\
        & \quad +  \frac{2 \sqrt{D_\gX}}{T} \left( \frac{2L}{3} + 2L\sqrt{T} \right) \log{\frac{6}{\delta}}
        + \frac{2L \sqrt{D_\gX} \left(2T\log{\frac{6}{\delta}}\right)^{\frac{1}{2}}}{T}.
    \end{aligned}
    \end{equation}    
    
    Similarly, we can bound $\Phi(\bar{\rvx}_T) - \frac{1}{T} \sum_{t=1}^T F_S(\rvx_t, \rvy_t)$. Firstly, we have
    \begin{align*}
        \| \rvy_{t+1} - \rvy \|^2 
        & = \| \rvy_t + \eta_{\rvy_t} \nabla_\rvy f(\rvx_t, \rvy_t; \rvz_{i_t}) - \rvy \|^2 \\
        & = \| \rvy_t - \rvy \|^2 + \eta_{\rvy_t}^2 \|\nabla_\rvy f(\rvx_t, \rvy_t; \rvz_{i_t})\|^2 + 2\eta_{\rvx_y} \langle \rvy_t - \rvy, \nabla_\rvy f(\rvx_t, \rvy_t; \rvz_{i_t}) \rangle \\
        & \leq  \| \rvy_t - \rvy \|^2 + \eta_{\rvy_t}^2 L^2 + 2\eta_{\rvy_t} \langle \rvy_t - \rvy, \nabla_\rvy f(\rvx_t, \rvy_t; \rvz_{i_t}) - \nabla_\rvy F_S (\rvx_t, \rvy_t) \rangle \\
        & \quad + 2\eta_{\rvy_t} \langle \rvy_t - \rvy, \nabla_\rvx F_S(\rvx_t, \rvy_t) \rangle,
    \end{align*}
    where the first inequality holds because of Assumption \ref{assumption:f-lipschitz}. According to the strong concavity of $F_S(\rvx_t, \cdot)$, we have
    \begin{align*}
        2\eta_{\rvy_t} (F_S(\rvx_t, \rvy) - F_S(\rvx_t, \rvy_t)) 
        \leq (1-\eta_{\rvy_t} \mu_\rvy)\| \rvy_t - \rvy \|^2 - \| \rvy_{t+1} - \rvy \|^2 + \eta_{\rvy_t}^2 L^2 \\
        + 2\eta_{\rvy_t} \langle \rvy_t -\rvy, \nabla_\rvy f(\rvx_t, \rvy_t; \rvz_{i_t}) - \nabla_\rvy F_S(\rvx_t, \rvy_t)\rangle.
    \end{align*}

    Let $\eta_{\rvy_t} = \frac{1}{\mu_\rvy (t+t_0)}$, we further get
    \begin{align*}
        \frac{2}{\mu_\rvy (t +t_0)} (F_S(\rvx_t, \rvy) - F_S(\rvx_t, \rvy_t)) \leq
        \left( 1-\frac{1}{t+t_0} \right) \|\rvy_t -\rvy\|^2 - \|\rvy_{t+1}-\rvy\|^2 \\
        + \left( \frac{L}{\mu_\rvy(t+t_0)} \right)^2  + \frac{2}{\mu_\rvy(t+t_0)} \langle \rvy_t -\rvy, \nabla_\rvy f(\rvx_t, \rvy_t; \rvz_{i_t}) - \nabla_\rvy F_S(\rvx_t, \rvy_t)\rangle.
    \end{align*}
    Multiplying both sides by $t+t_0$, we have
    \begin{align*}
        \frac{2}{\mu_\rvy} (F_S(\rvx_t, \rvy) - F_S(\rvx_t, \rvy_t)) \leq
        \left( t + t_0 -1 \right) \|\rvy_t -\rvy\|^2 - (t+t_0) \|\rvy_{t+1}-\rvy\|^2 \\
        + \frac{L^2}{\mu_\rvy^2(t+t_0)} + \frac{2}{\mu_\rvy} \langle \rvy_t -\rvy, \nabla_\rvy f(\rvx_t, \rvy_t; \rvz_{i_t}) - \nabla_\rvy F_S(\rvx_t, \rvy_t)\rangle.
    \end{align*}
    Since $\rvy_1 = 0$ and $\sum_{t=1}^T t^{-1} \leq \log{(eT)}$, by taking a summation of the above inequality from $t=1$ to $T$, we have 
    \begin{align*}
        & \sum_{t=1}^T (F_S(\rvx_t, \rvy) - F_S(\rvx_t, \rvy_t)) \leq
        \frac{\mu_\rvy}{2}t_0 D_\gY
        + \frac{L^2 \log{(eT)}}{2\mu_\rvy}  \\
        & + \sum_{t=1}^T \langle \rvy_t, \nabla_\rvy f(\rvx_t, \rvy_t; \rvz_{i_t}) - \nabla_\rvy F_S(\rvx_t, \rvy_t)\rangle 
        + \sum_{t=1}^T \langle \rvy, \nabla_\rvy F_S(\rvx_t, \rvy_t) - \nabla_\rvy f(\rvx_t, \rvy_t; \rvz_{i_t})\rangle.
    \end{align*}
    From the convexity of $F_S(\cdot, \rvy)$, we get
    \begin{align*}
        & \sum_{t=1}^T (F_S(\bar{\rvx}_T, \rvy) - F_S(\rvx_t, \rvy_t)) \leq
        \frac{\mu_\rvy}{2}t_0 D_\gY
        + \frac{L^2 \log{(eT)}}{2\mu_\rvy}  \\
        & + \sum_{t=1}^T \langle \rvy_t, \nabla_\rvy f(\rvx_t, \rvy_t; \rvz_{i_t}) - \nabla_\rvy F_S(\rvx_t, \rvy_t)\rangle 
        + \sum_{t=1}^T \langle \rvy, \nabla_\rvy F_S(\rvx_t, \rvy_t) - \nabla_\rvy f(\rvx_t, \rvy_t; \rvz_{i_t})\rangle.
    \end{align*}
    Since this inequality holds for any $\rvy$, we get
    \begin{align*}
        & \sum_{t=1}^T (\sup_{y\in\gY} F_S(\bar{\rvx}_T, \rvy) - F_S(\rvx_t, \rvy_t)) \leq
        \frac{\mu_\rvy}{2}t_0 D_\gY
        + \frac{L^2 \log{(eT)}}{2\mu_\rvy}  \\
        & + \sum_{t=1}^T \langle \rvy_t, \nabla_\rvy f(\rvx_t, \rvy_t; \rvz_{i_t}) - \nabla_\rvy F_S(\rvx_t, \rvy_t)\rangle 
        + \sum_{t=1}^T \sup_{y\in\gY} \langle \rvy, \nabla_\rvy F_S(\rvx_t, \rvy_t) - \nabla_\rvy f(\rvx_t, \rvy_t; \rvz_{i_t})\rangle,
    \end{align*}
    which implies that
    \begin{align*}
        & \sum_{t=1}^T (\Phi_S(\bar{\rvx}_T) - F_S(\rvx_t, \rvy_t)) \leq
        \frac{\mu_\rvy}{2}t_0 D_\gY
        + \frac{L^2 \log{(eT)}}{2\mu_\rvy}  \\
        & + \sum_{t=1}^T \langle \rvy_t, \nabla_\rvy f(\rvx_t, \rvy_t; \rvz_{i_t}) - \nabla_\rvy F_S(\rvx_t, \rvy_t)\rangle 
        + \sum_{t=1}^T  \sup_{y\in\gY} \langle \rvy, \nabla_\rvy  F_S(\rvx_t, \rvy_t) - \nabla_\rvy f(\rvx_t, \rvy_t; \rvz_{i_t})\rangle.
    \end{align*}
    By Schwarz's inequality, we have
    \begin{align*}
        & \sum_{t=1}^T (\Phi_S(\bar{\rvx}_T) - F_S(\rvx_t, \rvy_t)) \leq
        \frac{\mu_\rvy}{2}t_0 D_\gY
        + \frac{L^2 \log{(eT)}}{2\mu_\rvy}  \\
        & + \sum_{t=1}^T \langle \rvy_t, \nabla_\rvy f(\rvx_t, \rvy_t; \rvz_{i_t}) - \nabla_\rvy F_S(\rvx_t, \rvy_t)\rangle 
        + D_\gY \left\| \sum_{t=1}^T (\nabla_\rvy F_S(\rvx_t, \rvy_t) - \nabla_\rvy f(\rvx_t, \rvy_t; \rvz_{i_t})) \right\|.
    \end{align*}
    Denote that $\tilde{\xi}_t = \langle \rvy_t, \nabla_\rvy f(\rvx_t, \rvx_t;\rvz_{i_t}) - \nabla_\rvy F_S(\rvx_t, \rvy_t) \rangle$. Since $\E_{i_t} [\langle \rvy_t, \nabla_\rvy f(\rvx_t, \rvx_t;\rvz_{i_t}) - \nabla_\rvy F_S(\rvx_t, \rvy_t) \rangle] = 0$, so $\{ \tilde{\xi}_t | t = 1,\ldots,T \}$ is a martingale difference sequence. By Schwarz's inequality and Assumption \ref{assumption:f-lipschitz}, we know that $|\langle \rvy_t, \nabla_\rvy f(\rvx_t, \rvx_t;\rvz_{i_t}) - \nabla_\rvy F_S(\rvx_t, \rvy_t) \rangle| \leq 2L \sqrt{D_\gY}$. Then according to Lemma \ref{lemma:Azuma-Hoeffding-inequality} we have the following inequality with probability at least $1-\frac{\delta}{6}$
    \begin{align*}
        \sum_{t=1}^T \langle \rvy_t, \nabla_\rvy f(\rvx_t, \rvx_t;\rvz_{i_t}) - \nabla_\rvy F_S(\rvx_t, \rvy_t) \rangle
        \leq 2L \sqrt{D_\gY} \left(2T\log{\frac{6}{\delta}}\right)^{\frac{1}{2}}.
    \end{align*}
    Define $\tilde{\xi}'_t = \nabla_\rvy F_S(\rvx_t, \rvy_t) - \nabla_\rvy f(\rvx_t, \rvy_t; \rvz_{i_t})$. Then we get $\|\tilde{\xi}'_t\| \leq 2L$ and
    \begin{align*}
        \sum_{t=1}^T \E [\| \tilde{\xi}'_t \|^2|\tilde{\xi}'_1,\ldots, \tilde{\xi}'_{t-1}] \leq 4TL^2.
    \end{align*}
    Applying Lemma \ref{lemma:Bernstein-type-martingale-inequality} to the martingale difference sequence $\{ \tilde{\xi}'_t \}$, we have the following inequality with probability at least $1-\frac{\delta}{3}$
    \begin{align*}
        \left\| \sum_{t=1}^T \tilde{\xi}'_t \right\| 
        \leq 2 \left( \frac{2L}{3} + 2L\sqrt{T} \right) \log{\frac{6}{\delta}}.
    \end{align*}
    This implies that with probability at least $1-\frac{\delta}{3}$
    \begin{align*}
        \left\| \sum_{t=1}^T \left( \nabla_\rvy F_S(\rvx_t, \rvy_t) - \nabla_\rvy f(\rvx_t, \rvy_t; \rvz_{i_t}) \right) \right\| 
        \leq 2 \left( \frac{2L}{3} + 2L\sqrt{T} \right) \log{\frac{6}{\delta}}.
    \end{align*}
    Combined with the above results, we finally have the following inequality with probability at least $1-\frac{\delta}{2}$
    \begin{equation}
    \begin{aligned}
    \label{eq:sgda-optimization-error-1-2}
         & \frac{1}{T} \sum_{t=1}^T (\Phi_S(\bar{\rvx}_T) - F_S(\rvx_t, \rvy_t)) \leq
        \frac{\mu_\rvy t_0 D_\gY }{2T}
        + \frac{L^2 \log{(eT)}}{2\mu_\rvy T}  \\
        & \quad +  \frac{2 \sqrt{D_\gY}}{T} \left( \frac{2L}{3} + 2L\sqrt{T} \right) \log{\frac{6}{\delta}}
        + \frac{2L \sqrt{D_\gY} \left(2T\log{\frac{6}{\delta}}\right)^{\frac{1}{2}}}{T}.
    \end{aligned}
    \end{equation}    
    Combing (\ref{eq:sgda-optimization-error-1-1}) and (\ref{eq:sgda-optimization-error-1-2}) together, with probability at least $1-\delta$ we get the following inequality
    \begin{align*}
        &\Phi_S(\bar{\rvx}_T) - \Phi_S(\hat{\rvx}^*) \leq
        \frac{t_0(\mu_\rvx D_\rvx + \mu_\rvy D_\rvy)}{2T} + \frac{L^2 \log(eT)}{2T} \left( \frac{1}{\mu_\rvx} + \frac{1}{\mu_\rvy}\right) \\
        & \quad +  \frac{2(\sqrt{D_\gX} + \sqrt{D_\gY})}{T} \left( \frac{2L}{3} + 2L\sqrt{T} \right) \log{\frac{6}{\delta}}
        + \frac{2L (\sqrt{D_\gX} + \sqrt{D_\gY}) \left(2T\log{\frac{6}{\delta}}\right)^{\frac{1}{2}}}{T}.
    \end{align*}
\end{proof}

\begin{proof}[Proof of Theorem \ref{theorem:sgda-fast-rate}]
    According to Lemma \ref{lemma:primal-function-is-PL}, $\Phi(\rvx)$ satisfies the PL assumption with parameter $\mu_\rvx$, we have
    \begin{align}
    \label{eq:theorem:sgda-fast-rate-1-1}
        \Phi(\rvx) - \Phi(\rvx^*) \leq \frac{\| \nabla \Phi(\rvx) \|^2}{2\mu_\rvx}.
    \end{align}
    To bound $\Phi(\bar{\rvx}_T) - \Phi(\rvx^*)$, we need to bound the term $ \| \nabla \Phi(\bar{\rvx}_T) \|^2 $. There holds that
    \begin{align}
    \label{eq:theorem:sgda-fast-rate-1-2}
        \| \nabla \Phi(\bar{\rvx}_T) \|^2 \leq 2 \| \nabla \Phi(\bar{\rvx}_T) - 
        \nabla \Phi_S(\bar{\rvx}_T) \|^2 
        + 2 \| \nabla \Phi_S(\bar{\rvx}_T) \|^2.
    \end{align}
    From Theorem \ref{theorem:nablaf-without-d}, under Assumption \ref{assumption:minimax-bernstein-condition} and \ref{assumption:SC-SC}. Plugging $\bar{\rvx}_T$ into Theorem \ref{theorem:nablaf-without-d}, for any $\delta \in (0,1)$, when 
    $ n \geq \frac{c\beta^2(\mu_\rvy+\beta)^4(d+ \log{\frac{16\log_2{\sqrt{2}R_1 n + 1}}{\delta}})}{\mu_\rvy^4\mu_\rvx^2}$, 
    with probability at least $1-\frac{\delta}{2}$

    \begin{equation}
        \begin{aligned}
    \label{eq:theorem:sgda-fast-rate-1-3}
        &  \| \nabla \Phi(\bar{\rvx}_T) - \nabla \Phi_S(\bar{\rvx}_T) \| 
        \leq \| \nabla \Phi_S(\bar{\rvx}_T) \|
            + 2\sqrt{\frac{2 \E [ \|\nabla_\rvx f(\rvx^*, \rvy^*; \rvz)\|^2 ] \log{\frac{16}{\delta}}}{n}} \\
            & \quad + \frac{2 B_{\rvx^*} \log{\frac{16}{\delta}}}{n}
            + \frac{\mu_\rvx}{n}  
            + \frac{2\beta}{\mu_\rvy} \left(
        \sqrt{\frac{2 \E [ \|\nabla_\rvy f(\rvx^*, \rvy^*; \rvz)\|^2 ] \log{\frac{16}{\delta}}}{n}} 
        + \frac{B_{\rvy^*} \log{\frac{16}{\delta}}}{n}
        \right).
    \end{aligned}
    \end{equation}

    Next, we need to bound the optimization error bound $\| \nabla \Phi_S(\bar{\rvx}_T) \|$. Firstly, according to Lemma \ref{lemma:SGDA-SC-SC-optimization-error-bound}, with probability at least $1-\delta$, we have
    \begin{align*}
        \Phi_S(\bar{\rvx}_T) - \Phi_S(\rvx^*) 
        = O\left(
        \frac{ \log{\frac{1}{\delta}}}{\sqrt{T}}
        \right)
        .
    \end{align*}
    According to \cite{nesterov2003introductory} and Lemma \ref{lemma:primal-function-is-smooth}, there holds the following property for $\beta + \frac{\beta^2}{\mu_\rvy}$ function $\Phi_S(\rvx)$, we have
    \begin{align}
    \label{eq:theorem:sgda-fast-rate-1-4}
        \frac{1}{2\left(\beta + \frac{\beta^2}{\mu_\rvy}\right)} \| \nabla \Phi_S(\bar{\rvx}_T)\|^2 
        \leq
        \Phi_S(\bar{\rvx}_T) - \Phi_S(\rvx^*) 
         = O\left(
        \frac{\log{\frac{1}{\delta}}}{\sqrt{T}}
        \right).
    \end{align}

    Plugging (\ref{eq:theorem:sgda-fast-rate-1-4}) into (\ref{eq:theorem:sgda-fast-rate-1-3}), according to  Cauchy–Bunyakovsky–Schwarz inequality, we can derive that
    \begin{equation}
    \begin{aligned}
    \label{eq:theorem:sgda-fast-rate-1-5}
        & \| \nabla \Phi(\bar{\rvx}_T) - \nabla \Phi_S(\bar{\rvx}_T) \|^2 
        \leq  O\left( \frac{\log{\frac{1}{\delta}}}{\sqrt{T}} \right) 
        + \frac{32 \beta^2\E [ \|\nabla_\rvy f(\rvx^*, \rvy^*; \rvz)\|^2 ] \log{\frac{16}{\delta}}}{\mu_\rvy^2 n} \\
        & \quad + \frac{32 \E [ \|\nabla_\rvx f(\rvx^*, \rvy^*; \rvz)\|^2 ] \log{\frac{16}{\delta}}}{n}
        + \frac{ 4 \left(
        \frac{2\beta B_{\rvy^*}}{\mu_\rvy} \log{\frac{16}{\delta}} + 2B_{\rvx^*} \log{\frac{16}{\delta}} + \mu_\rvx
        \right)^2
        }{n^2}.
    \end{aligned}
    \end{equation}
    Then substituting (\ref{eq:theorem:sgda-fast-rate-1-5}), (\ref{eq:theorem:sgda-fast-rate-1-4}) into (\ref{eq:theorem:sgda-fast-rate-1-2}), we derive that
    \begin{equation}
        \begin{aligned}
        \label{eq:theorem:sgda-fast-rate-1-6}
         & \| \nabla \Phi(\bar{\rvx}_T) \|^2 
        \leq O\left( \frac{\log{\frac{1}{\delta}}}{\sqrt{T}} \right) 
        + \frac{64 \beta^2\E [ \|\nabla_\rvy f(\rvx^*, \rvy^*; \rvz)\|^2 ] \log{\frac{16}{\delta}}}{\mu_\rvy^2 n} \\
        & \quad + \frac{64 \E [ \|\nabla_\rvx f(\rvx^*, \rvy^*; \rvz)\|^2 ] \log{\frac{16}{\delta}}}{n}
        + \frac{ 8 \left(
        \frac{2\beta B_{\rvy^*}}{\mu_\rvy} \log{\frac{16}{\delta}} + 2B_{\rvx^*} \log{\frac{16}{\delta}} + \mu_\rvx
        \right)^2
        }{n^2}.
        \end{aligned}
    \end{equation}

    Finally, we plug (\ref{eq:theorem:sgda-fast-rate-1-6}) into (\ref{eq:theorem:sgda-fast-rate-1-1}) and choose $ T \asymp O\left( n^2 \right) $, with probability at least $1-\delta$
    \begin{align*}
        \Phi(\bar{\rvx}_T) - \Phi(\rvx^*) 
        = O \Bigg( \frac{\E [ \|\nabla_\rvx f(\rvx^*, \rvy^*; \rvz)\|^2 ] \log{\frac{1}{\delta}}}{n} 
        + \frac{\E [ \|\nabla_\rvy f(\rvx^*, \rvy^*; \rvz)\|^2 ] \log{\frac{1}{\delta}}}{n} + \frac{\log^2{\frac{1}{\delta}}}{n^2} \Bigg).
    \end{align*}

Next, if we further assume $f(\rvx,\rvy;\rvz)$ is non-negative, According to Lemma \ref{lemma:smooth-self-bounding-property}, we have $\E \|\nabla_\rvx f(\rvx^*, \rvy^*; \rvz)\|^2 \leq 4\beta \E [f(\rvx^*, \rvy^*; \rvz)] $ and $ \E \|\nabla_\rvy f(\rvx^*, \rvy^*; \rvz)\|^2 \leq 4\beta \E [f(\rvx^*, \rvy^*; \rvz)]$. Plugging (\ref{eq:theorem:sgda-fast-rate-1-6}) into (\ref{eq:theorem:sgda-fast-rate-1-1}), then we have 
\begin{align*}
         \Phi(\bar{\rvx}_T) - \Phi(\rvx^*) 
         & \leq \frac{\| \nabla \Phi(\bar{\rvx}_T) \|^2}{2\mu_\rvx} \\
         & \leq O\left( \frac{\log{\frac{1}{\delta}}}{\sqrt{T}} \right) 
         + \frac{128 \beta^3 \E [f(\rvx^*, \rvy^*; \rvz)] \log{\frac{16}{\delta}}}{\mu_\rvx \mu_\rvy^2 n}
        + \frac{128 \beta \E [f(\rvx^*, \rvy^*; \rvz)]  \log{\frac{16}{\delta}}}{\mu_\rvx n} \\ 
        & \quad + \frac{ 4 \left(
        \frac{2\beta B_{\rvy^*}}{\mu_\rvy} \log{\frac{16}{\delta}} + 2B_{\rvx^*} \log{\frac{16}{\delta}} + \mu_\rvx
        \right)^2
        }{\mu_\rvx n^2} \\
        & =  O\left( \frac{\log{\frac{1}{\delta}}}{\sqrt{T}} \right)  + 
        \frac{128 \beta (\beta^2 + \mu_\rvy^2) \Phi(\hat{\rvx}^*) \log{\frac{16}{\delta}}}{\mu_\rvx \mu_\rvy^2 n}
        + \frac{4 \left(
        \frac{2\beta B_{\rvy^*}}{\mu_\rvy} \log{\frac{16}{\delta}} + 2B_{\rvx^*} \log{\frac{16}{\delta}} + \mu_\rvx
        \right)^2
        }{\mu_\rvx n^2}.
\end{align*}

Furthermore, when $T \asymp n^4$ we have 
\begin{align*}
    \Phi(\bar{\rvx}_T) - \Phi(\rvx^*) = O \left(
    \frac{\Phi(\rvx^*) \log{\frac{1}{\delta}} }{n}
    + \frac{\log^2{\frac{1}{\delta}}}{n^2}
    \right).
\end{align*}
The proof is complete.

\end{proof}

\section{Some improved bounds with expectation formats}
\label{sec:expectation-bounds}

Firstly, we translate our high probability result of Theorem \ref{theorem:nablaf-without-d} into an expectation result.

\begin{theorem}
\label{theorem:expectation-nablaf-without-d}
Under Assumption \ref{assumption:NC-SC} and \ref{assumption:minimax-bernstein-condition}, assume that the population risk $F(\rvx,\rvy)$ satisfies Assumption \ref{assumption:PL-condition-x} with parameter $\mu_\rvx$ and let $c = \max \{16C^2, 1\}$. We have that for all $\rvx \in \gX$, when 
$ n \geq \frac{c\beta^2(\mu_\rvy+\beta)^4(d+ \log{\frac{16\log_2{\sqrt{2}R_1 n + 1}}{\delta}})}{\mu_\rvy^4\mu_\rvx^2}$,
the excess risks of primal functions can be bounded by
\begin{align*}
     \E \| \Phi(\rvx) - \Phi(\rvx^*) \| 
    \leq  \frac{8 \E [ \| \nabla \Phi_S(\rvx) \|^2 ]}{\mu_\rvx} + O \Bigg( 
    \frac{\E \left[ \|\nabla_\rvx f(\rvx^*, \rvy^*; \rvz)\|^2 \right]}{\mu_\rvx n} 
    + \frac{\beta^2 \E \left[ \|\nabla_\rvy f(\rvx^*, \rvy^*; \rvz)\|^2 \right]}{\mu_\rvx \mu_\rvy^2 n}
    + \frac{1}{n^2}
    \Bigg).
\end{align*}
\end{theorem}

\begin{proof}[Proof of Theorem \ref{theorem:expectation-nablaf-without-d}]
    According to Theorem \ref{theorem:nablaf-without-d}, we have that for all $\rvx \in \gX$, when 
$ n \geq \frac{c\beta^2(\mu_\rvy+\beta)^4(d+ \log{\frac{16\log_2{\sqrt{2}R_1 n + 1}}{\delta}})}{\mu_\rvy^4\mu_\rvx^2}$ 
with probability at least $1-\delta$
\begin{align*}
    \Phi(\rvx) - \Phi(\rvx^*) 
       \leq  &         
        \frac{8 \| \nabla \Phi_S(\rvx) \|^2}{\mu_\rvx} +
    \frac{16 \beta^2\E [ \|\nabla_\rvy f(\rvx^*, \rvy^*; \rvz)\|^2 ] \log{\frac{4}{\delta}}}{\mu_\rvx \mu_\rvy^2 n} \\
        & \quad + \frac{16 \E [ \|\nabla_\rvx f(\rvx^*, \rvy^*; \rvz)\|^2 ] \log{\frac{8}{\delta}}}{\mu_\rvx n}  
        + \frac{ 2 \left(
        \frac{2\beta B_{\rvy^*}}{\mu_\rvy} \log{\frac{8}{\delta}} + 2B_{\rvx^*} \log{\frac{8}{\delta}} + \mu_\rvx
        \right)^2
        }{\mu_\rvx n^2}.
\end{align*}

Thus, with probability at least $1-\delta$, we have
\begin{align*}
    \Phi(\rvx) - \Phi(\rvx^*) - \frac{8 \| \nabla \Phi_S(\rvx) \|^2}{\mu_\rvx}
       \leq                   
    \frac{16 \beta^2\E [ \|\nabla_\rvy f(\rvx^*, \rvy^*; \rvz)\|^2 ] \log{\frac{8}{\delta}}}{\mu_\rvx \mu_\rvy^2 n} \\
         \quad + \frac{16 \E [ \|\nabla_\rvx f(\rvx^*, \rvy^*; \rvz)\|^2 ] \log{\frac{8}{\delta}}}{\mu_\rvx n}  
        + \frac{ 2 \left(
        \frac{2\beta B_{\rvy^*}}{\mu_\rvy} \log{\frac{8}{\delta}} + 2B_{\rvx^*} \log{\frac{8}{\delta}} + \mu_\rvx
        \right)^2
        }{\mu_\rvx n^2}.
\end{align*}

According to the standard statistical analysis, Let $X$ be a random variable, if for some $v, c > 0$,$ \sP\{ X > v t + c t^2 \} \leq e^{-t}$ for every $t > 0$. Then we can easily derive that $\E [X] = \int_0^{\infty} \sP \{ X > x\} dx = v + 2c$. 
Thus, we have the expectation result
\begin{align*}
       \E \| \Phi(\rvx) - \Phi(\rvx^*) \| 
    \leq  \frac{8 \E [ \| \nabla \Phi_S(\rvx) \|^2 ]}{\mu_\rvx} + O \Bigg( 
    \frac{\E \left[ \|\nabla_\rvx f(\rvx^*, \rvy^*; \rvz)\|^2 \right]}{\mu_\rvx n} 
    + \frac{\beta^2 \E \left[ \|\nabla_\rvy f(\rvx^*, \rvy^*; \rvz)\|^2 \right]}{\mu_\rvx \mu_\rvy^2 n}
    + \frac{1}{n^2}
    \Bigg).
\end{align*}
The proof is complete.
\end{proof}

Next, we give the proofs of the expectation result to relax the SC-SC assumptions given in Table \ref{table:summary}. The proofs are similar with high probability format since we use the existing results for optimization error bounds with expectation format under NC-SC assumptions.

\subsection{GDA}

\begin{lemma}[Optimization error bound for GDA in NC-SC minimax problems \cite{lin2020gradient}]
    Under Assumption \ref{assumption:NC-SC}, and letting the step sizes be chosen as $\eta_\rvx = \frac{1}{16(\frac{\beta}{\mu} + 1)^2\beta}$ and $\eta_\rvy= \frac{1}{\beta}$, then the optimization error bound of \text{Algorithm} \ref{algo:gda} can be bounded by
    \begin{align*}
        \E \| \nabla \Phi_S(\rvx_T) \|^2         
        = O \left(
        \frac{\beta^3 \Delta_\Phi}{\mu_\rvy^2 T} 
        + 
        \frac{\beta^3 D_\gY}{\mu_\rvy T} \right),
    \end{align*}
    where $\Delta_\Phi = \Phi_S(\rvx_0) - \min_\rvx \Phi_S(\rvx)$.
\end{lemma}

Using above optimization error bound, we can obtain the following theorem.

\begin{theorem}
\label{theorem:expectation-gda-fast-rate}
Suppose Assumption \ref{assumption:NC-SC} and \ref{assumption:minimax-bernstein-condition} hold. Assume that the population risk $F(\rvx,\rvy)$ satisfies Assumption \ref{assumption:PL-condition-x} with parameter $\mu_\rvx$. Let the step sizes choose as $\eta_\rvx = \frac{1}{16(\frac{\beta}{\mu} + 1)^2\beta}$ and $\eta_\rvy= \frac{1}{\beta}$. When $T \asymp n$ and $ n \geq \frac{c\beta^2(\mu_\rvy+\beta)^4(d+ \log{\frac{16\log_2{\sqrt{2}R_1 n + 1}}{\delta}})}{\mu_\rvy^4\mu_\rvx^2}$, where $c$ is an absolute constant, then the excess risk for primal functions of \text{Algorithm} \ref{algo:gda} can be bounded by
\begin{align*}
    \E [ \Phi(\rvx_T) - \Phi(\rvx^*) ] 
        = O \Bigg( \frac{\E [ \|\nabla_\rvx f(\rvx^*, \rvy^*; \rvz)\|^2 ] }{n}
        + \frac{\E [ \|\nabla_\rvy f(\rvx^*, \rvy^*; \rvz)\|^2 ] }{n} + \frac{1}{n^2} \Bigg).
\end{align*}
Furthermore, Let $T \asymp n^2$. Assume the function $f(\rvx,\rvy;\rvz)$ is non-negative and $ \Phi(\rvx^*) = O\left(\frac{1}{n}\right)$, we have  
\begin{align*}
    \E [ \Phi(\rvx_T) - \Phi(\rvx^*) ] = O \left(
    \frac{1}{n^2}
    \right).
\end{align*}
\end{theorem}

\subsection{SGDA}

For SGDA settings, we introduce a weaker assumption comparing with Assumption \ref{assumption:f-lipschitz}.

\begin{assumption}
\label{assumption:SGDA-bound}
    Assume the existence of $\sigma > 0$ satisfies
    \begin{align*}
        \E [\nabla f(\rvx, \rvy; \rvz) - \nabla F_S(\rvx, \rvy) ] = 0, \\
        \E [\| \nabla f(\rvx, \rvy; \rvz) - \nabla F_S(\rvx, \rvy) \|^2] \leq \sigma^2.
    \end{align*}
\end{assumption}

\begin{lemma}[Optimization error bound for SGDA in NC-SC minimax problems \cite{lin2020gradient}]
    Under Assumption \ref{assumption:NC-SC} and \ref{assumption:SGDA-bound}, and letting the step sizes be chosen as $\eta_\rvx = \frac{1}{16(\frac{\beta}{\mu} + 1)^2\beta}$ and $\eta_\rvy= \frac{1}{\beta}$, then the optimization error bound of \text{Algorithm} \ref{algo:gda} can be bounded by
    \begin{align*}
        \E \| \nabla \Phi_S(\rvx_T) \|^2         
        = O \left(
        \sqrt[5]{\frac{\beta^6}{\mu_\rvy^6 T^2}}
        \right).
    \end{align*}
\end{lemma}

Using above optimization error bound, we can obtain the following theorem.

\begin{theorem}
\label{theorem:expectation-sgda-fast-rate}
Suppose Assumption \ref{assumption:NC-SC}, \ref{assumption:minimax-bernstein-condition} and \ref{assumption:SGDA-bound} hold. Assume that the population risk $F(\rvx,\rvy)$ satisfies Assumption \ref{assumption:PL-condition-x} with parameter $\mu_\rvx$. Let the step sizes choose as $\eta_\rvx = \frac{1}{16(\frac{\beta}{\mu} + 1)^2\beta}$ and $\eta_\rvy= \frac{1}{\beta}$. When $T \asymp n^{\frac{5}{2}}$ and $ n \geq \frac{c\beta^2(\mu_\rvy+\beta)^4(d+ \log{\frac{16\log_2{\sqrt{2}R_1 n + 1}}{\delta}})}{\mu_\rvy^4\mu_\rvx^2}$, where $c$ is an absolute constant, then the excess risk for primal functions of \text{Algorithm} \ref{algo:sgda} can be bounded by
\begin{align*}
    \E [ \Phi(\rvx_T) - \Phi(\rvx^*) ] 
        = O \Bigg( \frac{\E [ \|\nabla_\rvx f(\rvx^*, \rvy^*; \rvz)\|^2 ] }{n}
        + \frac{\E [ \|\nabla_\rvy f(\rvx^*, \rvy^*; \rvz)\|^2 ] }{n} + \frac{1}{n^2} \Bigg).
\end{align*}
Furthermore, Let $T \asymp n^5$. Assume the function $f(\rvx,\rvy;\rvz)$ is non-negative and $ \Phi(\rvx^*) = O\left(\frac{1}{n}\right)$, we have  
\begin{align*}
    \E [ \Phi(\rvx_T) - \Phi(\rvx^*) ] = O \left(
    \frac{1}{n^2}
    \right).
\end{align*}
\end{theorem}

\subsection{AGDA}
Alternating gradient descent ascent  presented in \text{Algorithm} \ref{algo:agda} was proposed recently to optimize nonconvex-nonconcave problems \cite{yang2020global}.

\begin{algorithm}[H]
  \begin{algorithmic}[1]
  \STATE {\bfseries Input:} {$(\rvx_1, \rvy_1) = (\bm{0},\bm{0})$, step sizes $\{\eta_{\rvx_t}\}_t > 0, \{\eta_{\rvy_t}\}_t > 0$ and dataset $S = \{\rvz_1, \dots ,\rvz_n\}$}
  
  \FOR{$t=1,\dots,T$}
    \STATE update $\rvx_{t+1} = \rvx_{t} - \eta_{\rvx_t} \nabla_{\rvx} f(\rvx_t, \rvy_t;\rvz_{i_t}) $ 
    \STATE update $\rvy_{t+1} = \rvy_{t} + \eta_{\rvy_t} \nabla_{\rvy} f(\rvx_{t+1}, \rvy_t;\rvz_{i_t}) $
    \ENDFOR
  \caption{Two-timescale AGDA for minimax problem}
  \label{algo:agda}
\end{algorithmic}
\end{algorithm}

\begin{lemma}[Optimization error bound for AGDA in PL-SC minimax problems \cite{yang2020global,lei2021stability}]
\label{lemma:expectation-AGDA-NC-SC-optimization-error-bound}
    Under Assumption \ref{assumption:NC-SC} and assume that the population risk $F(\rvx,\rvy)$ satisfies Assumption \ref{assumption:PL-condition-x} with parameter $\mu_\rvx$. Let $\{\rvx_t, \rvy_t\}_t$ be the sequence produced by \text{Algorithm} \ref{algo:agda} with the step sizes chosen as $\eta_{\rvx_t} \asymp \frac{1}{\mu_\rvx t}$ and $\eta_{\rvy_t} = \frac{1}{\mu_\rvx \mu_\rvy^2 t} $, then the optimization error bound of \text{Algorithm} \ref{algo:agda} can be bounded by
    \begin{align*}
        \E [ \| \nabla \Phi_S(\rvx_T) \|^2 ]         
        = O\left( \frac{1}{\mu_\rvx^2 \mu_\rvy^4 T} \right).
    \end{align*}
\end{lemma}

\begin{theorem}
\label{theorem:expectation-agda-fast-rate}
Suppose Assumption \ref{assumption:NC-SC}, \ref{assumption:minimax-bernstein-condition} and \ref{assumption:SGDA-bound} hold. Assume that the population risk $F(\rvx,\rvy)$ satisfies Assumption \ref{assumption:PL-condition-x} with parameter $\mu_\rvx$. Let $\{\rvx_t, \rvy_t\}_t$ be the sequence produced by \text{Algorithm} \ref{algo:agda} with the step sizes chosen as $\eta_{\rvx_t} \asymp \frac{1}{\mu_\rvx t}$ and $\eta_{\rvy_t} \asymp \frac{1}{\mu_\rvx \mu_\rvy^2 t} $. When $T \asymp n$ and $ n \geq \frac{c\beta^2(\mu_\rvy+\beta)^4(d+ \log{\frac{16\log_2{\sqrt{2}R_1 n + 1}}{\delta}})}{\mu_\rvy^4\mu_\rvx^2}$, where $c$ is an absolute constant, then the excess risk for primal functions of \text{Algorithm} \ref{algo:agda} can be bounded by
\begin{align*}
    \E [ \Phi(\rvx_T) - \Phi(\rvx^*) ] 
        = O \Bigg( \frac{\E [ \|\nabla_\rvx f(\rvx^*, \rvy^*; \rvz)\|^2 ] }{n}
        + \frac{\E [ \|\nabla_\rvy f(\rvx^*, \rvy^*; \rvz)\|^2 ] }{n} + \frac{1}{n^2} \Bigg).
\end{align*}
Furthermore, Let $T \asymp n^2$. Assume the function $f(\rvx,\rvy;\rvz)$ is non-negative and $ \Phi(\rvx^*) = O\left(\frac{1}{n}\right)$, we have  
\begin{align*}
    \E [ \Phi(\rvx_T) - \Phi(\rvx^*) ] = O \left(
    \frac{1}{n^2}
    \right).
\end{align*}
\end{theorem}

\begin{proof}[Proof of Theorem \ref{theorem:expectation-agda-fast-rate}]
    Since the proofs of Theorem \ref{theorem:expectation-gda-fast-rate}, \ref{theorem:expectation-sgda-fast-rate} and \ref{theorem:expectation-agda-fast-rate} are similar, we only prove Theorem \ref{theorem:expectation-agda-fast-rate} as an example.

    From Theorem \ref{theorem:expectation-nablaf-without-d}, under Assumption \ref{assumption:minimax-bernstein-condition} and \ref{assumption:NC-SC}, when we plug $\rvx_T$ into Theorem \ref{theorem:expectation-nablaf-without-d}, when 
    $ n \geq \frac{c\beta^2(\mu_\rvy+\beta)^4(d+ \log{\frac{16\log_2{\sqrt{2}R_1 n + 1}}{\delta}})}{\mu_\rvy^4\mu_\rvx^2}$, we have
    \begin{equation}
        \begin{aligned}
    \label{eq:theorem:expectation-agda-fast-rate-1-3}
         \E \| \Phi(\rvx) - \Phi(\rvx^*) \| 
    \leq  \frac{8 \E [ \| \nabla \Phi_S(\rvx) \|^2 ]}{\mu_\rvx} + O \Bigg( 
    \frac{\E \left[ \|\nabla_\rvx f(\rvx^*, \rvy^*; \rvz)\|^2 \right]}{\mu_\rvx n} 
    + \frac{\beta^2 \E \left[ \|\nabla_\rvy f(\rvx^*, \rvy^*; \rvz)\|^2 \right]}{\mu_\rvx \mu_\rvy^2 n}
    + \frac{1}{n^2}
    \Bigg).
    \end{aligned}
    \end{equation}

    According to Lemma \ref{lemma:expectation-AGDA-NC-SC-optimization-error-bound}, we have
    \begin{align}
    \label{eq:theorem:expectation-agda-fast-rate-1-4}
        \E [ \| \nabla \Phi_S(\rvx_T) \|^2 ]  
        = O\left( \frac{1}{\mu_\rvx^2 \mu_\rvy^4 T} \right).
    \end{align}

    Plugging (\ref{eq:theorem:expectation-agda-fast-rate-1-4}) into (\ref{eq:theorem:expectation-agda-fast-rate-1-3}) and choose $ T \asymp O\left( n \right) $, with probability at least $1-\delta$
    \begin{align*}
       \E [ \Phi(\rvx_T) - \Phi(\rvx^*) ]
        = O \Bigg( \frac{\E [ \|\nabla_\rvx f(\rvx^*, \rvy^*; \rvz)\|^2 ] }{\mu_\rvx n} 
        + \frac{\beta^2 \E [ \|\nabla_\rvy f(\rvx^*, \rvy^*; \rvz)\|^2 ] }{\mu_\rvx \mu_\rvy^2 n} + \frac{1}{n^2} \Bigg).
    \end{align*}

Next, if we further assume $\Phi(\rvx^*) = O\left( \frac{1}{n} \right)$, According to Lemma \ref{lemma:smooth-self-bounding-property}, we have $\E \|\nabla_\rvx f(\rvx^*, \rvy^*; \rvz)\|^2 \leq 4\beta \E [f(\rvx^*, \rvy^*; \rvz)] $ and $ \E \|\nabla_\rvy f(\rvx^*, \rvy^*; \rvz)\|^2 \leq 4\beta \E [f(\rvx^*, \rvy^*; \rvz)]$, then substituting (\ref{eq:theorem:expectation-agda-fast-rate-1-4}) into (\ref{eq:theorem:expectation-agda-fast-rate-1-3}) and choosing $T \asymp n^2$, we have 
\begin{align*}
        \E [ \Phi(\rvx_T) - \Phi(\rvx^*) ]
        = O \Bigg( \frac{ \beta (\beta^2 + \mu_\rvy^2) \Phi(\rvx^*) }{\mu_\rvx \mu_\rvy^2 n} 
         + \frac{1}{n^2} \Bigg).
\end{align*}
The proof is complete.
    
\end{proof}

\end{document}